\documentclass{article}

\newif\ifarxiv
\arxivtrue      % arXiv version
\ifarxiv
  \usepackage[preprint]{neurips_2026}
\else
  \usepackage[main,final]{neurips_2026}
\fi
\usepackage[utf8]{inputenc}
\usepackage[T1]{fontenc}
\usepackage{microtype}
\usepackage{graphicx}
\usepackage{subcaption}
\usepackage{booktabs}
\usepackage[breaklinks=true]{hyperref}
\usepackage{url}
\usepackage{xcolor}

\usepackage{eso-pic}

\usepackage{eso-pic}

\usepackage{amsmath}
\usepackage{amssymb}
\usepackage{mathtools}
\usepackage{amsthm}
\usepackage[capitalize,noabbrev]{cleveref}
\usepackage{enumitem}
\usepackage{makecell}
\theoremstyle{plain}
\newtheorem{theorem}{Theorem}[section]

\theoremstyle{definition}

\theoremstyle{remark}

\usepackage[textsize=tiny]{todonotes}

\usepackage{amsfonts}
\usepackage{nicefrac}       % compact symbols for 1/2, etc.
\usepackage{bm}
\usepackage{soul}
\usepackage{wrapfig}
\usepackage{placeins}
\usepackage{fancyvrb}
\usepackage{fvextra}

\newcommand{\reda}{\color{red} a}
\newcommand{\blueb}{\color{blue} b}
\title{Shattered Compositionality: Counterintuitive Learning Dynamics of Transformers for Arithmetic} 

\author{%
  Xingyu Zhao \\
  University of Wisconsin--Madison \\
  \texttt{xzhao468@wisc.edu}
  \And
  Darsh Sharma \\
  University of Wisconsin--Madison \\
  \texttt{darsh.sharma@wisc.edu}
  \And
  Rheeya Uppaal \\
  University of Wisconsin--Madison \\
  \texttt{ruppaal@wisc.edu}
  \And
  Yiqiao Zhong \\
  University of Wisconsin--Madison \\
  \texttt{yiqiao.zhong@wisc.edu}
}

\begin{document}
\raggedbottom
\maketitle
% \placerepo
\begin{abstract} 
% Large language models (LLMs) can achieve strong benchmark accuracy while remaining brittle under small distribution shifts. Recent mechanistic work suggests that models often compose skills differently from humans, but two questions remain underexplored: how these non-human compositions emerge during training, and how training data distributions shape them. We study these questions through small transformers trained on synthetic arithmetic tasks, using fine-grained black-box diagnostics to track subskill acquisition over training. We find that transformers often fail to acquire arithmetic skills in human-like sequential order: in addition they learn digits in reverse order, and in multiplication, comparison, and sorting they acquire multiple subskills in parallel. We call this training-dynamics failure mode shattered compositionality. To explain it, we provide evidence that learning is driven by correlational matching to the training distribution rather than by discovering a causal or procedural decomposition of the task. These non-human trajectories have concrete consequences: parallel subskills can compete, producing characteristic mixing errors and sharper failures under controlled distribution shifts. We further show that these patterns persist under alternative setups, larger and pretrained models, and scratchpad supervision, and that related brittleness appears in modern open-source LLMs on clause-extended GSM8K. Our results highlight a mismatch between high apparent competence and robust compositional reasoning. \PX{Alternative abstract version by GPT}

Large language models (LLMs) often achieve strong benchmark accuracy yet remain brittle under small distribution shifts. While recent mechanistic studies reveal the discrepancy between LLMs and humans in skill compositions, the learning dynamics of skill acquisition and the role of data distributions  remain elusive. 
%In this study, we complement existing mechanistic analyses of internal compositions by focusing on learning dynamics of skill acquisition and the consequences. 
In this study, we train transformers on synthetic arithmetic tasks with black-box model-agnostic metrics for analyzing non-human skill compositions. We discover that transformers often acquire skills for arithmetic in reverse order or in parallel instead of human-like sequential rules---a phenomenon we refer to as shattered compositionality. To explain these behaviors, we provide evidence that correlational matching to the training data, rather than causal or procedural composition, shapes learning dynamics. As a consequence, this non-human acquisition creates competition between partially learned skills, producing characteristic mixing errors and weaker robustness under controlled distribution shifts.
We further show that qualitatively similar behavior persists in modern LLMs
% We further show that the same qualitative behavior persists in modern LLMs
and is not mitigated by pure model scaling or scratchpad supervision. 
Our results highlight a mismatch between training-time skill acquisition and the human-like hierarchical compositions, with implications for reasoning
reliability and out-of-distribution robustness.
\footnote{Code repository:
\url{https://github.com/PeterXingyuZhao/Arithmetic_Transformer}.}
\end{abstract}

% \vspace{-0.5cm}
\section{Introduction}
\label{sec:introduction}

Large language models (LLMs) achieve high accuracy on many compositional benchmarks, from code generation to mathematical reasoning~\citep{white2024livebench}. Yet strong in-distribution performance does not guarantee robust composition: small shifts in problem formulation can expose reasoning that ``almost'' composes, but fails in brittle and unexpected ways~\citep{song2025a}. %Such brittleness is especially concerning in settings that require reliable reuse and recombination of learned capabilities, with direct implications for robustness, alignment, and safety~\citep{macdiarmid2025natural}.
Such brittleness is especially concerning in settings that require reliable skill reuse, with direct implications for alignment and safety~\citep{macdiarmid2025natural}.

%Existing mechanistic analyses, which focus on model's internal operations, reveal that models build  skill compositions differently from human procedures. 
Existing mechanistic analyses focus on a model’s internal operations and reveal that learned skill compositions can differ markedly from textbook or human procedures. 
Specifically, 
%A natural hypothesis is that these failures arise because models build and apply skill compositions differently from humans. Recent mechanistic analyses support this view: rather than implementing a single coherent procedure, 
LLMs often rely on multiple interacting mechanisms---sometimes complementary, sometimes competing---consistent with a ``bag of heuristics'' rather than a unified algorithm~\citep{wiegreffeanswer, lindsey2025biology, nikankin2025arithmetic}. 
%More broadly, models can appear competent while relying on non-human internal strategies that fail under distribution shift or adversarial variation~\citep{berglund2024the, shojaee2025illusion}. 
These findings raise a training-dynamics question: If a
model ends up using a non-canonical composition, what made that composition easier to acquire than the canonical one?
However, most prior work studies \emph{static} snapshots of trained models, leaving two fundamental questions open. 
\textit{Q1: How do non-human compositions \textbf{emerge} during training?} %\textit{Q2: what is the role of training \textbf{data distributions} in inducing non-human compositions and brittle generalization.}
\textit{Q2: How do training \textbf{data distributions} shape those compositions and their failures}?
%and what determines the order
%We still lack a developmental account of \emph{how compositional structures emerge during training, and what determines the order in which constituent subskills are learned}. Without this training-dynamics view, it is difficult to diagnose why compositions become brittle or to design interventions that shape them.

%In this work, we provide a complementary approach to existing mechanistic studies: we focus on learning dynamics of models throughout training using black-box data-centric measurements, thereby emphasizing the mechanism of emergence and data distributions. We mostly focus on
%We address this gap by training small transformers on synthetic arithmetic tasks. Arithmetic is a clean compositional domain with well-specified rules and minimal semantic ambiguity, and it has long served as a testbed for understanding learning dynamics such as grokking and length generalization~\citep{dziri2023faith, power2022grokking, anil2022exploring, kazemnejad2023impact}. Crucially, arithmetic admits fine-grained, mechanistically meaningful diagnostics. We decompose tasks into subskills---for example digit-wise prediction, carry-dependent dependencies, length comparison, and digit comparison---and track their acquisition across training, enabling a direct view of how compositions form and how they fail.

%In this work, we take a complementary, data-centric view: we analyze learning dynamics throughout training using black-box measurements.
To answer these questions,
we study small transformers on synthetic arithmetic tasks---a setting where the task's canonical (human-like) decomposition is known and where partial competence can be measured throughout training. Arithmetic is a clean compositional domain with well-specified rules and minimal semantic ambiguity, long serving as a testbed for understanding learning dynamics such as grokking and length generalization~\citep{dziri2023faith, power2022grokking, anil2022exploring, kazemnejad2023impact}. Crucially, arithmetic admits fine-grained diagnostics. We decompose tasks into subskills---for example digit-wise prediction, carry-dependent dependencies, length comparison, and digit comparison---and track their acquisition across training.
We use “counterintuitive” to describe acquisition orders that differ from those suggested by the tasks’ standard procedural decompositions.

For Q1, our central finding is a broader training-dynamics pattern which we call \textbf{shattered compositionality}: relative to a natural hierarchical decomposition of the task, transformers often acquire subskills in reverse order (for example, higher-place digits before lower-place digits) or in parallel. We use the term most strongly for cases where these partially learned subskills do not settle into a stable hierarchy and instead interfere with one another.
We propose that these patterns emerge as maximum-likelihood optimization first fits accessible input-output correlations and the resulting partial skills make further conditional dependencies available to learning.

%For Q1, our central finding is a training-dynamics failure mode which we name as \textbf{shattered compositionality}---transformers do not reliably acquire subskills according to human-like sequential rules; instead, they often learn in \emph{reverse order} (for example, higher-place digits before lower-place digits) or acquire multiple subskills \emph{in parallel}. 

For Q2, we provide evidence consistent with a simple hypothesis: learning dynamics are shaped strongly by the correlational signals exposed by the training distribution, rather than by immediately recovering a causal or procedural decomposition of the underlying rules. Concretely, the emergence of specific subskills aligns with information-theoretic signals available in the data, and additional subskills become learnable once those signals are saturated.

%For Q2, we provide evidence for a simple hypothesis: learning dynamics are shaped primarily by \emph{correlational matching} to the training distribution, rather than by discovering a causal or procedural decomposition of the underlying rules. Concretely, the emergence of specific subskills aligns with information-theoretic signals available in the data, and additional subskills become learnable once those signals are saturated.

%Through explaining the counterintuitive learning trajectories, our analyses reveal the \textbf{consequences} of shattered compositionality---
These analyses also clarify the \textbf{consequences} of shattered compositionality and why it matters. 
Parallel acquisition can create \emph{competition} between partially learned behaviors, which in turn yields characteristic \emph{mixing errors}. %Moreover, mild distribution shifts such as increased subskill contrasts and number of operands in evaluation  can lead to persistent degraded performance, which offers an explanation to empirical observations~\citep{sclarquantifying, mirzadehgsm, chatziveroglou2025exploring}.
Mild distribution shifts, such as stronger subskill contrasts or more operands at evaluation time, can then produce persistent performance degradation. This perspective helps explain empirical failures under prompt and distribution shifts~\citep{sclarquantifying, mirzadehgsm, chatziveroglou2025exploring} and connects them to broader safety concerns~\citep{ren2025llms}.

Our main contributions are as follows.
\begin{itemize}[itemsep=3pt, topsep=0pt, parsep=0pt, leftmargin=30pt]
    %\item We introduce \textbf{shattered compositionality} as a training-dynamics perspective on brittle skill composition: transformers may acquire subskills in reverse order or in parallel rather than through a human-like procedural decomposition.
    \item \emph{Non-canonical acquisition dynamics}. Using addition and multiplication as controlled case studies, we characterize non-canonical training trajectories of skill acquisition, and provide a correlational-matching account of how those orders emerge (Sec.~\ref{sec:mechanism}).
    %we provide evidence that these dynamics are driven by correlational matching to the training distribution (Sec.~\ref{sec:mechanism}).
    %\item Using addition and multiplication as %mechanistically analyzable 
    %in-depth case studies, we show the mechanism of learning dynamics: transformers exhibit non-human learning orders, and we provide evidence that these dynamics are driven by correlational matching to the training distribution (Sec.~\ref{sec:mechanism}).
    \item \emph{Consequences of shattered compositionality.} Using comparison and sorting, we show why non-canonical learning dynamics matter:
    %show the consequences of shattered compositionality:  
    parallel acquisition can create transient or persistent skill competition, producing mixing errors %and sharper failures 
    under controlled distribution shifts (Sec.~\ref{sec:parallel-competition}).
    \item \emph{Connections to broader compositional brittleness.} We show these behaviors and consequences persist under alternative positional encodings, larger models, pretrained models, and scratchpad-based reasoning, and we connect the same broader failure mode to clause-extended GSM8K evaluations on modern LLMs (Sec.~\ref{sec:robustness}--\ref{sec:external-robustness}).
\end{itemize}
%In all, our work analyzes a failure mode of compositionality via model-agnostic metrics evaluated primarily on learning dynamics, which highlights the critical role of data distributions, thus providing a complementary account to existing model-specific operational analysis of compositions.
Overall, our work offers a model-agnostic account of a failure mode of compositionality through training dynamics, complementing existing model-specific analyses of internal operations.

%The rest of the paper follows the causal narrative suggested by these findings. We first introduce the experimental setup. We then analyze the \emph{mechanism} of non-human learning dynamics in addition and multiplication, turn to the \emph{consequences} of parallel acquisition in comparison and sorting, test the \emph{persistence} of these behaviors under standard interventions, and finally examine whether a related brittleness appears in modern LLMs under mild distribution shift.

\begin{figure}[t]
    \centering
    \begin{subfigure}{0.49\textwidth} % Adjust width as needed
        \centering
        \includegraphics[width=\textwidth]{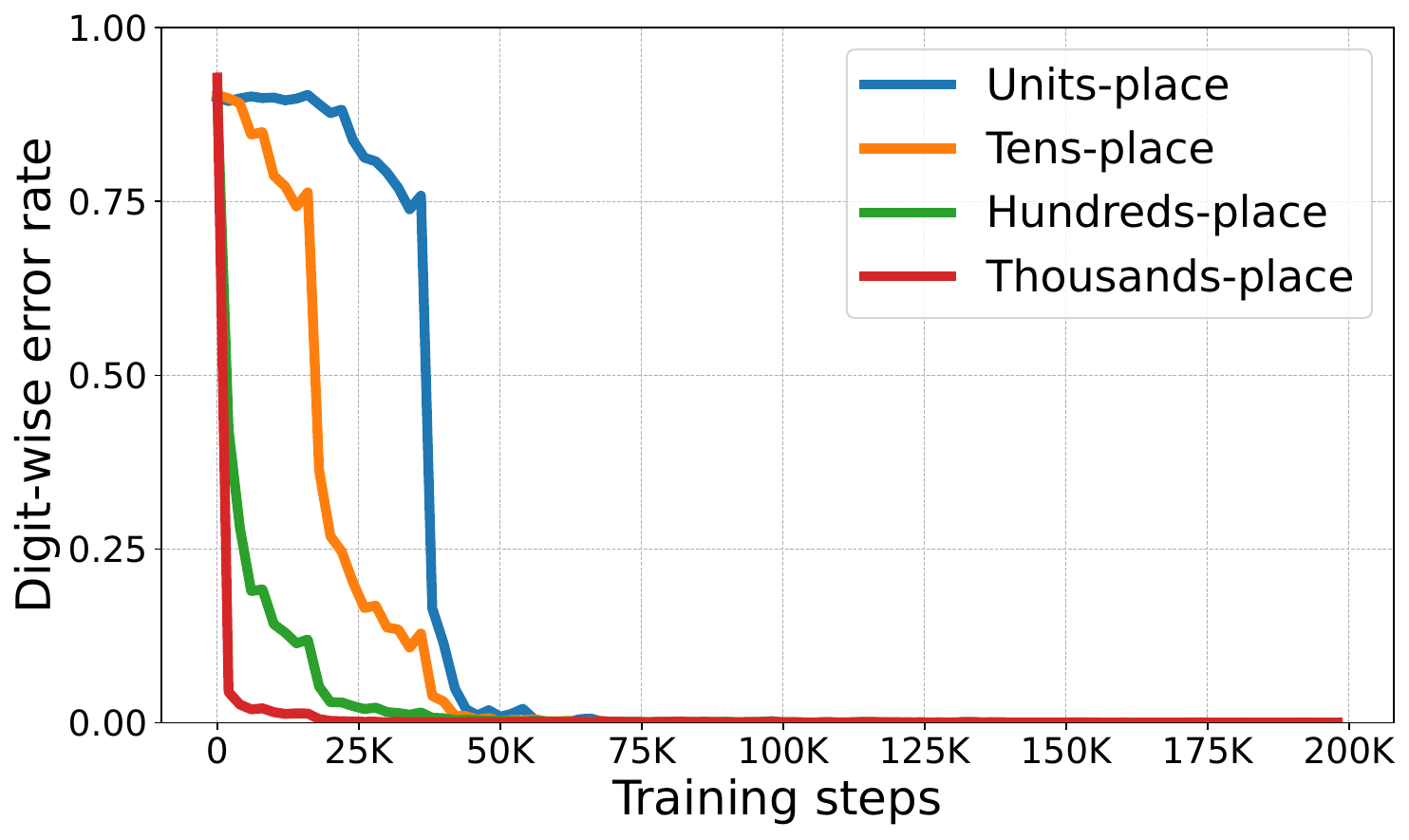} 
        \caption{Digit-wise error in plain output format}
        \label{addition-plain}
    \end{subfigure}%
    \hfill % Adds horizontal space between subfigures
    \begin{subfigure}{0.49\textwidth} % Adjust width as needed
        \centering
        \includegraphics[width=\textwidth]{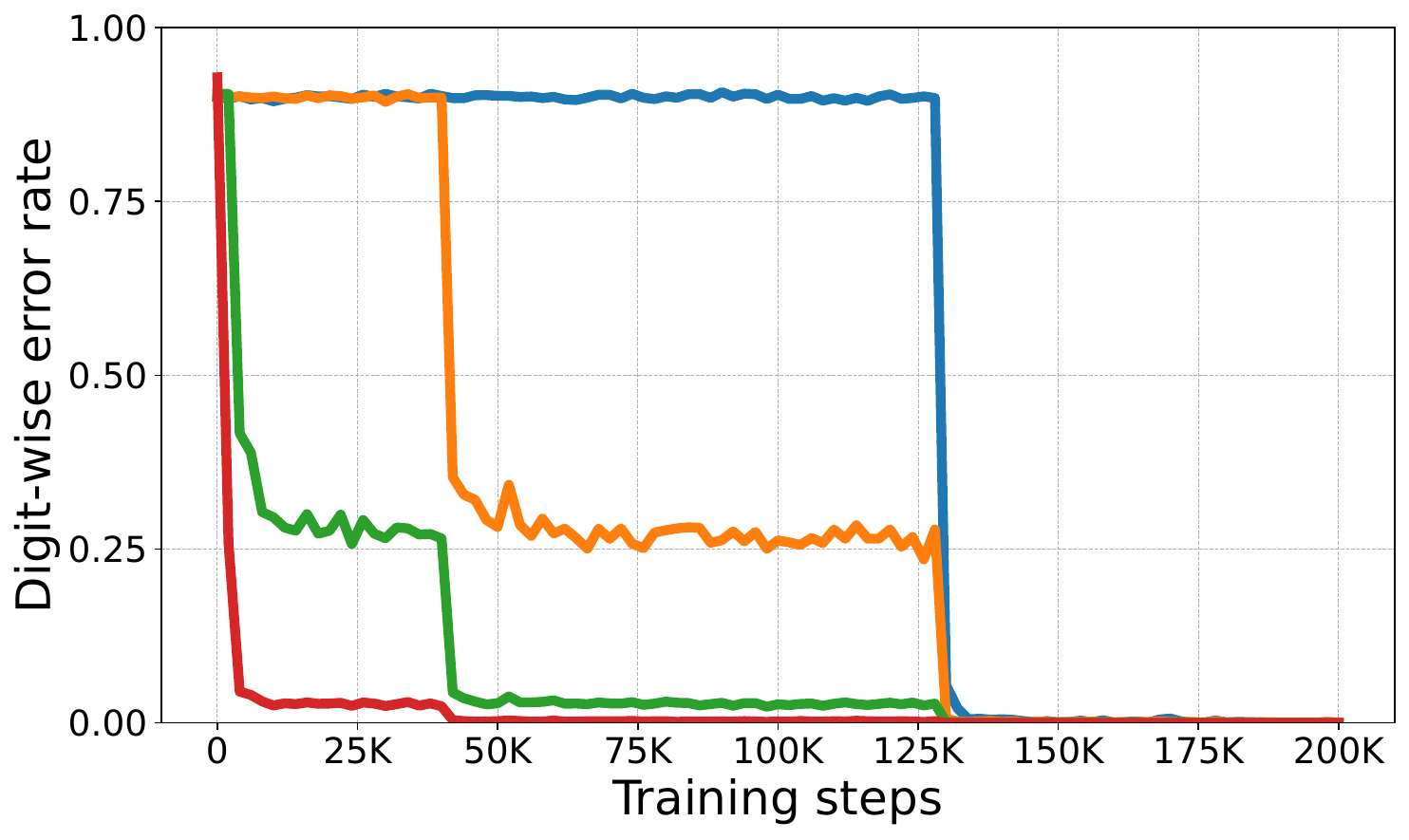}
        \caption{Digit-wise error in reverse output format}
        \label{addition-reverse}
    \end{subfigure}
    \caption{\textbf{Transformers learn digits in non-human order for addition regardless of output formats}. We train transformers on addition of the format $a+b+c+d=e$ with multiple-digit input integers $a,b,c,d$, and evaluate the digit-wise error rates of $e$ across training. In contrast to the standard human procedure, models learn from higher digits to lower digits. 
    %This non-human order remains so even when we reverse the digit order of output $e$ which incentivizes the model to follow the human's order.
    Reversing the output order, which incentivizes the model to follow the human-like order, does not restore human-like acquisition.
    %While it is easier for humans to solve addition under the reverse output format (due to carry), the models learn the task slower and still learn higher digits first.
    %\textbf{Left}: Train/test examples are in the plain output format \eqref{eq:plain} %\textbf{Right}: Train/test examples are in the reverse output format. \eqref{eq:reverse} 
    %\RU{Briefly explain why the left and right plots have different trends.}
    %\YZ{Clean up the plots. Make plot titles/labels, numbers, legends bigger. Make curves thicker. Train both for full 200K steps. } \PX{Updated the figures. These two are results using temperature 0.8.} \YZ{keep increasing the font size. The title should be ``digit-wise error in plain output format'', ``digit-wise error in reverse output format''. Right plot legend is redundant.} \PX{Updated the figures. Does the font size looks good now?} \RU{I think the font could be a little bigger. ML papers will often not have titles above the plots, that info is incorporated into the caption. Also, the legend is missing?} \PX{Figures updated again.}
    %\YZ{Change xticklabels to 25K instead of 25000. Change to x labels to ``Training steps'' (to be consistent with other places). Change the y labels to ``Error rate''.} \PX{Figures updated.}
    %\YZ{Remove the ``error ratio'' from the legend? It looks redundant. Also let's say error rate to be consistent.} \PX{Does this look better? I've removed all "error ratio" from the legend.}
    } 
    % \vskip -1mm
    \label{fig:1}
\end{figure}

%\input{Figures/addition_improving_guesser}
%\vspace{-0.2cm}
\section{Experimental Setup}
\label{sec:experimental-setup}

%We consider $4$ synthetic experiments where 
We train small transformers for $4$  arithmetic tasks. Following~\cite{power2022grokking, lee2024teaching}, we use digit-wise tokenization where integers are converted to multiple digit-tokens in $\{0,1,\ldots,9\}$, and operators and delimiters such as `$+$', `$=$', `$\times$', `$>$', `$<$', `,' are treated as separate symbol-tokens.

% \vspace{-0.1cm}
\subsection{Tasks and data generation}

\textbf{Addition task.} For the 4-operand addition task, we sample training examples of the following format $a + b + c + d = e$; for instance, 
\begin{equation}\label{eq:plain}
    \underbrace{349}_{a} + \underbrace{102}_{b} + \underbrace{382}_{c} + \underbrace{907}_{d} = \underbrace{1740}_{e} \tag{\text{plain format}}
\end{equation}
where $a,b,c,d$ are uniformly drawn from $\{0,1,\ldots,999\}$, and $e$ is the correct result. Following the setting of \cite{lee2024teaching}, we tokenize every character of the input sequence. For example, ``$349$'' is encoded by three separate tokens, and ``$+$'', ``$=$'' are encoded by two other tokens. We view subskills as predicting individual digits in the output. 

Following~\cite{lee2024teaching}, we also consider the \textit{reverse output format} where the order of digits in the output is reversed. This reverse format is consistent with the human's arithmetic rule since carry is determined by lower digits and used in higher digits.
\begin{equation}\label{eq:reverse}
    \underbrace{349}_{a} + \underbrace{102}_{b} + \underbrace{382}_{c} + \underbrace{907}_{d} = \underbrace{0471}_{e'}. \tag{\text{reverse format}}
\end{equation}
For evaluation, our test data has the same format as training data. We still interpret the digits using the original order; for example, we treat ``1'' in $e'$ as the thousands-place digit.

\textbf{Simple multiplication task.} The format of training examples is $a \times b = c$, where $a$ is a multiple-digit integer, $b$ is a single-digit integer, and $c$ is the correct result. For example,
\begin{equation*}
\underbrace{345901}_{a} \times \underbrace{8}_{b} = \underbrace{2767208}_{c}.
\end{equation*}
%\PX{We currently also use a balanced sampling for multiplication. The weights for different lengths are triangular, for the kth digits, the weight is $\frac{k}{N(N+1)/2}$}
We first sample the length $\ell$ for $a$ from $\{1,\ldots,40\}$ and then use uniform sampling to generate $a$ and $b$: we sample $a$ uniformly from $\{0,1,\ldots,10^\ell-1\}$ and $b$ uniformly from $\{0,1,\ldots,9\}$. 
We also consider the reverse output format where the output digits are arranged in the reverse order; for example, $c$ above is represented by ``8027672'' instead.

% \begin{wraptable}{r}{7cm}
%     \centering
%   \caption{\textbf{Summary of synthetic tasks.}}
%   \label{tab:syn-summary}
%   \begin{center}
%     \begin{small}
%       \begin{sc}
%       \resizebox{0.49\textwidth}{!}{
%         \begin{tabular}{lllll}
%           \toprule
%           &  {\small Addition} & {\small Multi} & {\small Comparison} & {\small Sorting} \\
%           \midrule
%           vocab size    & 15    & 15   & 18    & 15     \\
%           max seq len    & 21  & 85  & 12     & 40   \\
%           sample size & 1M  &  100K  &  45K   & 100K   \\
%           \bottomrule
%         \end{tabular}
%         }
%       \end{sc}
%     \end{small}
%   \end{center}
% \end{wraptable}

% \vspace{-0.5cm}
\begin{wraptable}{r}{0.48\columnwidth}
  % \vspace{-0.7\baselineskip} % optional: move table slightly upward
   \vspace{-0.4cm}
  \centering
  \caption{\textbf{Summary of synthetic tasks.}}
  \label{tab:syn-summary}
  \setlength{\tabcolsep}{4pt} % default is 6pt; smaller saves width
  \renewcommand{\arraystretch}{0.95} % slightly tighter rows
  \small
  \resizebox{\linewidth}{!}{%
    \begin{tabular}{@{}lcccc@{}}
      \toprule
      & Addition & Multi & Comparison & Sorting \\
      \midrule
      vocab size  & 15  & 15   & 18  & 15   \\
      max seq len & 21  & 85   & 12  & 40   \\
      sample size & 1M  & 100K & 45K & 100K \\
      \bottomrule
    \end{tabular}%
  }
\end{wraptable}

%\PX{When doing slicing batch preparation and scaling up experiment, I extended the vocabulary to include all lowercase, uppercase, digits, and operation symbols. The vocabulary size becomes 97.}\PX{Updated the table.}

\textbf{Comparison task.} In the comparison task, the model %We formulate the comparison task as a classification problem where the model 
predicts the relationship token $r \in \{>, <, =\}$ between two comma-separated 4-digit integers a and b. An example input sequence is formatted as:
\begin{equation*}%\label{eq:compare}
\underbrace{5293}_{a}\quad, \quad \underbrace{5241}_{b} \quad \to \quad \underbrace{>}_{r} %\tag{\text{comparison task}}
\end{equation*}
where we tokenize every character as before. We study two sampling strategies. \textbf{Uniform sampling} draws $a$ and $b$ from $\{1000,\ldots,9999\}$, which undersamples hard cases because shared leading digits are rare. \textbf{Balanced sampling} instead uses five equiprobable groups defined by the \textit{Number of Controlled Identical Digits} (NCID), i.e., the length of a guaranteed shared prefix. For NCID $k$, we enforce $a_j=b_j$ for $j\le k$ and sample the remaining digits uniformly. The $k=4$ group consists of identical pairs, so 20\% of training examples contain ``$=$''. See Section~\ref{sec:append-data-gen-comparison} for details.

\textbf{Sorting task.} The training examples take the following format $a,b,c,d \to \mathrm{sorted}(a,b,c,d) $, where $a,b,c,d$ are either 3-digit or 4-digit integers. For instance, 
\begin{equation*}
    \underbrace{9312}_{a}, \underbrace{4661}_{b}, \underbrace{405}_{c}, \underbrace{6252}_{d} \to \underbrace{405}_{c}, \underbrace{4661}_{b}, \underbrace{6252}_{d}, \underbrace{9312}_{a}. 
\end{equation*}
For data generation, we mainly focus on a \textit{doubly balanced} sampling strategy.
%and defer the results of uniform sampling to the Appendix (Sect.~\ref{sec:append-sorting-uniform}). 
This sampling is balanced in two aspects. (i) Length-balanced:
each integer has $1/2$ probability to be 3-digit, and $1/2$ probability to be 4-digit; (ii) Closeness-balanced: similar to the comparison task, the generated dataset consists of 3 equiprobable groups with NCID $\in\{0,1,2\}$.  See Section~\ref{sec:append-data-gen-sorting} for details about data generation and Section~\ref{sec:append-sorting-uniform} for an alternative sampling strategy. %\YZ{Complete the corresponding appendix section.} %\PX{For sorting, currently we are using NCID $\in\{0,1,2\}$.}
%\YZ{I am actually confused. Is the example from the first MSDD group? Do $b$ and $c$ share the highest digit? Perhaps we treat the highest digit of 405 as 0? Explain more in the appendix. Also, I think Peter's closeness grouping is not exactly that same as Darsh's grouping; if so, we need to change the sampling to the same MSDD definition.} \PX{For my closeness grouping, the $k$-th group ($1\le k \le 3)$ only conditions on $a_j = b_j$ for $j < k$ (and does not require $a_k \neq b_k$), and we uniformly samples digits that satisfy this single condition. So the example actually belongs to the first group, but $b_1$ happens to be equal to $c_1$.}

%We call the dataset generated using this approach "doubly balanced" since it is both length-balanced and closeness-balanced.

% \vspace{-0.2cm}
\subsection{Model and Training}

We use the NanoGPT architecture~\citep{Karpathy2022} with standard model configurations. Unless specified otherwise, the models have 6 layers, 6 attention heads, embedding dimension 384, totaling 10.63M parameters. We use absolute positional encoding by default. We adopt the standard autoregressive training using the AdamW optimizer~\citep{loshchilov2018decoupled}, with learning rate $0.001$, batch size 512, and dropout 0.2. For simplicity, each input sequence contains exactly one training example, which is right-padded to a fixed maximum sequence length.
For generation, we use sampling with temperature $0.8$. %When generating outputs, we sample from prediction probabilities with temperature $0.8$. %\PX{Do people call it "right-padded"? We are in fact adding '<pad>' to the end of sequence.} \YZ{Yes, changed.}

\textbf{Alternative setups.} We also consider several variants of our setup, including different input sequence formats, positional encodings, output digit permutations, and decoding schemes; Section~\ref{sec:robustness-alternative-setups} summarizes these robustness checks, and the appendix provides the corresponding figures.

\section{Data-Shaped Non-Human Learning Dynamics}
\label{sec:mechanism}

We begin with addition and simple multiplication because the two tasks admit clean digit-wise skill compositions in the sequential order, where carry connects neighboring digit places. For the two tasks, a model's digit-wise skill acquisition naturally reflects its learning behaviors.  
%they offer canonical, mechanistically analyzable cases. Addition isolates carry-dependent sequential structure in a short output, while multiplication lets us test whether the same style of non-human acquisition persists when the output is much longer and the carry structure is richer.

\subsection{Non-human sequential learning in arithmetic}

\textbf{Addition is learned in reverse skill order.}
Solving addition requires correct modular addition together with carry operations. From a human procedural perspective, one would expect competence to develop from the unit place upward, since carry propagates from lower places to higher ones. In contrast, the models acquire digit-level subskills in the opposite order: learning consistently begins with the most significant digit and then progresses toward lower digits (Figure~\ref{fig:1}). %\PX{Moved "numerical approximation" part to \ref{appendix-approximator}}

\textbf{Multiplication exhibits bidirectional acquisition rather than a single order.}
% The addition result above already shows that the model's preferred learning order is non-human.Multiplication reveals that the issue is broader than a simple reversal. 
We train on the task $a \times b = c$, where $a$ is up to $40$ digits, $b$ is a single digit, and $c$ can have up to $41$ digits. Figure~\ref{fig:mul-digit-err} shows that learning follows neither the human order (from lowest to highest digit) nor a clean reverse order. Instead, the model learns digits in a \emph{bidirectional} manner: competence spreads from the most significant end toward lower digits, while the unit-place digit is learned rapidly in parallel. This pattern is consistent across other output digit permutations, and alternative integer lengths; see Appendix~\ref{appendix_simple_mul}. 
Together, addition and multiplication show that transformers do not acquire arithmetic subskills in orders of standard, canonical algorithms, even in clean arithmetic domains.
% Together, addition and multiplication show that transformers do not naturally solve arithmetic in human-like sequential order, even in clean algorithmic domains.

\begin{figure}[t]
    \centering
    \begin{subfigure}{0.49\textwidth} % Adjust width as needed
       \centering
    \includegraphics[width=\textwidth]{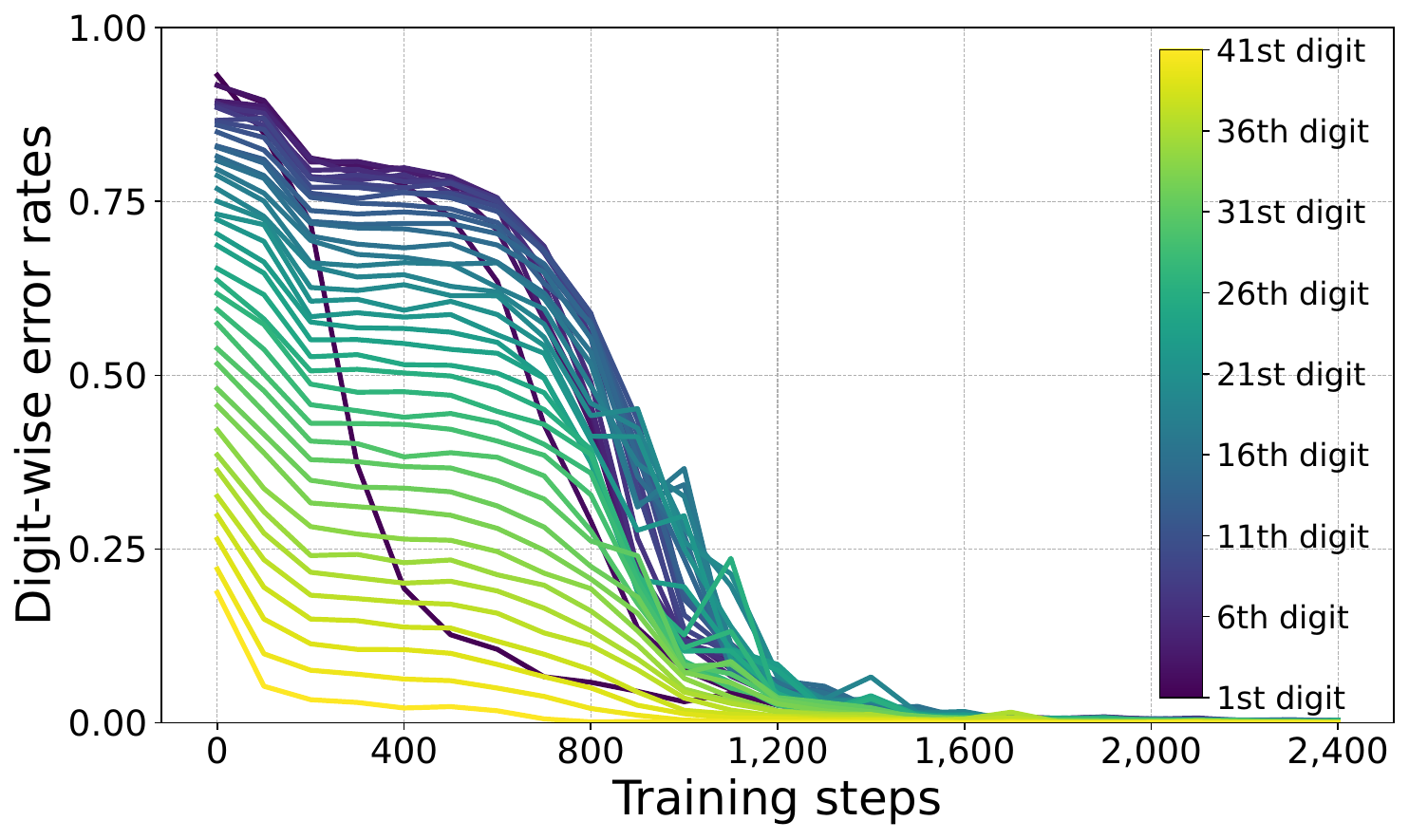} 
        \caption{Plain output format}
    \end{subfigure}%
    \hfill %
    \begin{subfigure}{0.49\textwidth} % Adjust width as needed
        \centering
    \includegraphics[width=\textwidth]{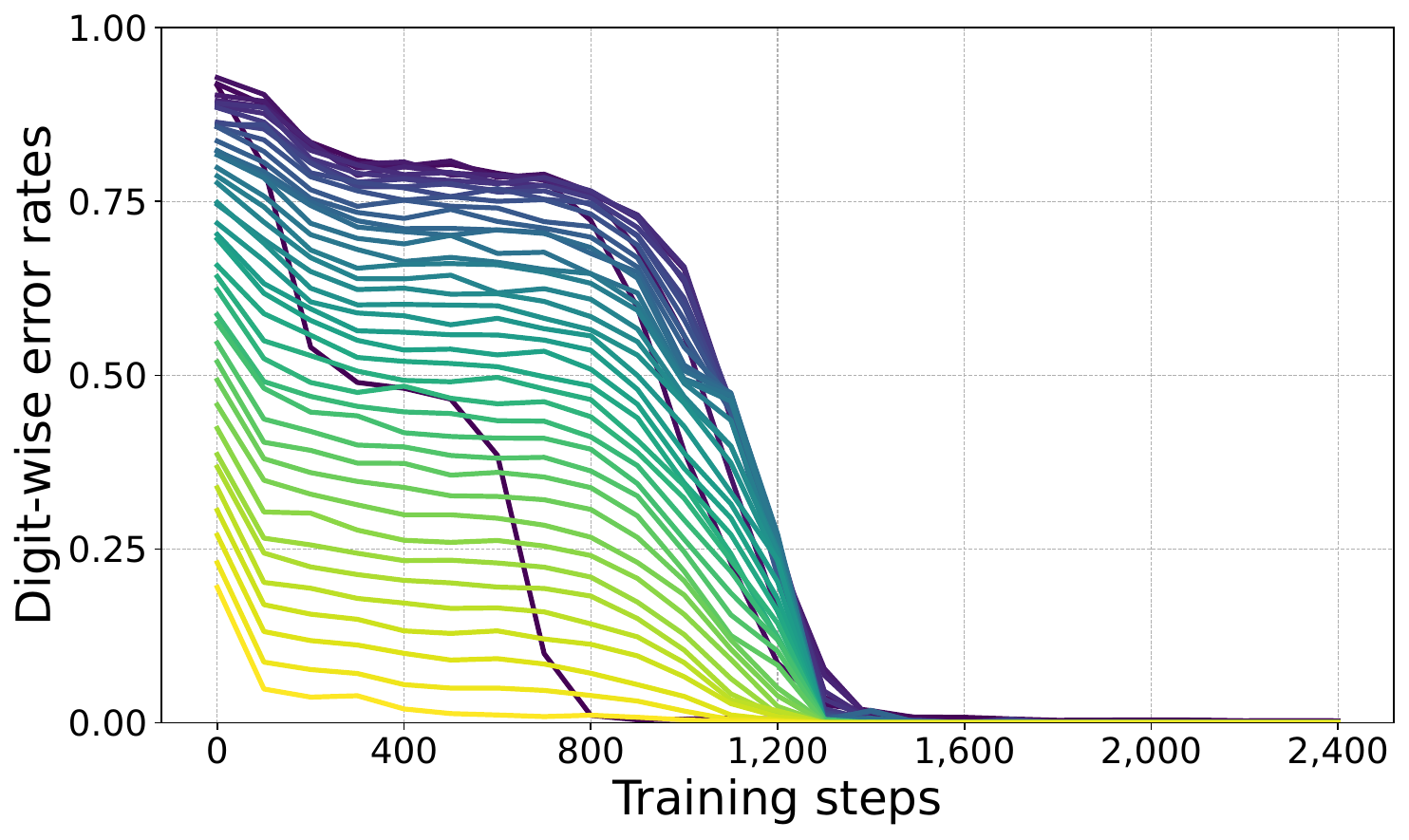} 
        \caption{Reverse output format}
    \end{subfigure}
    \caption{\textbf{Transformers learn digits in two-way order for simple multiplication}. %Training a transformer and evaluating its
    %From the evolution of digit-wise error rates, we find two opposite sequential orders are learned simultaneously. (i) Reverse order: starting from the 41st (highest) digit to the 3rd digit; (ii) Normal order: starting from the 1st (lowest) digit. %\YZ{Please help with the plot. Redraw the legends to make use of the white space on top right and fit into the figure. Use gradient colors with markers for the curves. Change x label to Training steps. Remove title. Change y label to ``Digit-wise error rates''. Reduce the total training steps if necessary, since the model converges early.}\PX{Updated the figure.}
    The model learns from both ends: learning spreads from the most significant end while the units place is acquired early in parallel. Multiplication therefore exhibits a bidirectional acquisition pattern rather than a single human-like or reverse order.
    } 
     % \vspace{-0.5cm}
    \label{fig:mul-digit-err}
\end{figure}

\subsection{The role of output formats}

A natural objection is that the reverse-output format should make addition \emph{easier}, because autoregressive prediction then aligns with the human procedure of computing from the least significant digit upward. To our surprise, Figure~\ref{fig:1} shows the opposite trend in our setting: the plain-output model converges faster overall, while the reverse-output model still learns higher digits first and exhibits sharper phase transitions. This indicates a conflict between the externally imposed format and the model's internally preferred learning trajectory. The key point is not merely that the learned order is ``reversed,'' but that optimization appears to prioritize the easiest correlational signals to fit, even when those signals are poorly aligned with the human algorithm.

%\subsection{Correlational matching as the driver}
\subsection{Evidence for correlational matching}
\label{sec:mi}

To understand these counterintuitive learning orders, we measure the correlational structure available in the training data. Here, we focus on 4-operand addition under the reverse format, leaving similar results of multiplication to Section~\ref{MI_mul}. Because all tokens are discrete, mutual information (MI) provides a natural way to quantify which input-output relationships can serve as learnable signals.

\textbf{Mutual-information metrics.}
Let $a=a_1a_2a_3$ denote one input integer and $e=e_0e_1e_2e_3$ denote the output sum, where smaller subscripts correspond to higher digit places. Inspired by \cite{dziri2023faith}, we first consider digit-wise MI $I(a_i,e_j)$. Under uniform sampling, most unconditional digit-wise MI terms vanish; the highest input digit and the highest output digit are the main exception. To obtain nontrivial statistics for lower places, we therefore consider conditional mutual information $I(a_i,e_i\mid c_{i-1})$, where $c_{i-1}$ is the carry into the next higher digit. 
Intuitively, these conditional MI terms quantify learnable correlational signals that, we hypothesize, only become available to the model once it has partially represented higher-place carry information.

The following theorem summarizes the basic structure of these signals; see proofs in Appendix~\ref{app:proof-thm1}.

\begin{theorem}\label{thm:1}
Under uniform sampling, we have $I(a_1,e_0)>0$ and $I(a_i,e_j)=0$ for all $j>0$. Moreover, $I(a_i,e_i\mid c_{i-1})>0$ for $i=1,2,3$.
\end{theorem}

\textbf{Learning dynamics align with these correlational signals.}
For each evaluation example, we extract the model's predictive distributions and estimate $I(a_1,\widehat p_0)$ and $I(a_i,\widehat p_i\mid K_{i-1})$ for $i=1,2,3$, where $K_{i-1}$ is the local carry, i.e. $c_{i-1}$ computed ignoring incoming carry. 
\footnote{Appendix~\ref{app:actual-carry-check} explains this minor distinction and reports the results of an additional run using the exact $c_{i-1}$ as the conditioning variable.}
% For each evaluation example, we extract the model's predictive distributions over output digits and estimate $I(a_1,\widehat p_0)$ together with $I(a_i,\widehat p_i\mid c_{i-1})$ for $i=1,2,3$.
Figure~\ref{fig:mi-addition} shows that sharp descents in training loss occur precisely when the model's predictions come to match the training-data MI structure. Figure~\ref{fig:addition_mi_wth_digitwise_error} shows the same phase transitions from the perspective of digit-wise error. 
% Taken together, these results support a correlational-matching account: the model first exploits the strongest unconditional signal, and then progressively unlocks lower-place skills once higher-place latent variables, especially carry, are sufficiently represented.
Although these measurements do not directly identify the model's internal representation of carry, their alignment with learning transitions supports a correlational-matching hypothesis: the model first exploits the strongest unconditional signal and then progressively unlocks lower-place skills as higher-place latent variables, especially carry, become sufficiently represented.

% \begin{figure}[t]
%     \centering
% % \begin{subfigure}[b]{0.5\textwidth}
% %   \includegraphics[width=1\linewidth]{Figures/Images/MI-addition.png}
% % \end{subfigure}

% % \medskip % insert a bit of vertical whitespace
% % \begin{subfigure}[b]{0.5\textwidth}
% %   \includegraphics[width=1\linewidth]{Figures/Images/loss-addition.png}
% % \end{subfigure}   

% \begin{subfigure}[b]{0.9\textwidth}
%   \centering\includegraphics[width=0.7\linewidth]{Figures/Images/mi_conditioned_on_carries_plot_1M.pdf}
% \end{subfigure}

% \caption{\textbf{Transformer's learning behavior is captured by mutual information (MI) metrics.} \textbf{Top:} We track the MI metrics $I(a_1,\widehat p_0)$ (highest digit) and $I(a_i,\widehat p_i|c_{i-1})$ (other digits) using the model's prediction probabilities across training (solid curves), which are compared against the corresponding $I(a_1,e_0)$ and $I(a_i,e_i|c_{i-1})$ based on the training data (dashed line). \textbf{Bottom:} The sharp descent of the training loss corresponds to the model's sudden matching with MI metrics at each output digit, indicating skill acquisition. %\YZ{Match the vertical dashed lines.} \PX{Updated the figure.}
%     %\YZ{I downloaded from W\&B. Need to change plot style to match other plots. Title, colors, legend, xlabels. Match MI metrics with the loss curve.} \PX{Updated the figure.}
%     %\PX{Updated to test on MI data of 1M lines}
%     } 
%     \label{fig:mi-addition}
% \end{figure}

% \begin{wrapfigure}

\begin{figure}[t]
  %\vskip 0.2in
  \begin{center}
    \centerline{\includegraphics[width=0.55\textwidth]{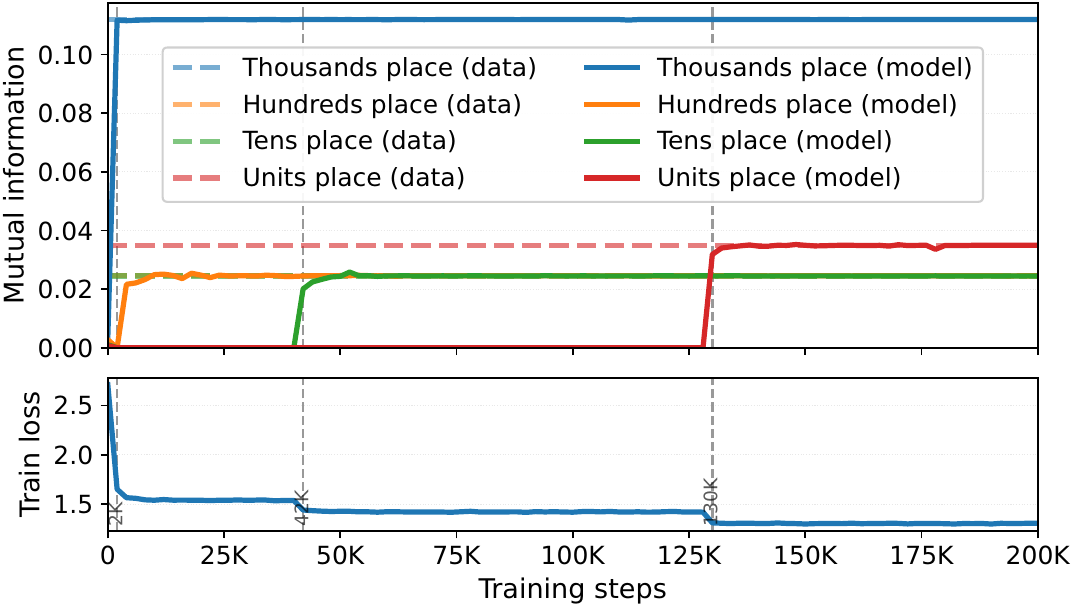}}
    \caption{
      \textbf{Transformer's learning behavior is captured by mutual information (MI) metrics}. 
      \textbf{Top:} We track the MI metrics $I(a_1,\widehat p_0)$ (highest digit) and $I(a_i,\widehat p_i\mid K_{i-1})$ (other digits) using the model's prediction probabilities across training (solid curves), which are compared against the corresponding $I(a_1,e_0)$ and $I(a_i,e_i\mid K_{i-1})$ based on the training data (dashed line).
      % \textbf{Top:} We track the MI metrics $I(a_1,\widehat p_0)$ (highest digit) and $I(a_i,\widehat p_i|c_{i-1})$ (other digits) using the model's prediction probabilities across training (solid curves), which are compared against the corresponding $I(a_1,e_0)$ and $I(a_i,e_i|c_{i-1})$ based on the training data (dashed line).
      \textbf{Bottom:} Sharp descents of training loss are matched with changes of MI metrics at each output digit, indicating skill acquisition.
    }
    \label{fig:mi-addition}
    % \vspace{-0.5cm}
  \end{center}
\end{figure}
% \end{wrapfigure}

\begin{wrapfigure}[19]{R}{0.5\textwidth}
% \begin{wrapfigure}[17]{R}{0.5\textwidth}
\vspace{-0.3cm}
\centering
% Note: The input image uses image_0.png, but your LaTeX code points to a PDF. Use the correct filename for your system.
\includegraphics[width=1\linewidth]{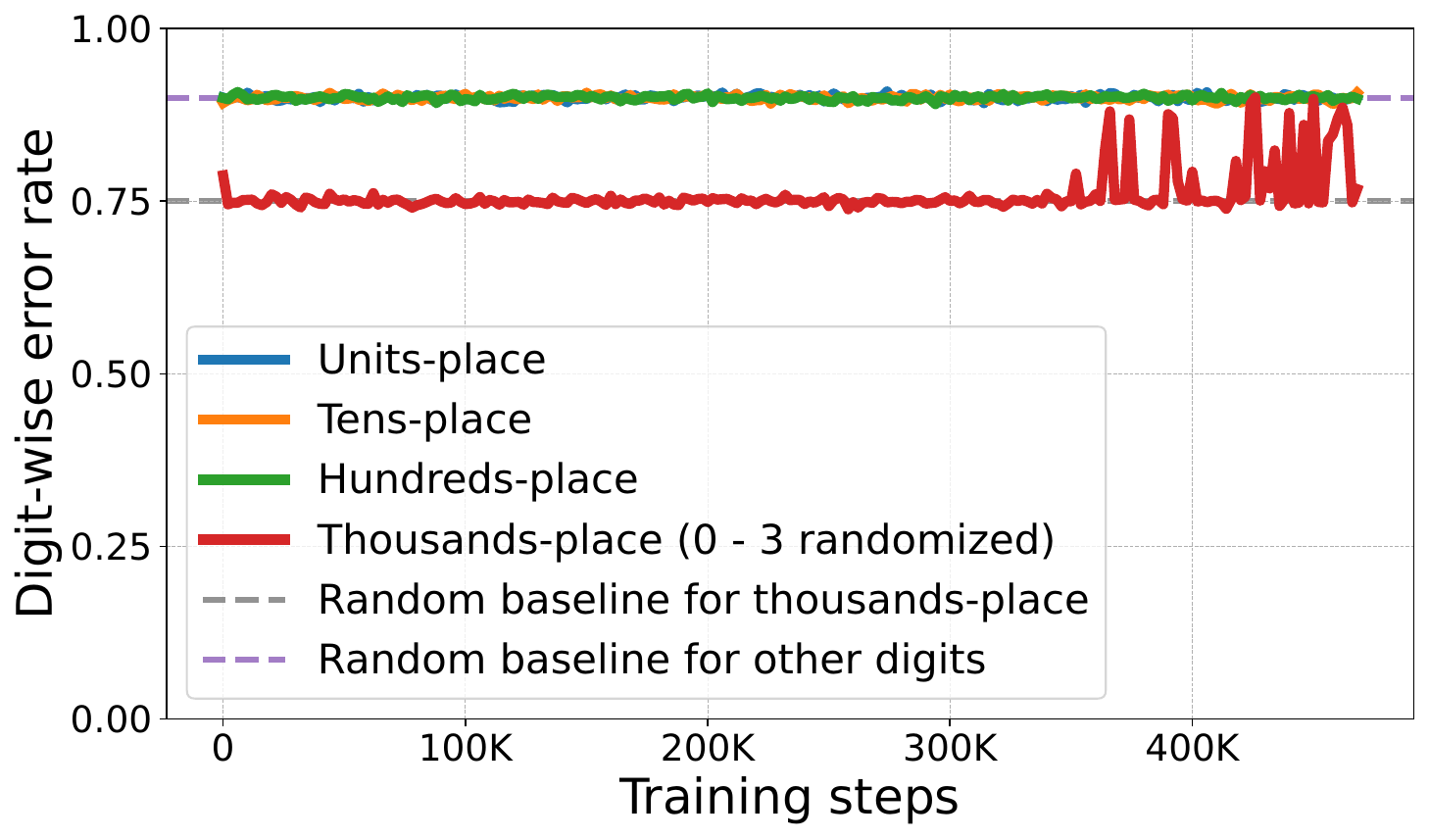}

\caption{\textbf{Ablation experiment shows randomizing higher digits disables learning.} We train transformers with modified training data for addition with the reverse format, where the thousands-place digits are replaced by uniform digits in $\{0,1,2,3\}$. The model fails to learn any lower digit, whose digit-wise error rates are not better than random baseline.}
\label{fig:ablation}
\end{wrapfigure}

\subsection{Conditional dependencies revealed by ablation experiment}

The MI analysis suggests: if higher-place signals are disrupted, lower-place skills should become much harder to learn because the relevant conditional dependencies never become available. We test this by modifying the addition data so that the thousands-place output digit is replaced by a uniformly sampled digit from $\{0,1,2,3\}$, and by running a similar ablation on the hundreds-place digit.

%The result is striking. 
As shown in Figure~\ref{fig:ablation} and Figure~\ref{fig:hundreds-randomized-reverse}, disrupting higher digits prevents or severely delays the acquisition of lower digits. Even after hundreds of thousands of training steps, most lower-place error rates remain near random baseline. 
%This is difficult to reconcile with a picture in which each digit is learned independently. Instead, it supports the view that 
Instead of each digit being learned independently, it supports the view that later-learned subskills depend on conditionally available structure created by earlier, non-human stages of learning.
To rule out noisy gradients introduced by randomized labels as the main explanation, we additionally retrain with correct output results but set the thousands-place loss to zero in the reverse-output setting. The three lower-place errors remain near the 90\% random baseline throughout. (Figure~\ref{fig:addition-thousands-zero-loss} in the appendix).

To summarize, our explanatory hypothesis for the learning dynamics is as follows: optimization matches the most accessible correlational structure first; this induces non-human acquisition orders; and those early-acquired latent structures then gate which additional subskills become learnable later.

% \vspace{-0.2cm}
%\section{Parallel Acquisition and Skill Competition}
\section{Parallel Acquisition, Skill Competition, and Brittleness}
\label{sec:parallel-competition}

%The previous section explained why learning can proceed in non-human orders. We now turn to the main consequences of that behavior. When several subskills are learned in parallel rather than in a sequential fashion,
%a through a stable hierarchy,  they can transiently or persistently interfere with one another. We study this phenomenon in comparison and sorting, where distinct subskills compete during training.
%can be isolated more directly than in addition.

After analyzing the mechanism of non-human orders, we now turn to the consequences of such behavior. When several subskills are learned in parallel rather than  sequentially,
they can transiently or persistently interfere with one another---which may lead to worse accuracy under distribution shifts despite strong in-distribution performance. We study this phenomenon in comparison and sorting.

\subsection{Defining the relevant subskills}

% Table~\ref{tab:subskills_overview} summarizes the decomposition we use throughout this section.
%\YZ{we need to briefly define subskills clearly here, and leave detailed descrption in the appendix.} \PX{Updated}
In comparison, the subskills involved are: (i) Compare the most significant
differing digit; (ii) output ‘=’ if all digits are equal. In sorting, the subskills are: (i) First compare integer length; (ii) if length is tied, compare the most significant differing digit.
See 
Table~\ref{tab:comparison_subskills} and Table~\ref{tab:sorting_subskills} for more detailed subskills definitions for comparison and sorting. We note that for both tasks, the desired rule is hierarchical: later subskills should be invoked only when earlier ones fail to resolve the answer.

% \begin{table}[t]
% \centering
% \caption{\textbf{Two dominant mixing error types in sorting.} Among the four input integers, two numbers, say 4-digit ${\reda}={\reda}_1 {\reda}_2 {\reda}_3 {\reda}_4$ and ${\blueb}={\blueb}_1 {\blueb}_2 {\blueb}_3 {\blueb}_4$, form recombined new integers in the output.}\label{tab:mixing}
% \begin{tabular}{c|cc}
%           & Error case 1 & Error case 2 \\ \hhline{=|==}
% Swapping  &     $ {\reda}_1 {\reda}_2 {\reda}_3 {\blueb}_4$, ${\blueb}_1 {\blueb}_2 {\blueb}_3 {\reda}_4$              &   ${\reda}_1 {\reda}_2 {\blueb}_3 {\blueb}_4$, ${\blueb}_1 {\blueb}_2 {\reda}_3 {\reda}_4$         \\
% Repeating &    $ {\reda}_1 {\reda}_2 {\reda}_3 {\blueb}_4$, ${\blueb}_1 {\blueb}_2 {\blueb}_3 {\blueb}_4$        &     ${\reda}_1 {\reda}_2 {\blueb}_3 {\blueb}_4$,  ${\blueb}_1 {\blueb}_2 {\blueb}_3 {\blueb}_4$         
% \end{tabular}
% \end{table}

\begin{table}[t]
  % \vspace{-0.09cm}
  \caption{\textbf{Two dominant mixing error types in sorting.} Among the four input integers, two numbers, say 4-digit ${\reda}={\reda}_1 {\reda}_2 {\reda}_3 {\reda}_4$ and ${\blueb}={\blueb}_1 {\blueb}_2 {\blueb}_3 {\blueb}_4$, form recombined new integers in the output.}
  \label{tab:mixing}
  \begin{center}
    \begin{small}
      \begin{sc}
        \begin{tabular}{lcccr}
          \toprule
            & Error case 1 & Error case 2 \\ 
          \midrule
Swapping  &     $ {\reda}_1 {\reda}_2 {\reda}_3 {\blueb}_4$, ${\blueb}_1 {\blueb}_2 {\blueb}_3 {\reda}_4$              &   ${\reda}_1 {\reda}_2 {\blueb}_3 {\blueb}_4$, ${\blueb}_1 {\blueb}_2 {\reda}_3 {\reda}_4$         \\
Repeating &    $ {\reda}_1 {\reda}_2 {\reda}_3 {\blueb}_4$, ${\blueb}_1 {\blueb}_2 {\blueb}_3 {\blueb}_4$        &     ${\reda}_1 {\reda}_2 {\blueb}_3 {\blueb}_4$,  ${\blueb}_1 {\blueb}_2 {\blueb}_3 {\blueb}_4$   \\
          \bottomrule
        \end{tabular}
      \end{sc}
    \end{small}
  \end{center}
  % \vspace{-0.5cm}
\end{table}

\begin{figure}[t]
% \begin{wrapfigure}{r}{0.5\textwidth}
    \centering
        \centering
        \includegraphics[width=0.75\textwidth]{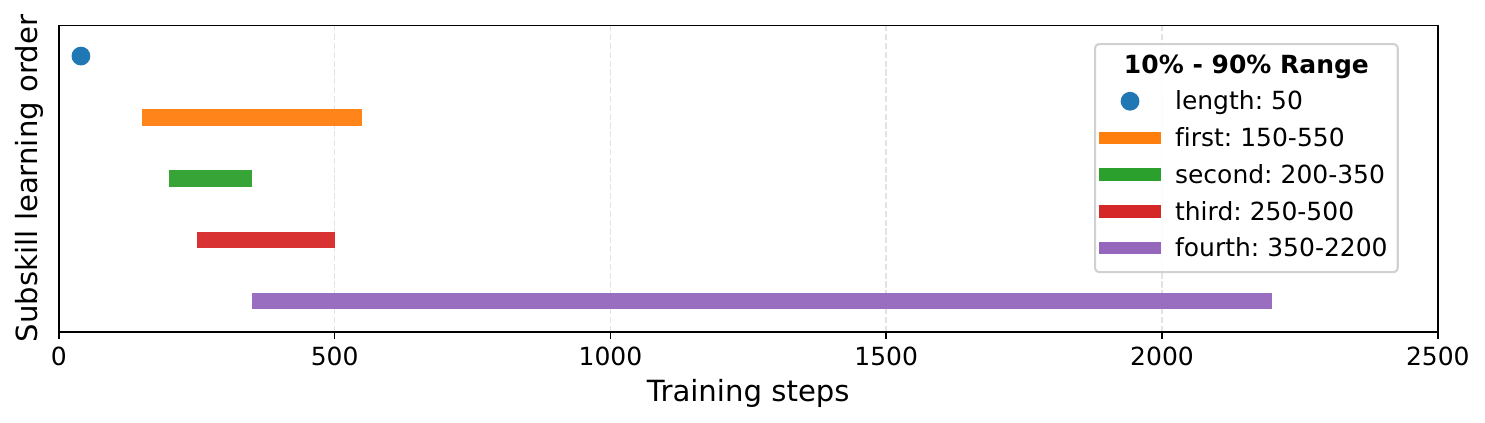} 
    \caption{%\textbf{Transformers learn the crude length first and the rest of digits in parallel}. We train NanoGPTs using doubly balanced dataset and use multiple test datasets to evaluate for what training step range these subskill accuracies are between 10\% and 90\%. In sorting, models learn the crude length (where to output the delimiter, e.g. comma) early in training, and then learn the rest of digits mostly concurrently. %\PX{Updated the figure.}
    \textbf{Progress bars show  concurrent learning of multiple
skills for sorting.} We decompose sorting into five skills: crude lengths of integers, and four digit-wise ranks. The model learns the crude length of input numbers early in training, and then learns the four digit-wise skills concurrently.
    }
    % \vskip -0.2in
    \label{fig:sorting_subskill_order}
\end{figure}
% \end{wrapfigure}

% \PX{ Shall we define subskills using sentences here, and remove Table 2. For comparison, the subskills are "Compare the most significant
% differing digit; output ‘=’ if all
% digits are equal.". For sorting the subskills are "First compare integer length; if
% length is tied, compare the 1st
% digit, then the 2nd, 3rd, and 4th
% digits only as needed." A more detailed subskill definition for comparison and sorting is Table~\ref{tab:comparison_subskills} and Table~\ref{tab:sorting_subskills}}

% \input{Tables/subskills_overview.tex}

\subsection{Comparison: parallel acquisition can create temporary competition}

For the comparison task, we evaluate learning dynamics using \emph{contrast pairs}: test examples $(a,b)$ where the two numbers differ at only one digit position. These controlled shifts let us separate the model's competence at each digit place. Figure~\ref{fig:comparison_staircase_test} in the appendix shows no clear sequential order on same-distribution tests, suggesting that multiple comparison-related subskills are acquired in parallel.

\begin{wrapfigure}[18]{R}{0.5\textwidth}
    \vspace{-0.4cm}
    \centering
    \includegraphics[width=0.952\linewidth]{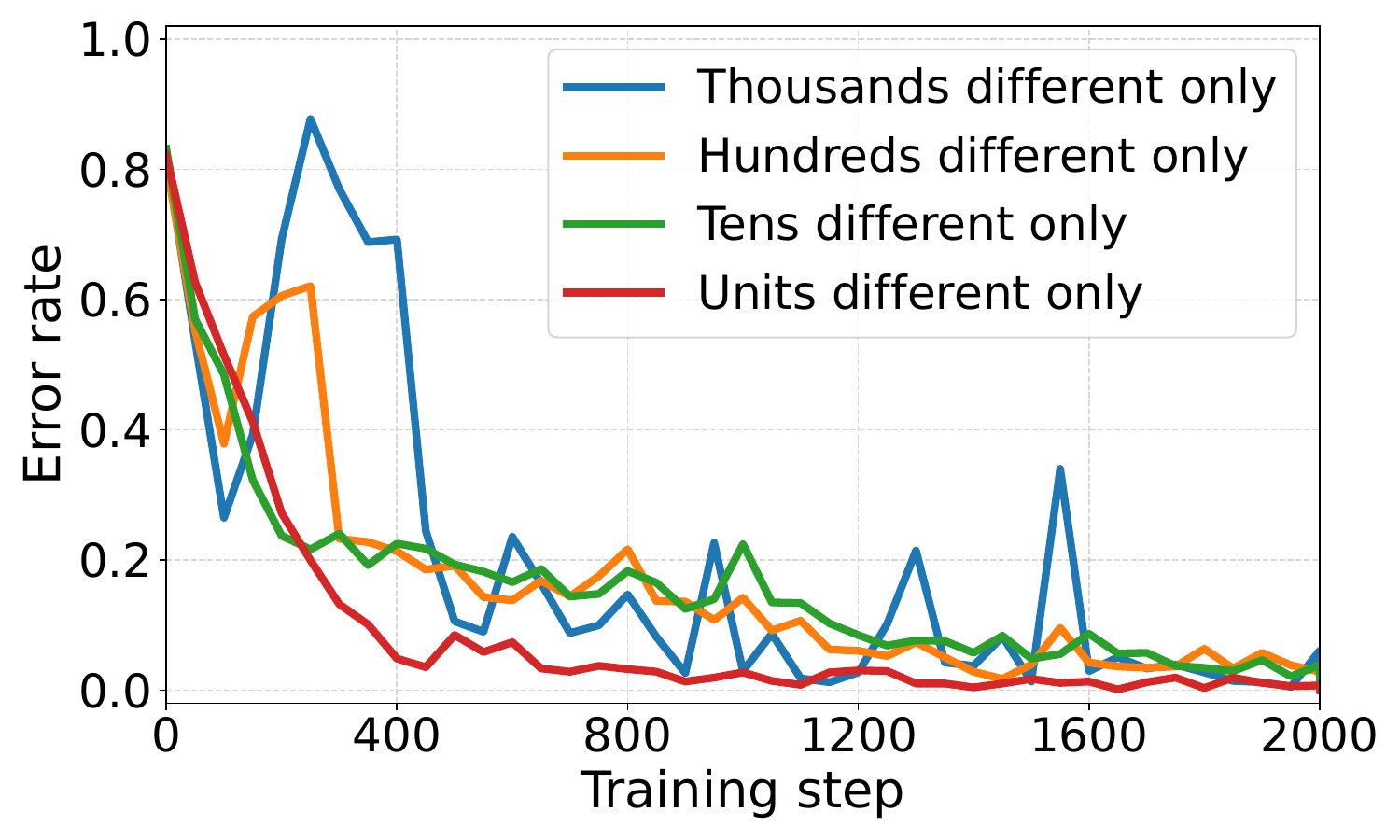}
    \caption{\textbf{Contrast pairs reveal competition between learning  subskills for the comparison task}. 
    %We use controlled test data to evaluate the model's learning dynamics of skills. \textbf{Left}: using test data of the same distribution as training data (grouped by NCID), we find that all skills are learned simultaneously. \textbf{Right}: 
    We use 4 test groups of contrast pairs (differing in single digits) to evaluate model's ability to distinguish similar integers. The increase of errors at 200 steps, particularly at thousands-place digits, suggests that learning ``='' temporarily interfered with learning to compare of thousands-place digits.}
    \label{fig:comparison-results}
    % \vspace{-0.1cm}
\end{wrapfigure}

The contrast pairs reveal what this parallelism costs. Figure~\ref{fig:comparison-results} shows a non-monotone trend for the thousands-place test group: the model first learns to compare the highest differing digit, then temporarily becomes worse when it starts learning to predict ``='' from the other identical digits, and only later reconciles the two behaviors. In other words, acquiring one useful subskill can briefly interfere with another. This is a minimal example of \emph{skill competition}: the model is not simply accumulating independent capabilities, but is reorganizing partially overlapping behaviors whose interaction can be antagonistic before it becomes stable.

\subsection[Sorting: parallel skills yield mixing errors and brittleness under distribution shift]{Sorting: parallel skills yield mixing \\ errors and brittleness under distribution \\ shift}
\label{sec:sorting}

Sorting requires a larger hierarchy of subskills, including identifying integer length and applying digit comparisons at the correct position. Figure~\ref{fig:sorting_subskill_order} shows that these subskills are again learned in parallel rather than in the clean hierarchical order that the task semantics would suggest.

This parallel acquisition produces a distinctive failure mode. The model often fails to preserve input integers as coherent wholes and instead outputs new integers assembled from fragments of different inputs. Table~\ref{tab:mixing} shows the dominant forms of these \emph{mixing errors}. These failures are especially informative because they do not look like simple arithmetic mistakes; they reveal a brittle composition in which partially learned subskills are recombined inconsistently.

To test whether such competition matters for robustness, we construct sorting examples where two comparison cues either \emph{agree} or \emph{conflict}. Table~\ref{tab:conflict} shows that the model performs markedly worse in the conflicting group, indicating that oppositely directed subskills amplify one another's interference. This makes the connection to shattered compositionality explicit: non-human parallel acquisition does not just change \emph{how} the model learns, but creates concrete brittleness under distribution shift.

\begin{table*}[t]
  \caption{\textbf{Evaluating the interaction of two skills for the sorting task}. Both groups (as test data under distribution shifts) yield higher error rates compared with the final test error (1.7\%). Furthermore, the model produces mixing errors (swapping or repeating) significantly more often on conflicting group than on agreeing group, showing two skills with opposite effects worsen performance.
  ``Mixing Error'' and ``Error SD'' denote the mean and sample
standard deviation, respectively, of the mixing-error rate
across the final 10 evaluated checkpoints.
  %\YZ{Defer the dynamics of repeating errors to the appendix.}
  }
  \label{tab:conflict}
  % \vskip -0.1in
  \begin{center}
    \begin{sc}
      % Scales the tabular environment to the linewidth
      \resizebox{\linewidth}{!}{
        \begin{tabular}{lcccc}
          \toprule
          Group & Example input & Condition & Mixing Error & Error SD \\
          \midrule  
          
            Conflicting & \texttt{1000,6{\color{blue}5}8{\color{red}9},6{\color{blue}6}8{\color{red}2},9999} & conflicting in $({\color{blue}b}_2,{\color{blue}c}_2)$ and $({\color{red}b}_4,{\color{red}c}_4)$ & $6.63\%$ & $0.46\%$ \\
            Agreeing & \texttt{1000,6{\color{blue}5}8{\color{blue}2},6{\color{blue}6}8{\color{blue}9},9999} & agreeing in $({\color{blue}b}_2,{\color{blue}c}_2)$ and $({\color{blue}b}_4,{\color{blue}c}_4)$ & $3.61\%$ & $0.58$\% \\

          \bottomrule
        \end{tabular}
      }
    \end{sc}
  \end{center}
  % \vskip -0.1in
\end{table*}

% \vspace{-0.1cm}
%\section{Robustness and the Persistence of Shattered Compositionality}
\section{Persistence of Shattered Compositionality}
\label{sec:robustness}

A natural question is whether shattered compositionality is merely an artifact of one architecture, one training recipe, or one random initialization. We therefore test a range of standard variations and interventions, including additional runs with different random seeds, different positional encodings, larger models, pretraining, and scratchpad-based reasoning.

\subsection{Alternative architectures and experimental setups}
% \subsection{Alternative setups: input format, positional encoding, output organization, decoding, and random seeds}

\label{sec:robustness-alternative-setups}

The main arithmetic learning patterns are not fragile artifacts of a single implementation choice. For addition, the reverse-learning pattern is consistent across several alternative setups, including different positional encodings, input sequence formats, reversing digit order of input operands, output digit permutations, greedy decoding, and different random seeds; see Appendix~\ref{appendix-addition}. 
The encoder-only next-token prediction (ENTP) variant exhibits a
reverse-dominant, partially non-monotonic acquisition pattern rather
than a purely highest-to-lowest order: higher-place digits improve
first, but the units-place error later falls below the tens-place
error (Appendix~\ref{app:addition-entp}).
For multiplication, the same qualitative behavior also holds under alternative output digit permutations and alternative number lengths; see Appendix~\ref{appendix_simple_mul}. For comparison, the contrast-pair interference pattern also persists across different random seeds; see Appendix~\ref{appendix-comparison}. Together, these checks suggest that shattered compositionality is not a brittle consequence of one particular token order, positional scheme, decoding rule, or random initialization.

% Together, these checks suggest that shattered compositionality is not caused by one particular token order, positional scheme, or decoding rule.

% \subsection{Alternative positional encodings and output organizations}

% The reverse-learning pattern in addition is not tied to one positional encoding. Figure~\ref{fig:addition_postional_encoding} shows that when we replace absolute positional embeddings with RoPE or T5-style relative biases, the model still learns higher-place digits before lower-place digits under reverse-format addition. \PX{If there is not enough space, we could also  move Figure 7 to Appendix.} Additional appendix results show the same qualitative behavior under alternative output organizations for addition and multiplication; see Appendix~\ref{sec:append-alt-1} and Appendix~\ref{appendix_simple_mul}. These checks suggest that the phenomenon is not a fragile artifact of one specific token order or positional scheme.

% \input{Figures/addition_positional_encoding_main}

\subsection{Does model scaling fix it?}

The non-human acquisition pattern persists under model scaling. We scale NanoGPT from roughly 10M parameters to 20M and 100M parameters, and we also fine-tune a pretrained Pythia-1B model~\citep{biderman2023pythia} on the same addition task. Figure~\ref{fig:scaling_results} shows that the larger models and the pretrained model still display the same qualitative reverse-learning pattern observed in the base experiment. In other words, more parameters and prior exposure to broad text corpora do not automatically induce the human procedural decomposition for addition.

% In other words, more parameters and prior exposure to broad text corpora do not automatically induce the human procedural decomposition for addition.

% \begin{figure}[t]
%     \centering
%     \begin{subfigure}{0.99\textwidth} % Adjust width as needed
%                 \centering
%         \includegraphics[width=\linewidth]{Figures/Images/digitwise_bars_subplot.pdf}
%     \end{subfigure}%
%     \hfill % Adds horizontal space between subfigures
%     \caption{
% \textbf{Effects of model scale and scratchpad The authors declare no competing interests.on digit-wise error in addition.}
% Each panel reports digit-wise error rates (units, tens, hundreds, thousands) at a fixed training step.
% Panels (a)(b): pure model scaling when training NanoGPT from scratch. Panel (c): fine-tuning a pretrained Pythia-1B model.
% Panels (d)(e): training with different scratchpad (reasoning) formats. \PX{Do we want a bigger font size for the labels/ticks? I know it's a bit crowded to put all the five subfigures in a single line. But perhaps make them just a little bigger?}
% }
%     \label{fig:scaling-experiments}
% \end{figure}

\begin{figure*}[ht]
  %\vskip 0.2in
  \begin{center}
    \centerline{\includegraphics[width=0.8\textwidth]{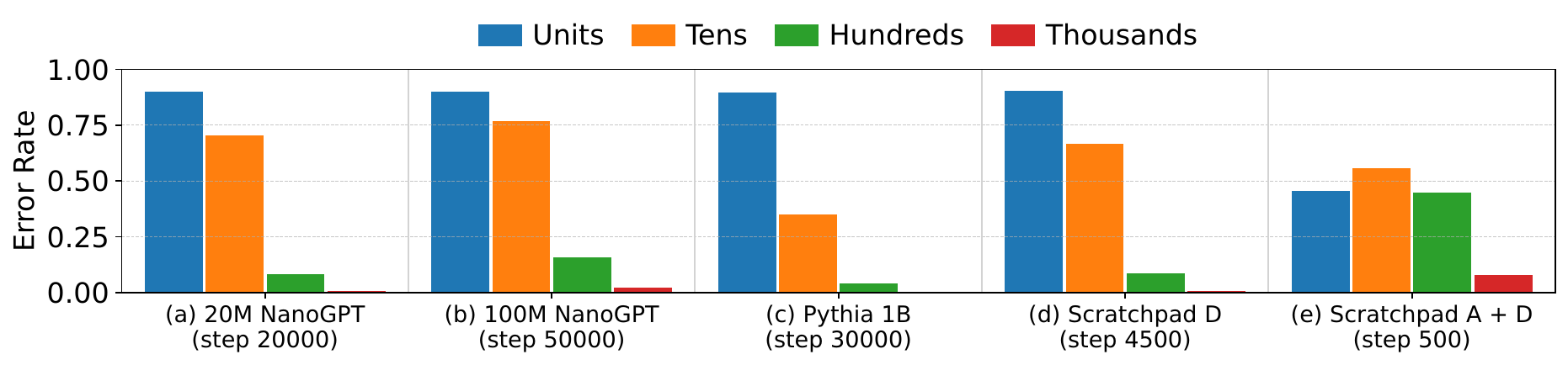}}
    \caption{
      \textbf{Effects of model scale and scratchpad on digit-wise error in addition}.
      Each panel reports digit-wise error rates (units, tens, hundreds, thousands) at a fixed training step.
Panels (a) and (b): pure model scaling when training NanoGPT from scratch. Panel (c): fine-tuning a pretrained Pythia-1B model.
Panels (d) and (e): training with different scratchpad (reasoning) formats. 
% \PX{Do we want a bigger font size for the labels/ticks? I know it's a bit crowded to put all the five subfigures in a single line. But perhaps make them just a little bigger?}
% \YZ{NanoGPT instead of nanogpt. D Scratchpad, D+A Scratchpad instead of Reasoning A, Reasoning A+D.}
    }
    \label{fig:scaling_results}
    % \vskip -0.2in
  \end{center}
\end{figure*}

\subsection{Does scratchpad fix it?}

Reasoning traces change the training distribution, so they are a natural candidate intervention. We therefore train on two scratchpad formats that explicitly decompose the addends by place value:
\begin{align*}
        811& +856+239+313 = 100(8+8+2+3) +  10(1+5 \\&+3+1) +   1(1+6+9+3) = 2219
        \tag{D scratchpad} \\
        811&+856+239+313 =  100(8+8+2+3) +  10(1+5\\&+3+1) +  1(1+6+9+3) \notag
    =100(21)+10(10)+1(19) = 2219 \tag{D+A scratchpad}
\end{align*}

Scratchpad helps neither by default nor uniformly. As shown in Figure~\ref{fig:scaling_results} panels (d) and (e), the exact dynamics depend strongly on the scratchpad format, but the resulting trajectories still fail to recover a clean human-like learning order. Scratchpad therefore reshapes the route the model takes through training, yet it does not by itself eliminate shattered compositionality.

% \vspace{-0.1cm}
%\section{External Robustness Test: Implications for Modern LLMs}
%\section{Arithmetic brittleness: Implications for Modern LLMs}
\section{External Validation of Arithmetic Brittleness for Modern LLMs}
\label{sec:external-robustness}

The synthetic tasks above let us isolate mechanism and consequences. We now ask whether a qualitatively similar brittleness appears in modern LLMs under a mild distribution shift.
%We now evaluate LLMs suffer from a qualitatively similar consequence of brittleness under a mild distribution shift.

\textbf{Defining a distribution shift for evaluation.} In Section~\ref{sec:parallel-competition}, 
%contrast pairs are sampled from a different distribution from training data, which is a type of distribution shift. Since LLMs are not based on digit-wise tokenization as in our synthetic setting, we consider a different distribution shift: increasing the number of operands in arithmetic. We will directly evaluate a variety of LLMs on such distribution shift instead of training them.
contrast pairs are sampled from a different distribution than training data, which is one form of distribution shift. Since modern LLMs are not evaluated through digit-wise tokenization as in our synthetic setting, we instead consider a prompt-level compositional shift: increasing the number of operands in arithmetic word problems. We directly evaluate pretrained LLMs under that shift rather than retraining them.

\subsection{Clause-extended GSM8K}

We curate 50 templates derived from GSM8K~\citep{cobbe2021training}
% We curate 50 GSM8K templates 
by selecting questions whose solution involves multi-operand addition and then extending each question with additional clauses. For example, a template may begin with ``Billy sells DVDs. He has 8 customers on Tuesday. His first 3 customers buy one DVD each,'' and we append $1$--$6$ extra clauses of the form ``his next $m$ customers buy $k$ DVDs each.'' These edits preserve the general narrative style of GSM8K while increasing the number of operands in the final summation. Details of the template construction are given in Appendix~\ref{sec:append-clauses}.

% \begin{figure}[t]
%     \centering
%     \begin{subfigure}{0.99\textwidth} % Adjust width as needed
%                 \centering
%         \includegraphics[width=\linewidth]{Figures/Images/LLM_50_questions.pdf}
%     \end{subfigure}%
%     \hfill % Adds horizontal space between subfigures
%     \caption{
% \textbf{LLMs are susceptible to distribution shifts (extra clauses) of GSM8K data.}
% We evaluate 10 open-source LLMs on 50 curated templates built on selected questions from GSM8K datasets. Each template admits $k$ new clauses ($k=1,\ldots,6$) as addends in a reasoning step. We find that both instruct-finetuned and reasoning-finetuned models have decreasing accuracy as we introduce more clauses. %There is a clear trend that as we increase the number of added clauses, the performance worsens for each model.
% }
%     \label{fig:LLM_50_questions}
% \end{figure}

\begin{figure*}[ht]
  %\vskip -0.2in
  \begin{center}
    \centerline{\includegraphics[width=\textwidth]{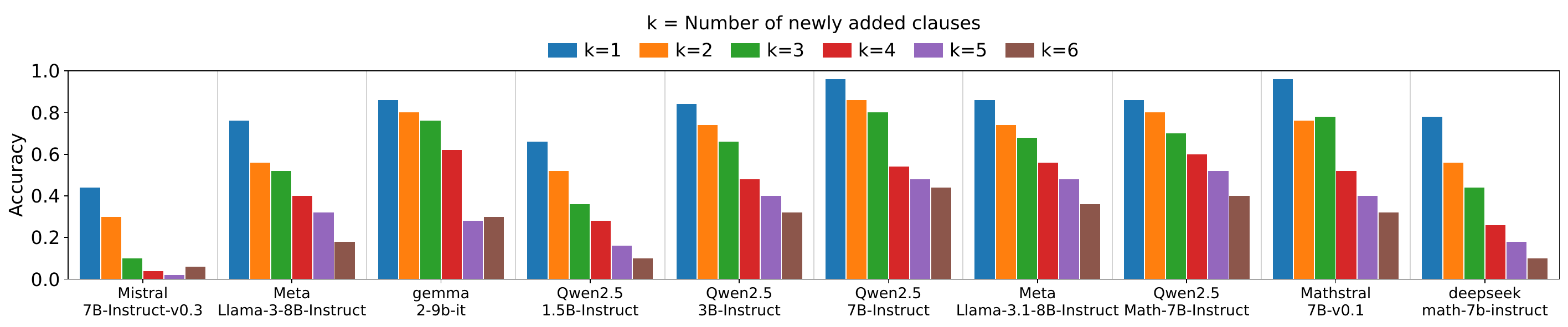}}
    \caption{
      \textbf{LLMs are susceptible to distribution shifts (extra clauses) of GSM8K data}.
      %\RU{This figure is eating up a lot of vertical space. Could we try text-wrapping the X axis text into 1-2 horizontal lines, instead of the long diagonal line?} \PX{Good idea. I've updated.}
      We evaluate 10 LLMs on 50 curated templates built on selected questions from GSM8K datasets. Each template admits $k$ new clauses ($k=1,\ldots,6$) as operands. %in a reasoning step. 
      We find that both instruct-finetuned and reasoning-finetuned models have decreasing accuracy as we introduce more clauses. 
    } 
    % \vskip -0.2in
    \label{fig:LLM_50_questions}
  \end{center}
\end{figure*}

\subsection{Connection to shattered compositionality}

Figure~\ref{fig:LLM_50_questions} shows that all 10 evaluated open-source LLMs deteriorate as the number of added clauses increases, including instruction-tuned and math-tuned models. A common error is omitting or repeating a term in the final summation (Appendix~\ref{gsm-error-type}), resembling the mixing errors in Section~\ref{sec:parallel-competition}. While this does not prove that LLMs implement the exact same mechanisms as our small transformers, it supports the broader picture: models may rely on brittle, non-human compositions that work well for in-distribution data but fail when the compositional demand is mildly perturbed.

% Figure~\ref{fig:LLM_50_questions} shows that all 10 evaluated open-source LLMs deteriorate as the number of added clauses increases, including instruction-tuned and math-tuned models. One common error type is a missing or repeated term in the final summation (Appendix~\ref{gsm-error-type}), which bears similarity to the mixing errors we identified in Section~\ref{sec:parallel-competition}. This does not prove that LLMs implement the exact same mechanisms as our small transformers. However, it is qualitatively consistent with the broader picture: models may rely on brittle, non-human compositions that work well for in-distribution data but fail when the compositional demand is mildly perturbed. We therefore present the GSM8K results as an external robustness test of the same \emph{pattern}, rather than as a one-to-one mechanistic identification.

% \vspace{-0.2cm}
\section{Related Work}
\label{sec:related-work}

\textbf{Static analyses of arithmetic and reasoning.} Mechanistic interpretability identifies task-relevant components in trained neural networks~\citep{elhage2021mathematical, olsson2022context, hendel2023context,  doi:10.1073/pnas.2417182122}.
Studies have shown that specific attention heads, MLPs, or low-dimensional activation subspaces coordinate computation~\citep{stolfo2023mechanistic, zhang2024interpreting, zhou2024pretrained, hu2025understanding}. 
Notably,~\cite{lindsey2025biology} identify complementary mechanisms for addition---one encoding coarse magnitude and another handling unit-place computation.~\cite{nikankin2025arithmetic} show that arithmetic is implemented via a ``bag of heuristics'', with constituent skills represented in parallel rather than as a unified algorithm.
%These results suggest that arithmetic behavior arises from multiple interacting mechanisms rather than a single procedural rule, consistent with non-human internal representations.
However, nearly all such analyses operate on \emph{static, trained models}, so they cannot explain how mechanisms emerge, what determines learning order, or how partially learned skills cause downstream failures.
% However, nearly all such analyses operate on \emph{static, trained models}. Thus, they cannot explain how these mechanisms emerge during training, how learning order is determined, or how interactions between partially learned skills cause downstream failures.

% \paragraph{Mechanistic interpretability (MI).} A growing literature \citep{elhage2021mathematical, olsson2022context, hendel2023context,  doi:10.1073/pnas.2417182122} on MI aims to uncover the internal mechanism of neural networks by identifying task-relevant components of a model. For basic arithmetic tasks, \cite{stolfo2023mechanistic, zhang2024interpreting, zhou2024pretrained} showed that specific attention heads and MLPs coordinate to carry out the calculation, likely through a low-dimensional subspace \cite{hu2025understanding}. Pertinent to our findings, \cite{lindsey2025biology} discovered two complementary features for addition: one feature determines the rough range of the solution, while the other handles unit-place addition similar to a lookup table, which suggests the learned mechanism is different from human's rules. Further, \cite{nikankin2025arithmetic} showed that LLMs use a ``bag of heuristics'' to solve arithmetic, meaning that constituent skills are represented in a parallel fashion, which is consistent with our findings about non-human learning order. However, most MI analysis focuses on static pre-trained LLMs due to prohibitive training costs, whereas our results shed light on a model's learning dynamics.

\textbf{Learning dynamics on arithmetic without decomposition.}
Synthetic arithmetic tasks have long been used to probe learning dynamics in neural networks.
Grokking is identified in \cite{power2022grokking}, where generalization emerges abruptly after extended overfitting, inspiring follow-up work on representation dynamics and weight evolution~\citep{nandaprogress, liu2022towards, mallinar2024emergence}. 
Other studies examine how input format, curriculum design, or recursive training affect arithmetic generalization~\citep{lee2024teaching, lee2025selfimproving}.
While these works characterize \emph{when} models generalize or \emph{whether} specific training strategies help, they typically treat arithmetic as a monolithic task. %They do not decompose problems into constituent subskills or track how those subskills are acquired over time.
Consequently, existing training-dynamics analyses cannot explain learning order or interference effects—phenomena central to shattered compositionality.

% \paragraph{Training dynamics for arithmetic.} \cite{power2022grokking} identified an intriguing phenomenon known as \textit{grokking}, where a transformer trained on modular arithmetic initially overfits to the training data before suddenly learns to generalize on test data. This phenomenon inspired a line of papers on understanding the dynamics of model's weights \citep{nandaprogress, liu2022towards, mallinar2024emergence}. Synthetic arithmetic tasks are also used to study the impact of input format \citep{lee2024teaching}, recursive training \citep{lee2025selfimproving}, etc. Compared with existing literature, we decomposed arithmetic tasks into base skills and showed that reverse/parallel skill acquisition is an intrinsic learning characteristic of transformers.

\textbf{Skill composition and its limits in language models.}
Recent theoretical and empirical work suggests that LLMs solve complex tasks by composing simpler skills.~\cite{arora2023theory} propose skill graphs whose nodes are reusable subroutines, while~\cite{lindsey2025biology} present attribution-graph analyses showing that learned compositions can involve cooperative or competing mechanisms. At the same time, several studies argue that statistical matching fundamentally limits LLM compositional generalization~\cite{dziri2023faith, abbe2024far, abedsoltan2025task, wang2025learning}.
Relatedly, the Fractured Entangled Representation (FER)
hypothesis \citep{kumar2025questioning} proposes that learned knowledge
can be fragmented across contexts and entangled with other
functions, and highlights learning order as a possible
contributor. Whether the acquisition patterns we observe
give rise to such representations remains an open question.
Our work complements these perspectives by directly analyzing
the emergence and interaction of subskills during training, and how competition between partially learned
skills contributes to brittleness under distribution shift.
\section{Limitations and Future Directions}
\label{sec:conclusion}

In this work, we study a failure mode of compositionality through the training dynamics of small transformers on synthetic tasks, using carefully designed test data and model-agnostic metrics. 
Important next steps are to develop practical LLM diagnostics
and test interventions that may reduce compositional brittleness.
Beyond the scratchpad formats and positional encodings studied
here, potential directions include curricula that build component
skills before requiring their composition, finer-grained
supervision of column sums and carry propagation, and
architectural or positional biases aligned with task structure.
Open-ended learning that encourages the preservation and reuse
of component skills is another possible
direction~\citep{kumar2025questioning}.
Whether these approaches lead to more robust skill composition
remains an open question.
% Important next steps are to develop practical LLM diagnostics and test interventions, such as curriculum design, that reduce shattered-composition brittleness. 

Our analysis is intentionally black-box: it reveals non-human skill acquisition and its consequences, but not the circuits implementing them. A useful future direction is to combine white-box mechanistic analysis with the training-dynamics view, so that model operations, data distribution, and learning trajectories can be studied together.

\begin{ack}
Y.Z.~is partially supported by NSF-DMS grant 2412052 and by a Coefficient Giving (formerly Open Philanthropy) grant. We also thank Zaid Harchaoui and Abhishek Panigrahi for helpful discussions.
The authors declare no competing interests.
\end{ack}

\bibliographystyle{unsrt}
\bibliography{refs}

%%%%%%%%%%%%%%%%%%%%%%%%%%%%%%%%%%%%%%%%%%%%%%%%%%%%%%%%%%%%%%%%%%%%%%%%%%%%%%%
%%%%%%%%%%%%%%%%%%%%%%%%%%%%%%%%%%%%%%%%%%%%%%%%%%%%%%%%%%%%%%%%%%%%%%%%%%%%%%%
% APPENDIX
%%%%%%%%%%%%%%%%%%%%%%%%%%%%%%%%%%%%%%%%%%%%%%%%%%%%%%%%%%%%%%%%%%%%%%%%%%%%%%%
%%%%%%%%%%%%%%%%%%%%%%%%%%%%%%%%%%%%%%%%%%%%%%%%%%%%%%%%%%%%%%%%%%%%%%%%%%%%%%%
\newpage
\appendix
\onecolumn

\section{Experiment Settings}

\subsection{Models and training}

We performed our primary experiments using a transformer based on the NanoGPT architecture \citep{Karpathy2022}. Our main model consists of 6 layers and 6 attention heads, with an embedding dimension of 384.
We use cosine learning rate scheduler with a linear warmup. We used the same training hyperparameters for all NanoGPT and scaled NanoGPT experiments. Table \ref{tab:hyperparams_comparison} details these values.

% We performed our experiments on transformers based on the NanoGPT architecture \citep{nanogpt}. Unless specified otherwise, our model has 6 layers, 6 head, using embedding dimension of 384. 
% For all the experiments done on NanoGPT and scaled NanoGPT model we used same training hyperparameters. Table \ref{tab:hyperparameters} contains detailed values of hyperparameters. 
% are detailed in Table \ref{tab:hyperparameters}during training  than fine-tuning on Pythia-1B we used the same hyper
% During training, we use a batch size of 512, a learning rate of 1e-3, beta1 = 0.9, beta2 = 0.98. See Table~\ref{tab:hyperparameters} for the list of key hyperparameters.

%\subsection{Training}
%\DS{Do we need to add this detail?}
%The block size is subject to specific tasks, varying with the length of one single example for that task. For example, we used a block size of 32 for 4 operand addition.

\begin{table}[h]
    \centering
    \caption{Architectural specifications for the scaled NanoGPT models. Training hyperparameters remained consistent across all sizes.}
    \label{tab:model_scaling}
    \begin{tabular}{l c c c}
        \toprule
        \textbf{Parameter} & \textbf{10M (Baseline)} & \textbf{20M} & \textbf{100M} \\
        \midrule
        Layers ($n_{layer}$) & 6 & 12 & 14 \\
        Attention Heads ($n_{head}$) & 6 & 12 & 12 \\
        Embedding Dim ($d_{model}$) & 384 & 384 & 768 \\
        Total Parameters & $\sim$10M & $\sim$20M & $\sim$100M \\
        \bottomrule
    \end{tabular}
\end{table}

\begin{table}[h]
    \centering
    \caption{Hyperparameters for the NanoGPT (training from scratch) and Pythia-1B (fine-tuning) experiments.}
    \label{tab:hyperparams_comparison}
    \begin{tabular}{l c c}
        \toprule
        \textbf{Hyperparameter} & \textbf{NanoGPT (Default)} & \textbf{Pythia-1B (Fine-tune)} \\
        \midrule
        Optimizer & AdamW & AdamW \\
        LR Schedule & Cosine & Cosine \\
        Learning Rate & $1 \times 10^{-3}$ & $1 \times 10^{-4}$ \\
        Min Learning Rate & $1 \times 10^{-4}$ & $1 \times 10^{-5}$ \\
        Weight Decay & 0.1 & 0.1 \\
        Betas ($\beta_1, \beta_2$) & (0.9, 0.98) & (0.9, 0.98) \\
        Warmup Iterations & 100 & 500 \\
        LR Decay Iterations & Max Iters & Max Iters \\
        Gradient Clipping & 1.0 & 1.0 \\
        Batch Size & 512 & 512 \\
        Dropout & 0.2 & 0.1 \\
        \bottomrule
    \end{tabular}
\end{table}

\subsection{Data generation for comparison task}\label{sec:append-data-gen-comparison}

Let $a$ and $b$ be two 4-digit integers with decimal representations $a = a_1 a_2 a_3 a_4$ and $b = b_1 b_2 b_3 b_4$. We partition the training data by the \textit{Number of Controlled Identical Digits} (NCID) into 5 equiprobable groups indexed by $k \in \{0, 1, 2, 3, 4\}$.

For a given group $k$, the data generation process is defined as follows:

\begin{enumerate}
    \item \textbf{Prefix Generation ($1 \le j \le k$):} The first $k$ digits are forced to be identical. 
    To ensure that $a$ and $b$ are valid 4-digit integers, the leading controlled digit is sampled from 
    $\{1,\dots,9\}$ whenever $k \ge 1$, while the remaining controlled digits are sampled from 
    $\{0,\dots,9\}$. That is, we sample
    \[
    p_1 \sim \mathcal{U}\{1,\dots,9\}, 
    \qquad 
    p_j \sim \mathcal{U}\{0,\dots,9\} \quad \text{for } j=2,\dots,k,
    \]
    and assign
    \[
    a_j = b_j = p_j \quad \text{for } j = 1,\dots,k.
    \]
    
    \item \textbf{Suffix Generation ($j > k$):} The remaining $4-k$ digits are sampled uniformly and independently for both numbers. 
    For every unconstrained digit position $j>k$, we sample
    \[
    a_j,b_j \sim \mathcal{U}\{0,\dots,9\},
    \]
    except when $j=1$, where the sampling range is $\{1,\dots,9\}$.
\end{enumerate}

This sampling strategy ensures that the dataset contains a uniform distribution of ``generative constraints," preventing the model from encountering only trivial comparisons (e.g., where the first digits differ) or only identical pairs.

Table~\ref{tab:ncid_examples} illustrates examples of pairs $(a, b)$ generated from each NCID group.

\begin{table}[h]
\centering
\caption{Examples of generated pairs $(a, b)$ for each NCID group $k$. The bolded digits indicate the guaranteed controlled prefix where $a_j = b_j$.}
\label{tab:ncid_examples}
\begin{tabular}{@{}c l l l l@{}}
\toprule
\textbf{NCID ($k$)} & \textbf{Constraint} & \textbf{Generation Logic} & \textbf{Example $a$} & \textbf{Example $b$} \\ \midrule
0 & No guaranteed match & 
\makecell[l]{
$a_1,b_1 \sim \mathcal{U}\{1,\dots,9\}$; \\
$a_{2:4},b_{2:4} \sim \mathcal{U}\{0,\dots,9\}^3$
}
& $7\,2\,9\,1$ & $4\,8\,1\,3$ \\
1 & $a_1 = b_1$ & $a_1=b_1$; others random & $\mathbf{3}\,5\,0\,2$ & $\mathbf{3}\,1\,9\,8$ \\
2 & $a_{1:2} = b_{1:2}$ & $a_{1:2}=b_{1:2}$; others random & $\mathbf{9\,2}\,4\,7$ & $\mathbf{9\,2}\,0\,5$ \\
3 & $a_{1:3} = b_{1:3}$ & $a_{1:3}=b_{1:3}$; last digit random & $\mathbf{6\,1\,8}\,3$ & $\mathbf{6\,1\,8}\,9$ \\
4 & $a = b$ & $a_{1:4}=b_{1:4}$ & $\mathbf{5\,4\,2\,1}$ & $\mathbf{5\,4\,2\,1}$ \\ \bottomrule
\end{tabular}
\end{table}

\subsection{Data generation for sorting task}\label{sec:append-data-gen-sorting}

To enhance the model's ability to sort both numbers of different lengths and close numbers like $6983,6981,6988,6987$, we adopt the "doubly balanced" sampling strategy. Specifically, in each example, each number has 0.5 probability to be 3-digit, 0.5 probability to be 4-digit. Once we have determined the length of all input numbers in that example, we randomly draw the NCID group from $k \in\{0,1,2\}$, each with probability $1/3$. The prefix generation and suffix generation are the same as the procedure described in Section~\ref{sec:append-data-gen-comparison}, except that in suffix generation if a number is 3-digit, we only need to sample the remaining $3-k$ digits, whereas for a 4-digit number, we need to sample the remaining $4-k$ digits.

% \begin{enumerate}
%     \item \textbf{Level 1: No closeness enforcement} Each 3-digit number is uniformly drawn from 100 to 999. Each 4-digit number is uniformly drawn from 1000 to 9999.
%     \item \textbf{Level 2: Same 1st digit} The single common digit is uniformly drawn from 1 to 9. The rest 2 digits for each 3-digit number is uniformly drawn from 00 to 99. The rest 3 digits for each 4-digit number is uniformly drawn from 000 to 999. (e.g. 581,5099,581,5285)
%     \item \textbf{Level 3: Same 1st \& 2nd digit} The two common digits are uniformly drawn from 10 to 99. The only rest digit for each 3-digit number is uniformly drawn from 0 to 9, while the two rest digits for each 4-digit number is uniformly drawn from 00 to 99. (e.g. 255,250,2563,258)
% \end{enumerate}

We list some training examples from our training data in Table~\ref{tab:sorting_doubly_bal_examples}.
\begin{table}[h]
\centering
\caption{Examples of doubly balanced datasets for the sorting task. $l(a,b,c,d)$ denotes the length of the four numbers.}
\label{tab:sorting_doubly_bal_examples}
\begin{tabular}{@{}c c p{4.6cm} l@{}}
\toprule
\textbf{Length} & \textbf{NCID ($k$)} & \textbf{Constraint} & \textbf{Example input} \\
\midrule
$(4,4,4,4)$ & 0 &
$l(a,b,c,d)=(4,4,4,4)$ &
$3888,2374,8914,1858$ \\

$(3,4,3,4)$ & 1 &
\begin{tabular}[c]{@{}l@{}}
$l(a,b,c,d)=(3,4,3,4)$ \& \\
$a_1=b_1=c_1=d_1$
\end{tabular} &
$581,5099,581,5285$ \\

$(3,3,4,3)$ & 2 &
\begin{tabular}[c]{@{}l@{}}
$l(a,b,c,d)=(3,3,4,3)$ \& \\
$a_{1:2}=b_{1:2}=c_{1:2}=d_{1:2}$
\end{tabular} &
$255,250,2563,258$ \\
\bottomrule
\end{tabular}
\end{table}

\subsection{Output generation}\label{sec:append-metrics}

\paragraph{Sampling for output generation.} 
For all the different experiments we used $0.8$ as a default temperature for sampling. Apart from sampling with temperature $0.8$, we also experimented with greedy decoding to validate our findings. Figure~\ref{fig:addition_greedy} reports the results for the addition task with greedy decoding, which are similar to our original results Figure~\ref{fig:1}.

\subsection{Existing assets and licenses}
\label{app:assets}

Our transformer implementation builds on NanoGPT
\citep{Karpathy2022}, distributed under the MIT License.
Our fine-tuning experiment uses Pythia-1B
\citep{biderman2023pythia}, distributed under Apache License 2.0.
Our clause-extended evaluation adapts questions from GSM8K
\citep{cobbe2021training}, whose source repository uses the MIT
License. Table~\ref{tab:asset-licenses} identifies the reused assets,
their source repositories, and their licenses or terms of use.

% Include this file from the one-column appendix:
%   \input{Tables/asset_licenses}
%
% Required packages (already present in the supplied manuscript):
%   \usepackage{booktabs}
%   \usepackage{url}
%   \usepackage{hyperref}
%
% Before final submission, confirm that these repository identifiers match
% the assets actually used, including the Pythia variant. Record the exact
% model/tokenizer revisions and code commits in the experimental settings.
% Listing licenses does not, by itself, establish compliance with their terms.

\begin{table}[htbp]
  \centering
  \small
  \renewcommand{\arraystretch}{1.20}
  \setlength{\tabcolsep}{5pt}
  \caption{Licenses and terms for existing assets.
  Asset names link to official repositories or model cards.}
  \label{tab:asset-licenses}
  \begin{tabular}{@{}p{0.47\linewidth}p{0.49\linewidth}@{}}
    \toprule
    \textbf{Asset / official source} &
    \textbf{License or terms} \\
    \midrule

    \raggedright
    \href{https://github.com/karpathy/nanoGPT/blob/master/LICENSE}{NanoGPT}
    &
    \raggedright MIT
    \tabularnewline

    \raggedright
    \href{https://huggingface.co/EleutherAI/pythia-1b}{\nolinkurl{EleutherAI/pythia-1b}}
    &
    \raggedright Apache License 2.0
    \tabularnewline

    \raggedright
    \href{https://github.com/openai/grade-school-math/blob/master/LICENSE}{GSM8K}
    &
    \raggedright MIT
    \tabularnewline

    \raggedright
    \href{https://huggingface.co/mistralai/Mistral-7B-Instruct-v0.3}{\nolinkurl{mistralai/Mistral-7B-Instruct-v0.3}}
    &
    \raggedright Apache License 2.0
    \tabularnewline

    \raggedright
    \href{https://huggingface.co/meta-llama/Meta-Llama-3-8B-Instruct/blob/main/LICENSE}{\nolinkurl{meta-llama/Meta-Llama-3-8B-Instruct}}
    &
    \raggedright Meta Llama 3 Community License Agreement and Acceptable Use Policy
    \tabularnewline

    \raggedright
    \href{https://huggingface.co/google/gemma-2-9b-it}{\nolinkurl{google/gemma-2-9b-it}}
    &
    \raggedright \href{https://ai.google.dev/gemma/terms}{Gemma Terms of Use},
    including the Prohibited Use Policy
    \tabularnewline

    \raggedright
    \href{https://huggingface.co/Qwen/Qwen2.5-1.5B-Instruct}{\nolinkurl{Qwen/Qwen2.5-1.5B-Instruct}}
    &
    \raggedright Apache License 2.0
    \tabularnewline

    \raggedright
    \href{https://huggingface.co/Qwen/Qwen2.5-3B-Instruct/blob/main/LICENSE}{\nolinkurl{Qwen/Qwen2.5-3B-Instruct}}
    &
    \raggedright Qwen Research License Agreement
    \tabularnewline

    \raggedright
    \href{https://huggingface.co/Qwen/Qwen2.5-7B-Instruct}{\nolinkurl{Qwen/Qwen2.5-7B-Instruct}}
    &
    \raggedright Apache License 2.0
    \tabularnewline

    \raggedright
    \href{https://huggingface.co/meta-llama/Llama-3.1-8B-Instruct}{\nolinkurl{meta-llama/Llama-3.1-8B-Instruct}}
    &
    \raggedright Llama 3.1 Community License Agreement and Acceptable Use Policy
    \tabularnewline

    \raggedright
    \href{https://huggingface.co/Qwen/Qwen2.5-Math-7B-Instruct}{\nolinkurl{Qwen/Qwen2.5-Math-7B-Instruct}}
    &
    \raggedright Apache License 2.0
    \tabularnewline

    \raggedright
    \href{https://huggingface.co/mistralai/Mathstral-7B-v0.1}{\nolinkurl{mistralai/Mathstral-7B-v0.1}}
    &
    \raggedright Apache License 2.0
    \tabularnewline

    \raggedright
    \href{https://huggingface.co/deepseek-ai/deepseek-math-7b-instruct}{\nolinkurl{deepseek-ai/deepseek-math-7b-instruct}}
    &
    \raggedright
    \href{https://github.com/deepseek-ai/DeepSeek-Math/blob/main/LICENSE-MODEL}{DeepSeek License Agreement v1.0}
    (model weights)
    \tabularnewline

    \bottomrule
  \end{tabular}
\end{table}

%\section{Details of experiment setup}

%\subsection{Data generation}

%\paragraph{Sorting.} We used uniform sampling in our initial experiment where we uniformly and independently drew integers from $\{100,101,\ldots,9999\}$ to generate $a,b,c,d$ in the sorting task. Results are pretty bad blablabla; see Figure~\ref{}, so we used doubly balanced sampling instead....

\section{Additional details about evaluation metrics}

\subsection{Digit-wise Error}

%\YZ{Why do we have decimal digits? I don't understand this part.}

%\PX{Previously, the description is a bit confusing. I've revised this part. Here, \(W\) (for the convenience of writing the math formula for \(E_i(t)\)) is meant to take the longest length among all actuals. In our tests, the actual addition results of some test examples may be 3-digit while others are 4-digit. We take \(W\) to be 4.}

Let a collection of $N$ test examples be indexed by $k = 1, \dots, N$.
For each example $k$, let $a^{(k)}$ denote the gold (actual) value
and $p^{(k)}_{t}$ the model’s prediction at training step $t$.

Define the digit width $W$ to be the maximal digit length of these actuals:
\[
W = \max_{k} \left| a^{(k)} \right|,
\]
where \(\lvert\cdot\rvert\) denotes string length (number of digits). In particular, \(W\) equals 4 in our addition task.

For each \(k\), let \(\tilde a^{(k)}\) be the width-\(W\) representation of the actual obtained by left zero-padding the actual.  Before applying the same truncation-and-padding rule to model outputs, we first express each prediction in the standard integer order.  Thus, for plain-format outputs we set \(q_t^{(k)} = p_t^{(k)}\), while for reverse-format outputs we set \(q_t^{(k)} = \operatorname{rev}(p_t^{(k)})\), after removing any end-of-sequence marker.  We then truncate this standard-order string \(q_t^{(k)}\) to its rightmost \(W\) characters and left zero-pad it to width \(W\):
\[
\tilde a^{(k)} \;=\; \mathrm{zfill}_W\bigl(a^{(k)}\bigr), \qquad
\tilde p_t^{(k)} \;=\; \mathrm{zfill}_W\bigl(\operatorname{rightmost}_W(q_t^{(k)})\bigr).
\]
In particular, for reverse-format predictions the raw generated string is \emph{not} truncated from the right directly; it is first reversed back to standard order, so truncation removes over-generated high-order digits rather than the units digit.

Let \(d_i(x)\) denote the \(i\)-th digit from the right (e.g., \(i=0\) for units, \(i=1\) for tens, etc.) in the width-\(W\) string \(\tilde x\).  The digit-wise error rate at place \(i\) for training step \(t\) is 
\[
E_i(t) \;=\; \frac{1}{N}\sum_{k=1}^N \mathbf{1}\!\bigl( d_i(\tilde a^{(k)}) \neq d_i(\tilde p_t^{(k)}) \bigr), 
\qquad i=0,\ldots,W-1,
\]
where \(\mathbf{1}[\cdot]\) is the indicator function.  Thus \(E_i(t)\) measures the fraction of examples whose predicted \(i\)-th-from-right digit, after standard-order conversion, truncation, and padding, differs from the gold \(i\)-th-from-right digit.

\subsection{Mutual Information}
\label{appendix-MI-metric}
While the digit-wise error tells us where the model makes mistakes, it does not capture the statistical dependencies the model may exploit between digit-positions in training examples. To probe this, we compute dataset-based mutual information (MI) as benchmarks and track the MI from model's prediction distribution: %\YZ{Re-work the MI metrics.} \PX{Updated.}

\paragraph{Digits and carries.}
Each example is a 4-operand addition
\[
a_1a_2a_3 + b_1b_2b_3 + c_1c_2c_3 + d_1d_2d_3 \;=\; e_0e_1e_2e_3,
\]
where $a_1,a_2,a_3$ are the (hundreds, tens, units) digits of the first addend (and similarly for $b,c,d$), and
$e_0,e_1,e_2,e_3$ are the (thousands, hundreds, tens, units) digits of the sum.
For each column $i\in\{1,2,3\}$ we define the \emph{local carry-out} by
\[
K_{i-1} \;=\;\Big\lfloor \tfrac{a_i+b_i+c_i+d_i}{10}\Big\rfloor,
\]
i.e.\ the carry generated by the digits in column $i$.

\paragraph{Dataset-based MI.}
Given a dataset $\mathcal D=\{(x^{(n)},y^{(n)},z^{(n)})\}_{n=1}^N$ extracted from the generated 1M addition examples,
we use the empirical estimators
\begin{align*}
\hat p(x,y)
&= \frac1N\sum_{n=1}^N \mathbf 1[x^{(n)}=x,\,y^{(n)}=y],\\
\hat p(x,y,z)
&= \frac1N\sum_{n=1}^N \mathbf 1[x^{(n)}=x,\,y^{(n)}=y,\,z^{(n)}=z].
\end{align*}
Mutual information and conditional mutual information are then
\begin{align*}
\widehat I(X;Y)
&= \sum_{x,y}\hat p(x,y)\log\frac{\hat p(x,y)}{\hat p(x)\hat p(y)},\\
\widehat I(X;Y\mid Z)
&= \sum_{z}\hat p(z)\sum_{x,y}\hat p(x,y\mid z)\log\frac{\hat p(x,y\mid z)}{\hat p(x\mid z)\hat p(y\mid z)}.
\end{align*}
% \[
% \hat p(x,y) \;=\;\frac1N\sum_{n=1}^N \mathbf 1[x^{(n)}=x,\,y^{(n)}=y],
% \qquad
% \hat p(x,y,z) \;=\;\frac1N\sum_{n=1}^N \mathbf 1[x^{(n)}=x,\,y^{(n)}=y,\,z^{(n)}=z].
% \]
% Mutual information and conditional mutual information are then
% \[
% \widehat I(X;Y)
% =\sum_{x,y}\hat p(x,y)\log\frac{\hat p(x,y)}{\hat p(x)\hat p(y)},
% \qquad
% \widehat I(X;Y\mid Z)
% =\sum_{z}\hat p(z)\sum_{x,y}\hat p(x,y\mid z)\log\frac{\hat p(x,y\mid z)}{\hat p(x\mid z)\hat p(y\mid z)}.
% \]

The reported \emph{dataset} probes are:
\[
\textbf{Thousands:}\quad \widehat I(a_1; e_0),\qquad
\]
\[
\textbf{Carries: }
\widehat I(a_i; e_i \mid K_{i-1}) \quad \text{for } i\in\{1,2,3\}
\]
\[
\textbf{Immediate higher digit :}\quad \widehat I(a_i; e_i \mid e_{i-1}) \quad \text{for } i\in\{1,2,3\}.
\]

The local carry statistic $K_{i-1}$ constrains the current column's
operand sum and can expose dependence between an operand digit and
the output digit. For the units and tens columns ($i=3,2$), the next
higher output digit combines the actual carry with an independent
column sum modulo ten, masking that carry signal under independent uniform
digit sampling. The highest output digit is different: $e_0$ equals
the actual carry leaving the hundreds column.

\paragraph{MI from model prediction distributions.}
For each example $n$ and each output position (digit) $i$, let
\[
\hat p^{(n)}_i(y) \;=\; p_\theta(E_i=y\mid \text{input}^{(n)})
\]
be the model's softmax distribution over the output vocabulary at that position.
To estimate MI between an input digit $X$ and the model's predicted output distribution at position $i$,
we first aggregate predictions by the value of $X$:
\[
\bar p_i(y\mid x)
\;=\;\frac{1}{|\mathcal D_x|}\sum_{n:\,x^{(n)}=x}\hat p^{(n)}_i(y),
\qquad
\hat p(x)=\frac{|\mathcal D_x|}{N},
\qquad
\bar p_i(y)=\sum_x \hat p(x)\,\bar p_i(y\mid x).
\]
We use the expected KL divergence as the estimator
\[
\widehat I_\theta(X;\widehat P_i)
\;=\;\sum_x \hat p(x)\,\mathrm{KL}\!\big(\bar p_i(\cdot\mid x)\,\|\,\bar p_i(\cdot)\big),
\quad
\mathrm{KL}(p\|q)=\sum_y p(y)\log\frac{p(y)}{q(y)}.
\]
For conditioning on a discrete variable $Z$ (either $Z=e_{i-1}$ or $Z=K_{i-1}$), we aggregate by $(x,z)$:
\begin{align*}
\bar p_i(y\mid x,z)
&=\frac{1}{|\mathcal D_{x,z}|}\sum_{n:\,x^{(n)}=x,\,z^{(n)}=z}\hat p^{(n)}_i(y), \\
\hat p(x\mid z)
&=\frac{|\mathcal D_{x,z}|}{|\mathcal D_z|}, \\
\bar p_i(y\mid z)
&=\sum_x \hat p(x\mid z)\,\bar p_i(y\mid x,z).
\end{align*}
and compute
\begin{align*}
\widehat I_\theta(X;\widehat P_i\mid Z)
&=\sum_z \hat p(z)\sum_x \hat p(x\mid z)\,
\mathrm{KL}\!\big(\bar p_i(\cdot\mid x,z)\,\|\,\bar p_i(\cdot\mid z)\big).
\end{align*}
The reported \emph{model} probes are:
\begin{align*}
\textbf{Thousands:}\quad
&\widehat I_\theta(a_1;\widehat P_0), \\
\textbf{Carries:}\quad
&\widehat I_\theta(a_i;\widehat P_i\mid K_{i-1}), \\
\textbf{Immediate higher digit:}\quad
&\widehat I_\theta(a_i;\widehat P_i\mid e_{i-1}).
\end{align*}
% For conditioning on a discrete variable $Z$ (either $Z=e_{i-1}$ or $Z=K_{i-1}$), we aggregate by $(x,z)$:
% \[
% \bar p_i(y\mid x,z)
% =\frac{1}{|\mathcal D_{x,z}|}\sum_{n:\,x^{(n)}=x,\,z^{(n)}=z}\hat p^{(n)}_i(y),
% \quad
% \hat p(x\mid z)=\frac{|\mathcal D_{x,z}|}{|\mathcal D_z|},
% \quad
% \bar p_i(y\mid z)=\sum_x \hat p(x\mid z)\,\bar p_i(y\mid x,z),
% \]
% and compute 
% \[
% \widehat I_\theta(X;\widehat P_i\mid Z)
% =\sum_z \hat p(z)\sum_x \hat p(x\mid z)\,
% \mathrm{KL}\!\big(\bar p_i(\cdot\mid x,z)\,\|\,\bar p_i(\cdot\mid z)\big).
% \]
% The reported \emph{model} probes are:
% \[
% \textbf{Thousands:}\quad \widehat I_\theta(a_1;\widehat P_0),\qquad
% \textbf{Carries:}\quad \widehat I_\theta(a_i;\widehat P_i\mid K_{i-1}),\qquad
% \textbf{Immediate higher digit:}\quad \widehat I_\theta(a_i;\widehat P_i\mid e_{i-1}).
% \]

\section{Details and additional results for synthetic experiments}\label{sec:append-synthetic-details}
\subsection{Addition task.}
\label{appendix-addition}
\subsubsection{Alternative format of input sequence}\label{sec:append-alt-1}
\begin{figure}[t] 
    \centering 
    \begin{subfigure}{0.5\textwidth} 
        \centering 
        % Replace with your actual image file
        \includegraphics[width=\linewidth]{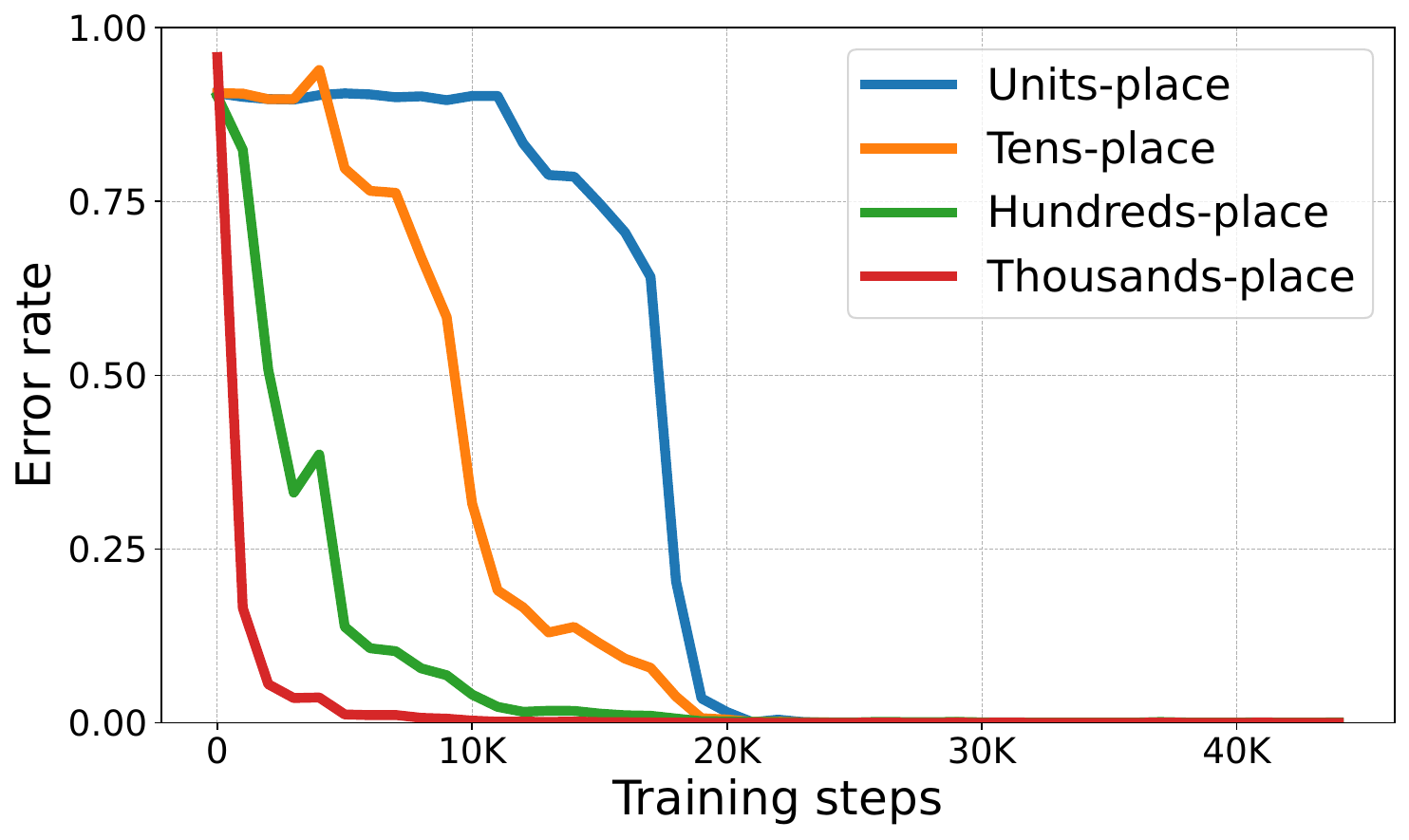} 
        \caption{Digit-wise error in plain output format} 
    \end{subfigure}% 
    \hfill 
    \begin{subfigure}{0.5\textwidth} 
        \centering 
        % Replace with your actual image file
        \includegraphics[width=\linewidth]{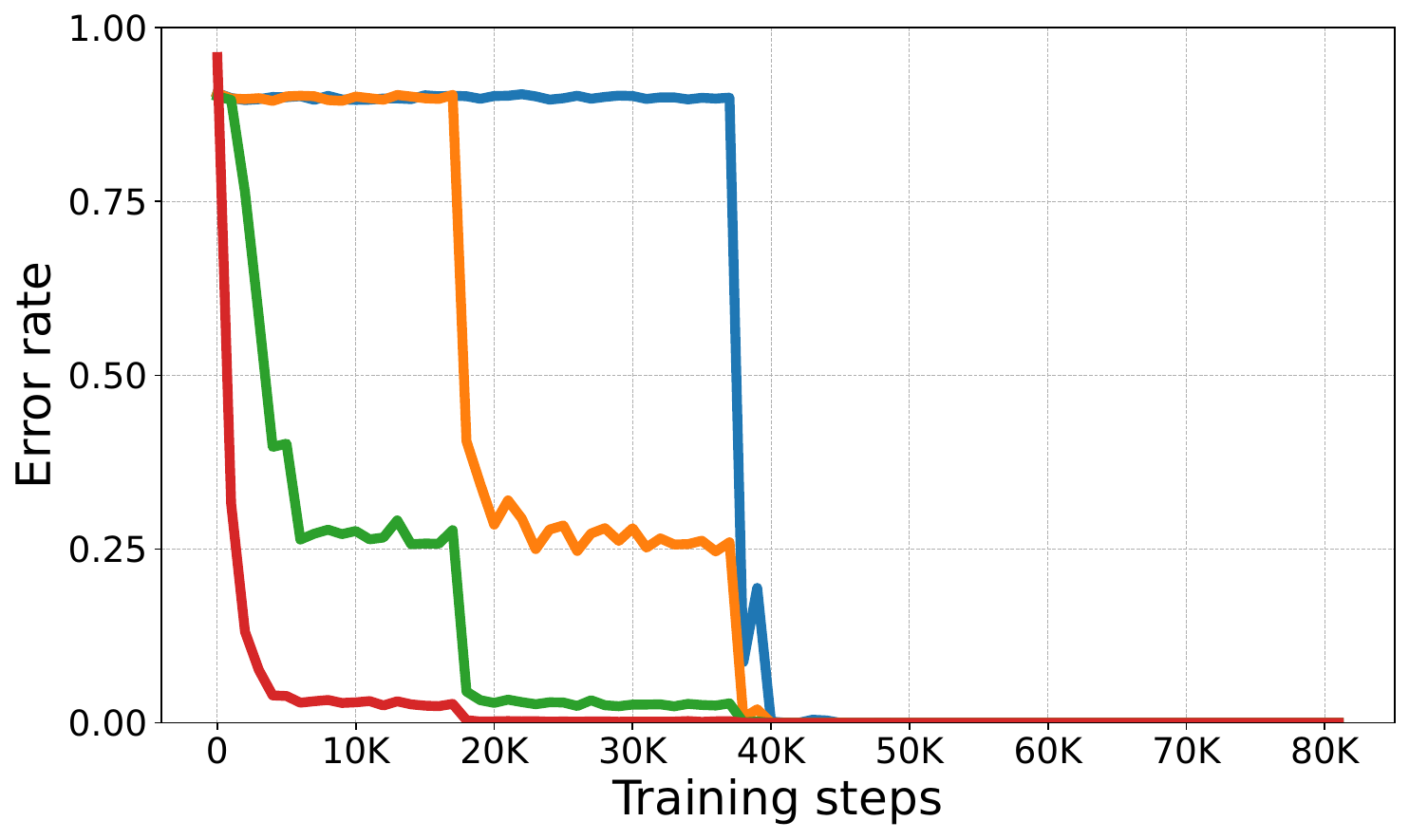} 
        \caption{Digit-wise error in reverse output format} 
    \end{subfigure} 
    
    \caption{\textbf{Using randomly chunked input sequence, transformers still learn digits in non-human order for addition.} We train transformers with randomly chunked blocks as input sequences and evaluate the digit-wise error rates. Same as single-example input sequence format, models learn digits from the most significant digit to the least significant digit.} 
    
    % The label MUST be outside the comments and after the caption
    \label{fig:addition_slicing} 
\end{figure}
As an alternative to the single-example input sequence format adopted in the main paper, we consider the randomly chunked input sequence format--concatenating all training examples and randomly chunking a fixed-size window as an input sequence. As shown in Figure~\ref{fig:addition_slicing}, models still learn from the most significant digit to the least significant one.

\subsubsection{Output digit permutations}
\label{addition_perm}
\begin{figure}[t] 
    \centering 
    \begin{subfigure}{0.5\textwidth} 
        \centering 
        % Replace with your actual image file
        \includegraphics[width=\linewidth]{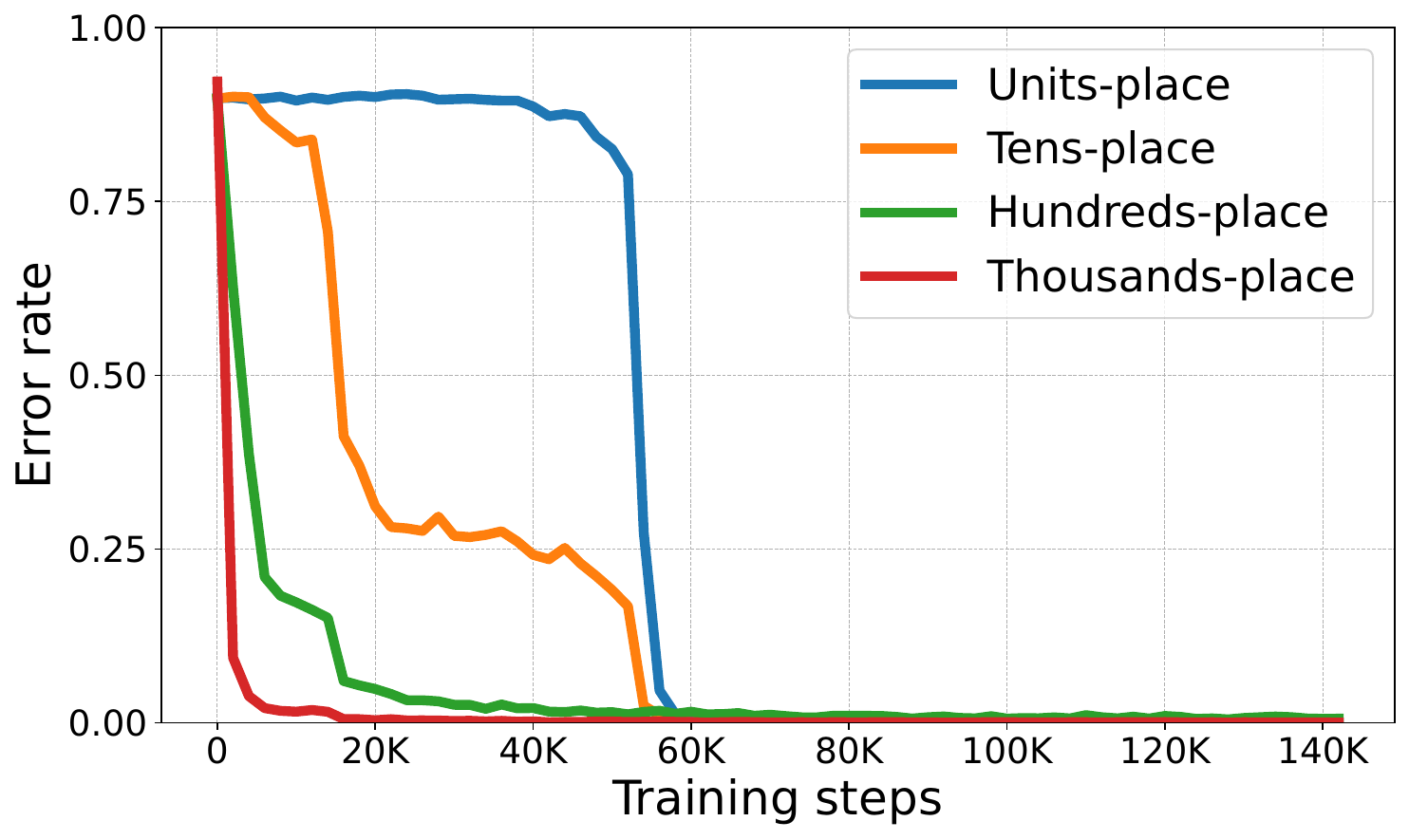} 
        \caption{Output permutation ``2143''}
    \end{subfigure}% 
    \hfill 
    \begin{subfigure}{0.5\textwidth} 
        \centering 
        % Replace with your actual image file
        \includegraphics[width=\linewidth]{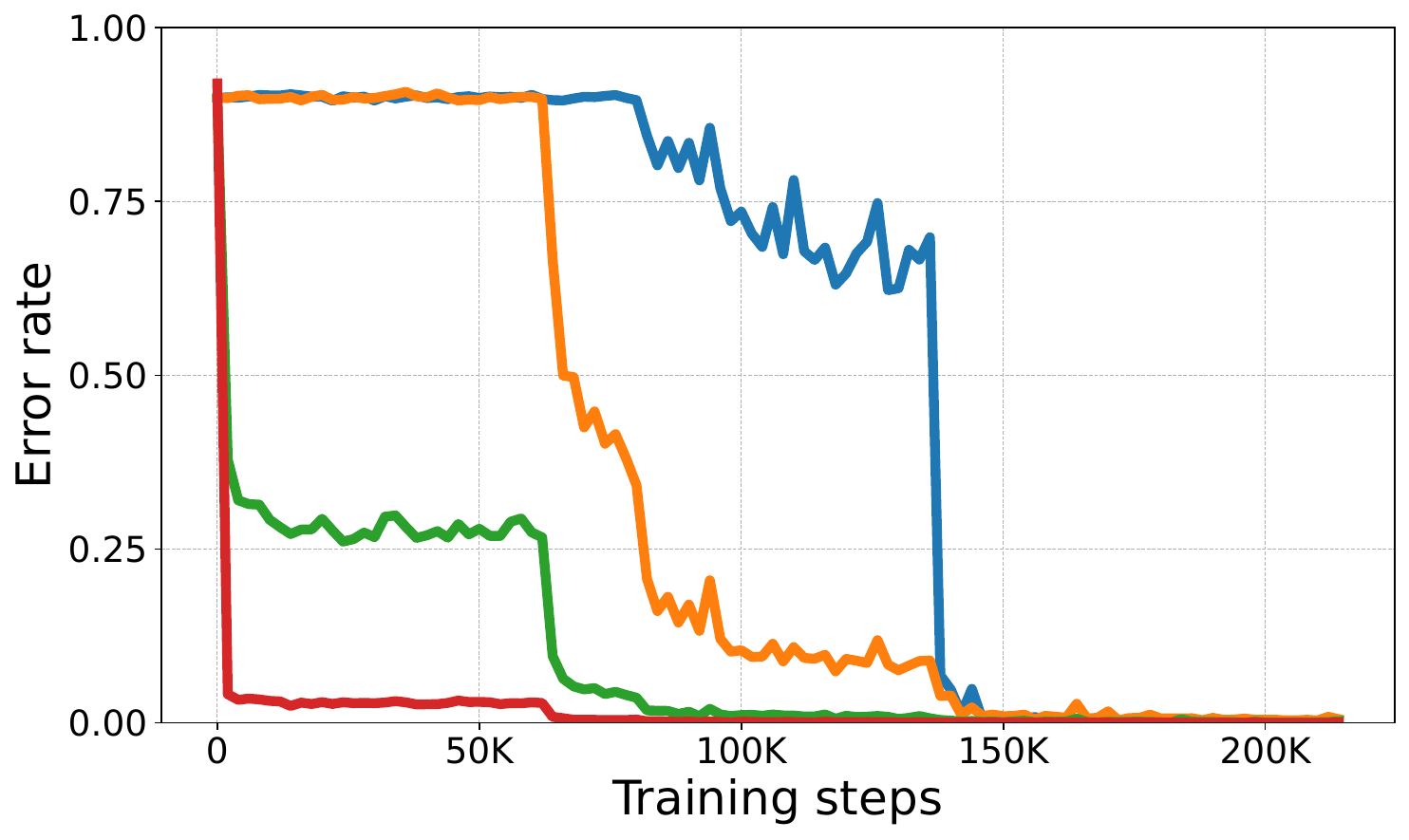} 
        \caption{Output permutation ``3412''}
    \end{subfigure} 
    
    \caption{\textbf{Regardless of permutations of the output order, transformers learn digits in non-human order for addition.} We train transformers on addition of the format $a+b+c+d=o$ with the digits of $o$ written in different permutation orders. In permutation "2143", $o$ is written as $o_2o_1o_4o_3$. In permutation "3412", $o$ is written as $o_3o_4o_1o_2$. Same as plain and reverse output format, models learn digits from the most significant digit ($o_1$) to the least significant digit ($o_4$).} 
    
    \label{fig:addition_permutation} 
\end{figure}
In the main paper, we train transformer models on 4 operand addition using both plain and reverse order as the output format. Here we consider two other permutations of output digits order. Concretely, let $o_1o_2o_3o_4$ be the groundtruth label for a 4 operand addition instance $a_1a_2a_3+b_1b_2b_3+c_1c_2c_3+d_1d_2d_3$, with digits ordered from the most significant to the least significant. We consider permutation "2143" and "3412".

For permutation "2143", a training example takes the form:

\begin{equation}
    a_1a_2a_3 + b_1b_2b_3 + c_1c_2c_3 + d_1d_2d_3 = o_2o_1o_4o_3
\end{equation}

For permutation ``3412", a training example takes the form:

\begin{equation}
    a_1a_2a_3 + b_1b_2b_3 + c_1c_2c_3 + d_1d_2d_3 = o_3o_4o_1o_2
\end{equation}

As shown in Figure~\ref{fig:addition_permutation}, regardless of the permutation of the output orders, models learn digits from the most significant digit ($o_1$) to the least significant digit ($o_4$).

\subsubsection{Reversing input digits}
\label{app:addition-input-reversal}

\begin{figure}[t]
    \centering
    \begin{subfigure}[t]{0.49\linewidth}
        \centering
        \includegraphics[width=\linewidth]{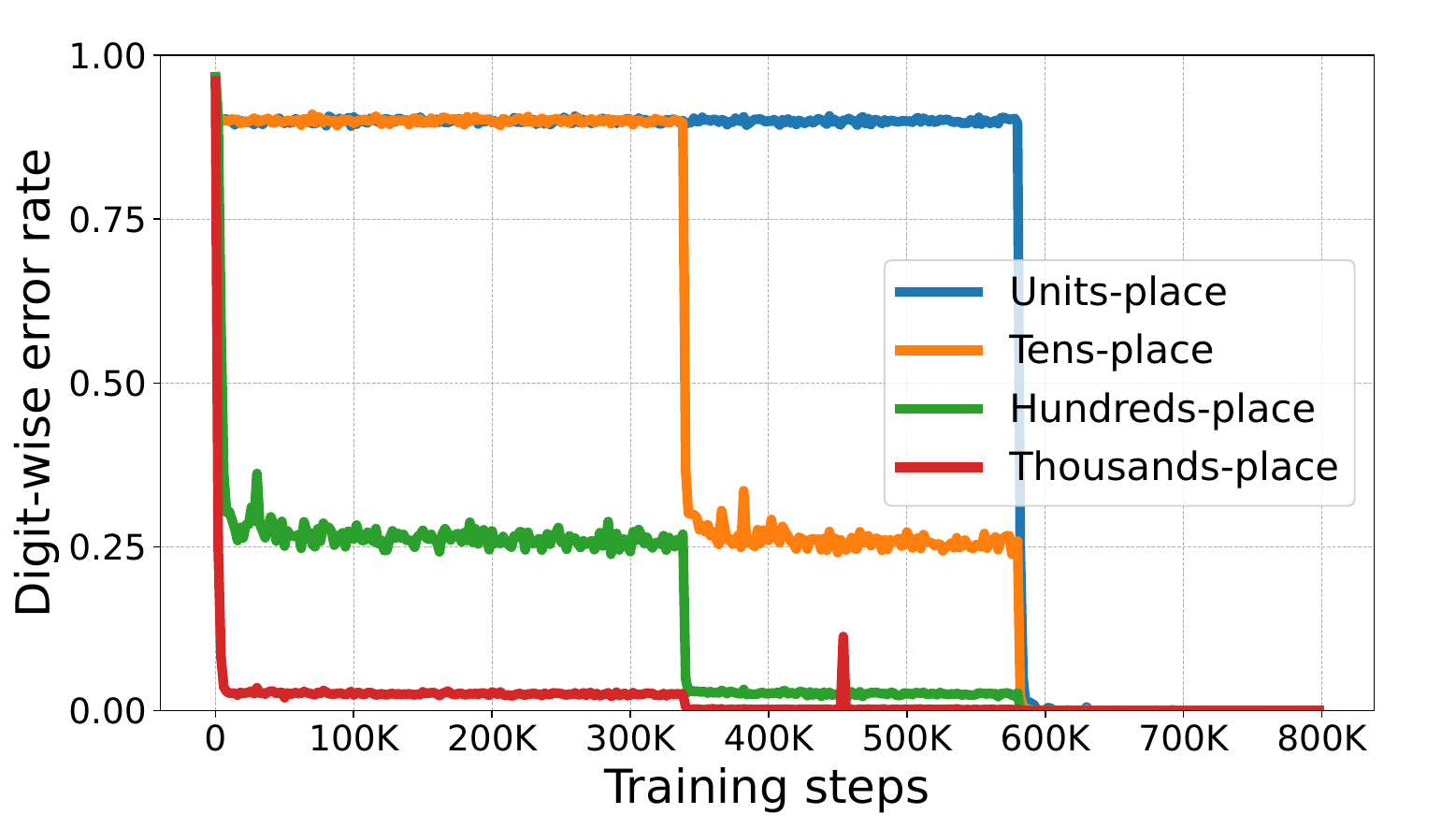}
        \caption{Reversed inputs, reversed output}
        \label{fig:addition-reverse-both}
    \end{subfigure}\hfill
    \begin{subfigure}[t]{0.49\linewidth}
        \centering
        \includegraphics[width=\linewidth]{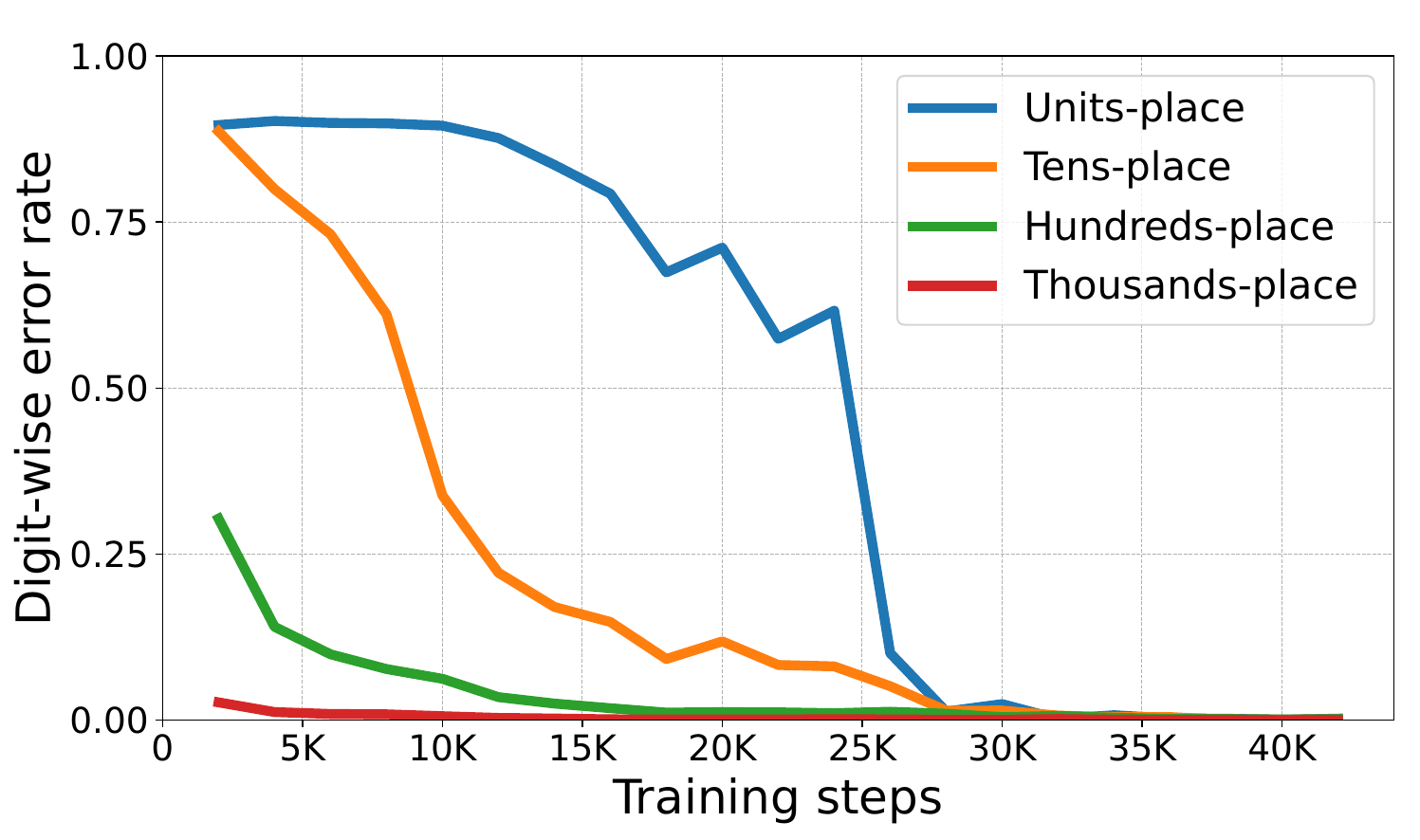}
        \caption{Reversed inputs, plain output}
        \label{fig:addition-reverse-inputs}
    \end{subfigure}
    \caption{\textbf{Addition results with reversed input digits.}
    We train transformers on 4-operand addition with digits reversed within
    each input operand. Plot (a) uses reversed output digits, from least
    to most significant. Plot (b) uses plain output digits, from most
    to least significant. In both settings, models learn digits from
    the most significant to the least significant.}
    \label{fig:addition-input-reversal}
\end{figure}

We repeat four-operand addition with operands sampled uniformly
from $\{0,\ldots,999\}$, reversing the digit order within each
input operand while preserving operand order and the underlying
addition problem. We pair these inputs with either reversed or
plain outputs. As shown in Figure~\ref{fig:addition-input-reversal},
both settings preserve the most-significant-to-least-significant
acquisition pattern. Thus, this pattern is not specific to the
original within-operand input order.

\subsubsection{Encoder-only next-token prediction}
\label{app:addition-entp}

We repeat four-operand addition with reversed outputs using encoder-only
next-token prediction (ENTP; \citep{ewer2025encoder}). Each operand is sampled
independently and uniformly from $\{0,\ldots,999\}$, with input digits
in their ordinary order.
Unlike the causal decoder, ENTP applies bidirectional self-attention
at every layer over the available prefix. To prevent target leakage,
each training instance contains a truncated prefix and supervises
only the next answer digit or EOS. During generation, the model
recomputes the growing prefix after each predicted token;
generation therefore remains autoregressive.
We train for 800K updates and evaluate using greedy decoding.

Figure~\ref{fig:addition-entp} shows that the thousands- and hundreds-place
errors decrease while the units- and tens-place errors initially remain
near 90\%. The tens-place error then falls before the units-place error,
but later the units-place error becomes lower than the tens-place error.
Thus, ENTP retains the initial higher-place-first pattern but changes the
later relative learning trajectory, rather than preserving a strict
highest-to-lowest ordering throughout training.
Whether this altered acquisition trajectory is related to ENTP's
reported sample-efficiency and length-generalization benefits
\citep{ewer2025encoder} remains an open question, which our experiment does not directly evaluate.

\begin{figure}[t]
    \centering
    \includegraphics[width=0.65\linewidth]
{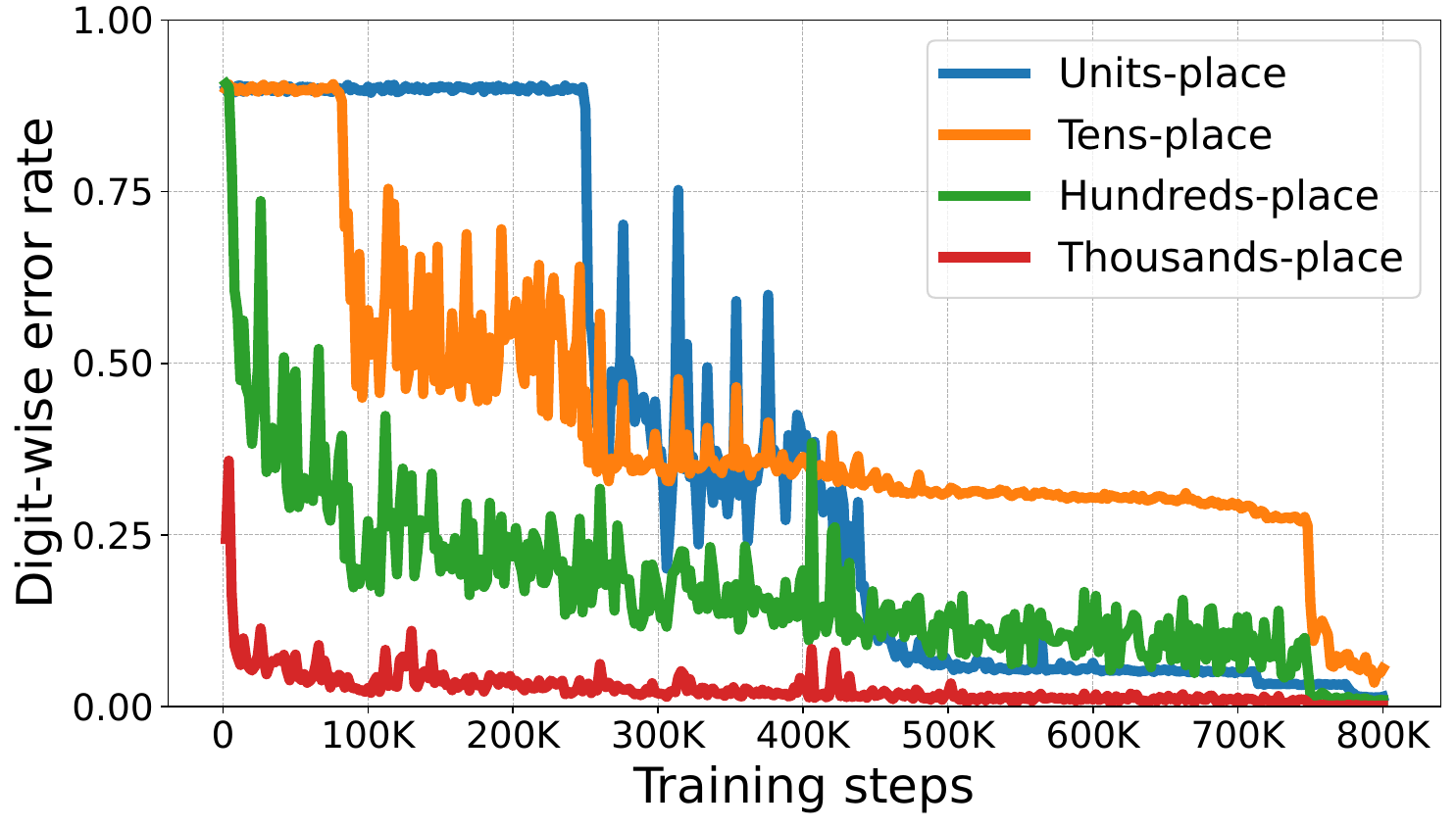}
    \caption{\textbf{Encoder-only next-token prediction for reverse-output addition.}
    Digit-wise test errors for four-operand addition, with operands sampled
    independently and uniformly from $\{0,\ldots,999\}$, over 800K training
    steps using ENTP and greedy decoding. Higher-place digits improve first,
    but the units-place error later falls below the tens-place error;
    the trajectory is therefore not strictly highest-to-lowest throughout training.}
    \label{fig:addition-entp}
\end{figure}

\subsubsection{Positional-encoding variants: RoPE and T5 relative}
\label{app:positional_variants}
\begin{figure}[t] 
    \centering 
    \begin{subfigure}{0.5\textwidth} 
        \centering 
        % Replace with your actual image file
        \includegraphics[width=\linewidth]{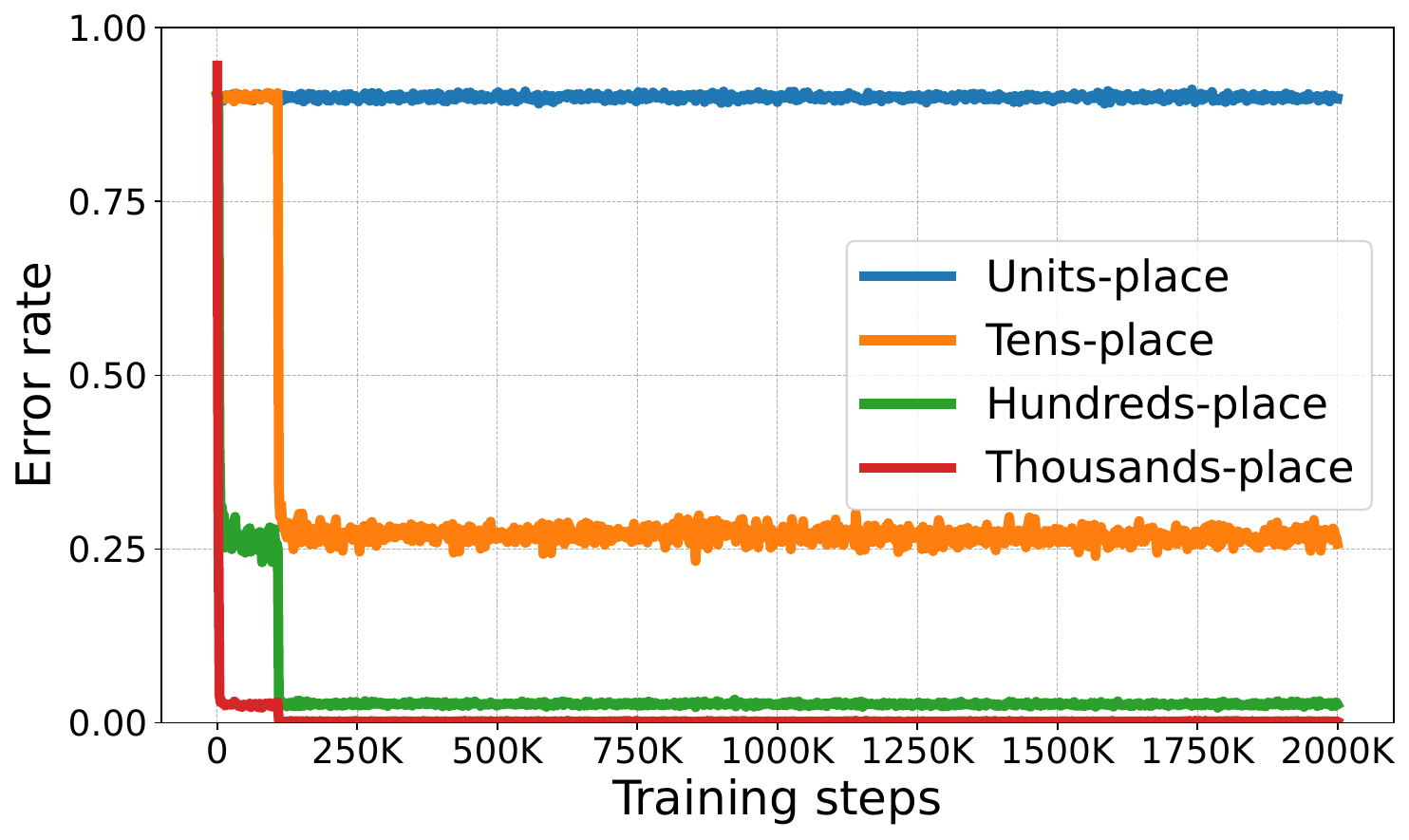} 
        \caption{Digit-wise error with RoPE} 
    \end{subfigure}% 
    \hfill 
    \begin{subfigure}{0.5\textwidth} 
        \centering 
        % Replace with your actual image file
        \includegraphics[width=\linewidth]{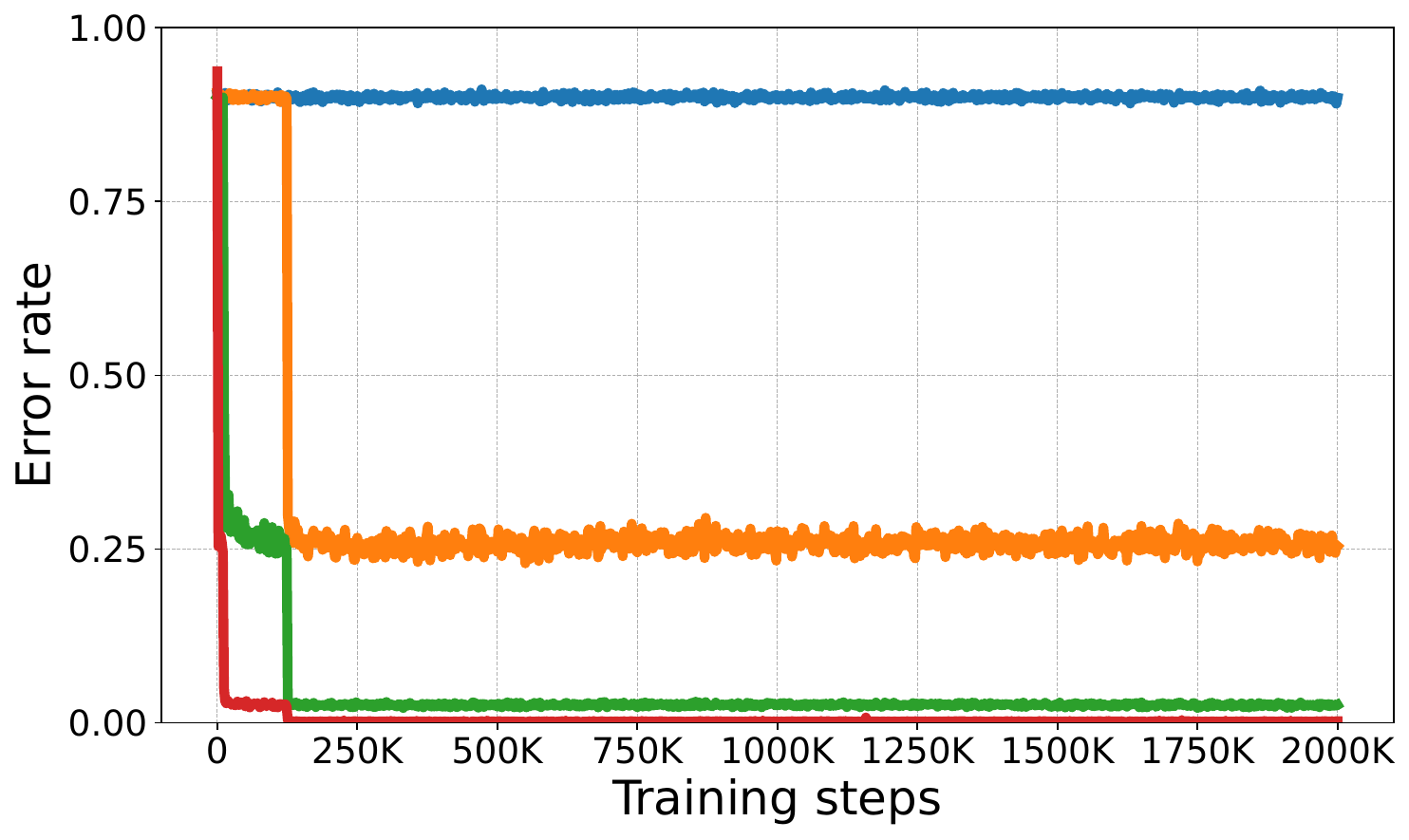} 
        \caption{Digit-wise error with T5-style relative biases} 
    \end{subfigure} 
    
    \caption{\textbf{Additional result for the addition task under two different positional encoding schemes.} We train transformers on 4-operand addition in reverse output format, using RoPE / T5-style relative biases. The units-place error remains high throughout training while tens/hundreds/thousands are learned in reverse order.} 
    \label{fig:addition_postional_encoding} 
\end{figure}

In the main experiments we used absolute positional encodings. To test robustness to positional-encoding choice, we re-ran the 4-operand addition in reverse output format experiment with two alternative encodings: rotary positional embeddings and T5-style relative position biases. Implementation changes were limited to replacing the absolute embedding layer with the alternate encoding described below.

\begin{itemize}
  \item \textbf{RoPE (rotary positional embeddings)}: We applied RoPE to the query/key projections following standard practice: a rotary transformation is applied to query and key vectors before computing attention scores.
  \item \textbf{T5-style relative biases}: We implemented learned relative-position bias terms added to the attention logits. The number of relative buckets and max distance were set to 32 and 128 respectively, as are used in the canonical T5 implementation.
\end{itemize}

Figure~\ref{fig:addition_postional_encoding} shows the digit-wise error rates (units, tens, hundreds, thousands) as training progresses. With both RoPE and T5 relative-position biases the model learns the higher-place digits (tens, hundreds, thousands) in reverse order (i.e., it learns thousands first), but it fails to learn the units-place even after more than 2M training steps.

\subsubsection{Greedy decoding}
\begin{figure}[t] 
    \centering 
    \begin{subfigure}{0.49\textwidth} 
        \centering 
        % Replace with your actual image file
        \includegraphics[width=\linewidth]{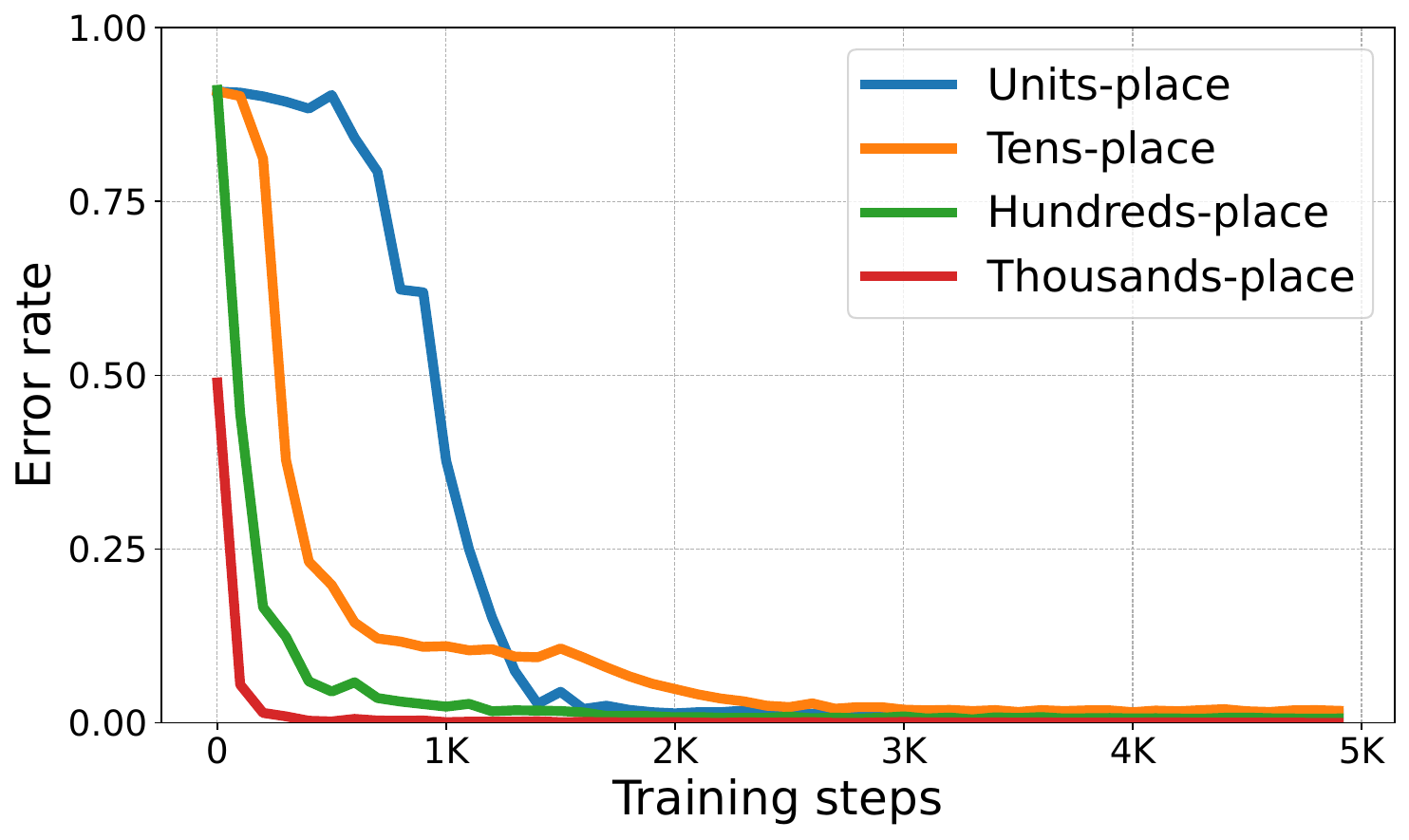} 
        \caption{2-operands addition plain output format with greedy decoding} 
    \end{subfigure}% 
    \hfill 
    \begin{subfigure}{0.49\textwidth} 
        \centering 
        % Replace with your actual image file
        \includegraphics[width=\linewidth]{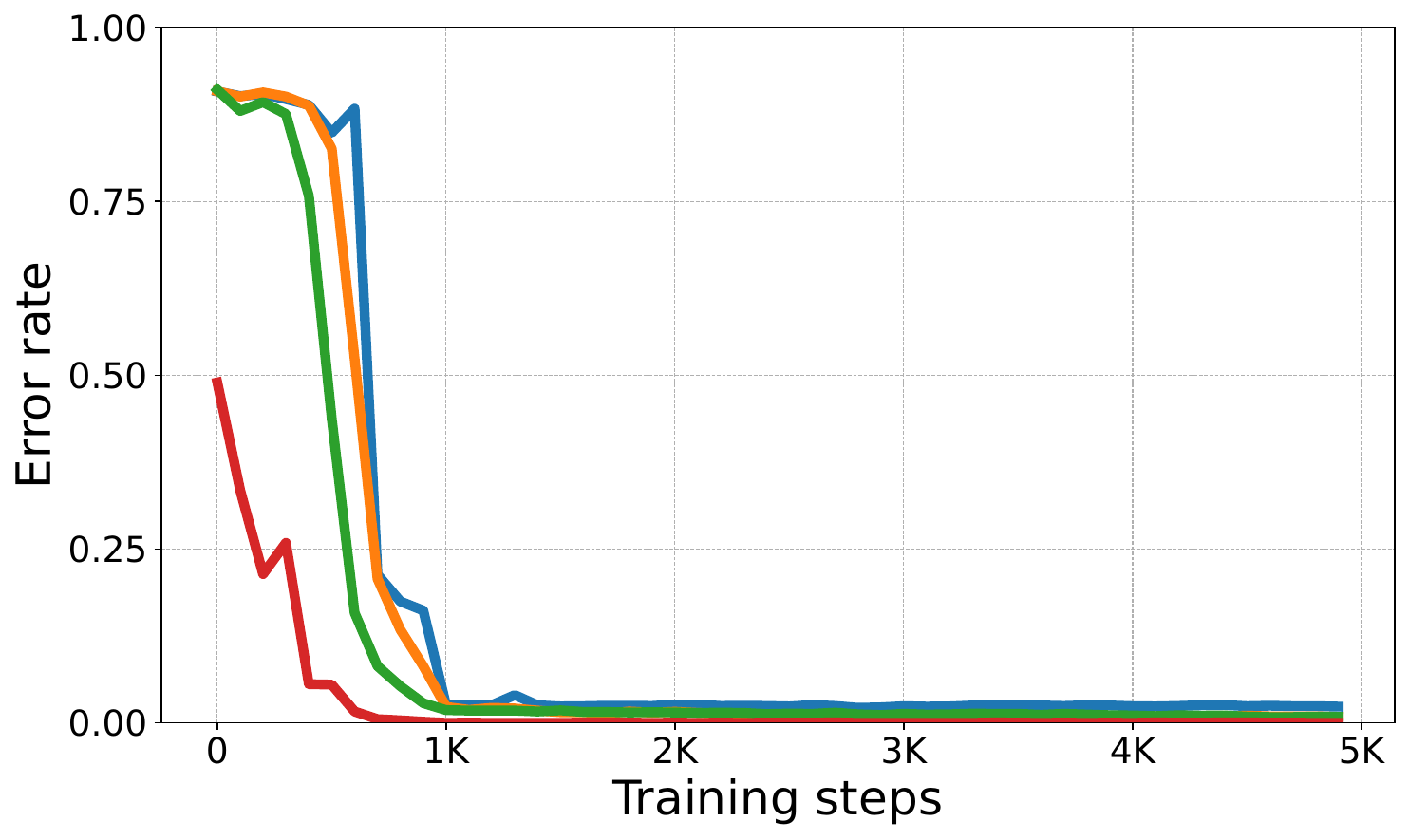} 
        \caption{2-operands addition reverse output format with greedy decoding} 
    \end{subfigure} 
    \hfill
    \begin{subfigure}{0.49\textwidth} 
        \centering 
        % Replace with your actual image file
        \includegraphics[width=\linewidth]{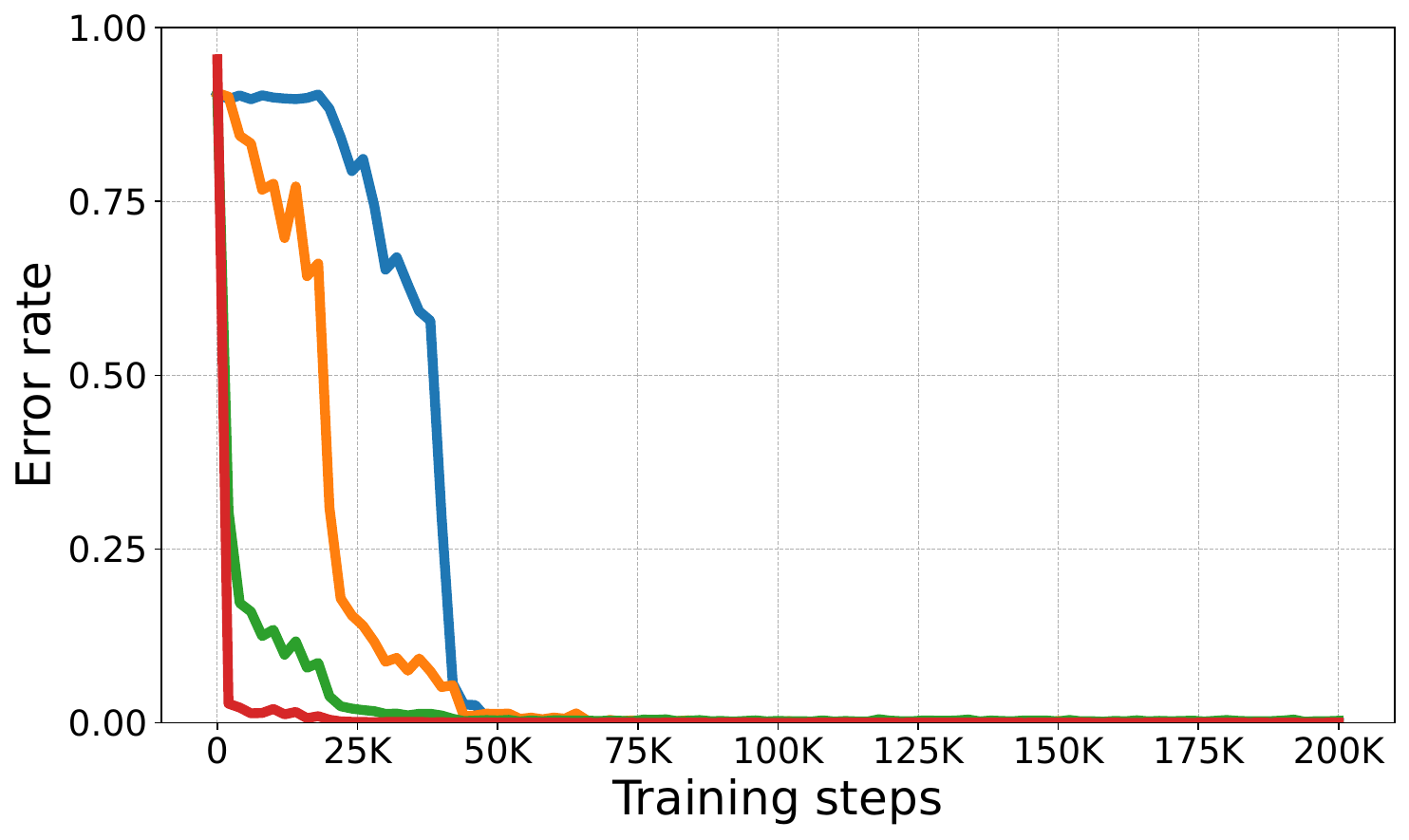} 
        \caption{4-operands addition plain output format with greedy decoding} 
    \end{subfigure} 
    \begin{subfigure}{0.49\textwidth} 
        \centering 
        % Replace with your actual image file
        \includegraphics[width=\linewidth]{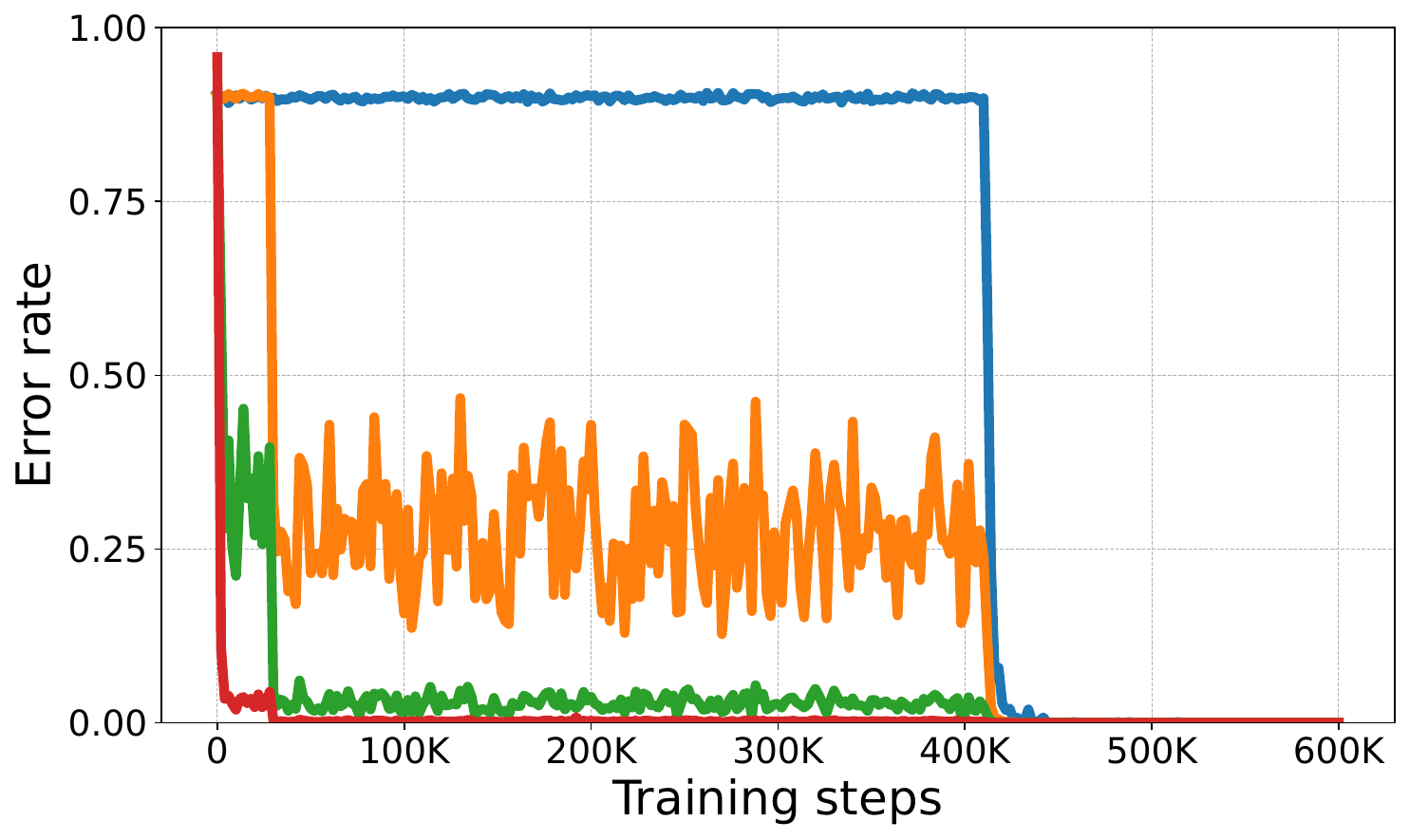} 
        \caption{4-operands addition reverse output format with greedy decoding} 
    \end{subfigure} 
    \caption{\textbf{Results for the addition task under the alternative setup and decoding scheme}. Using greedy decoding, we train transformers on 2-operands and 4-operands addition, in both plain and reverse output format, and evaluate the digit-wise error rate as training progresses. Same as the results obtained under the default setting (with temperature 0.8), models learn digits from the most significant digit to the least significant digit.} 
    
    % The label MUST be outside the comments and after the caption
    \label{fig:addition_greedy} 
\end{figure}
To verify that the choice of generation procedure does not alter the digit-wise learning dynamics we re-ran a suite of experiments using greedy decoding (implemented as top-k decoding with $k = 1$) at evaluation time.

Figure~\ref{fig:addition_greedy} shows that, as in the default sampling setup (temperature 0.8), models evaluated with greedy decoding continue to learn in the non-human order, from the most significant to the least significant digit.

\subsubsection{Different random seeds}
\label{app:addition-random-seeds}

\begin{figure}[t] 
    \centering 
    \begin{subfigure}{0.5\textwidth} 
        \centering 
        % Replace with your actual image file
        \includegraphics[width=\linewidth]{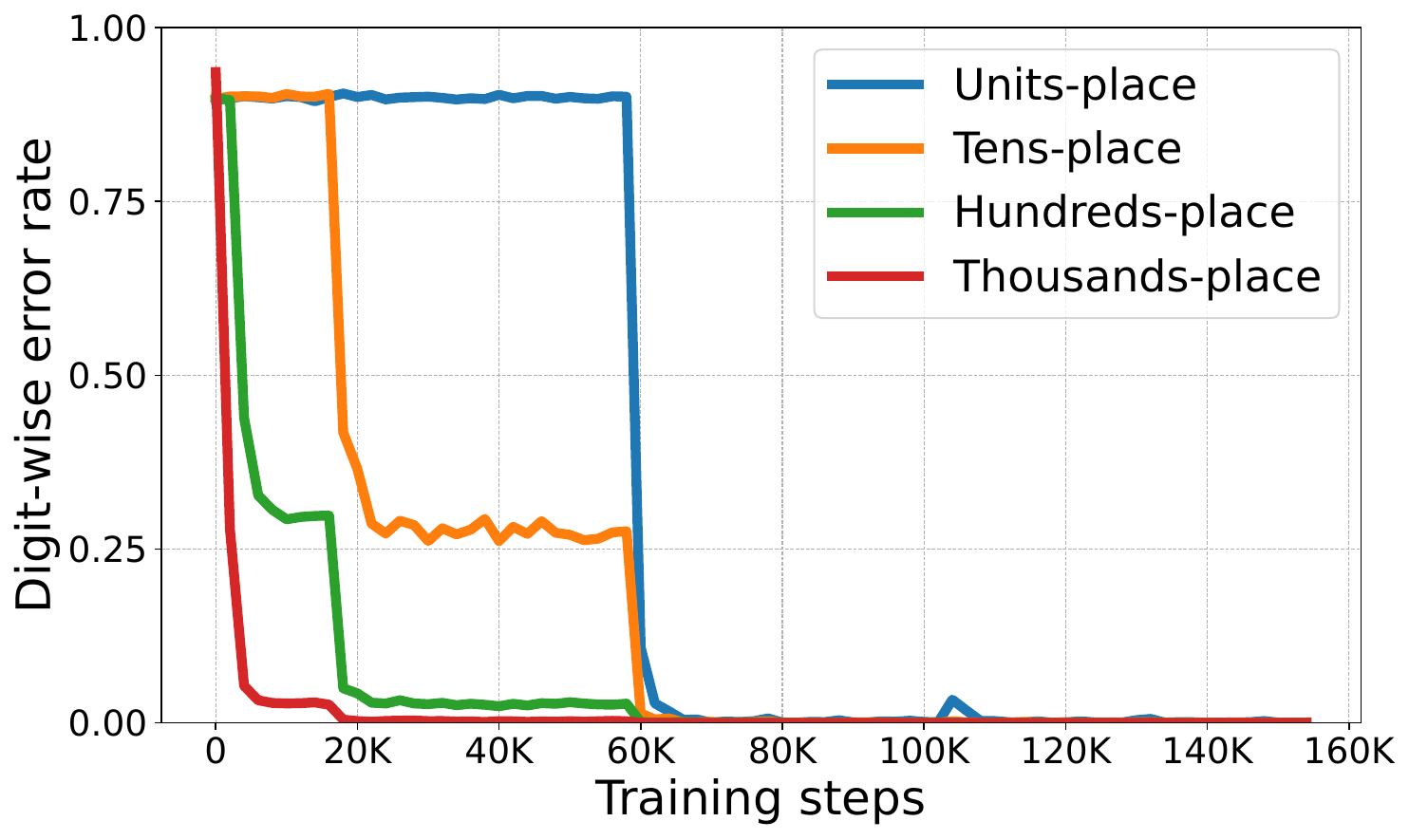} 
        \caption{Using random seed 786631946} 
    \end{subfigure}% 
    \hfill 
    \begin{subfigure}{0.5\textwidth} 
        \centering 
        % Replace with your actual image file
        \includegraphics[width=\linewidth]{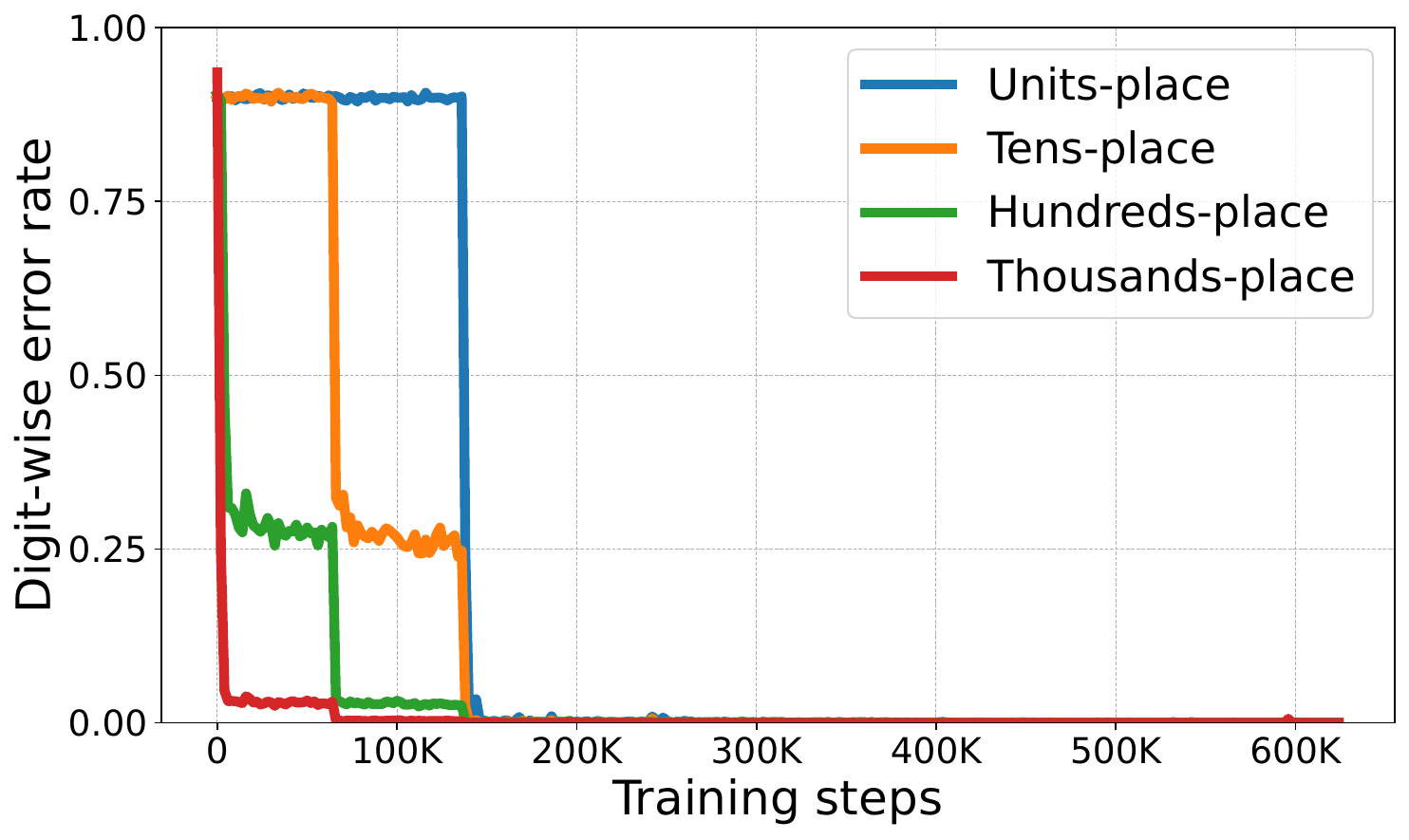} 
        \caption{Using random seed 927048795} 
    \end{subfigure} 
    
    \caption{\textbf{Additional result for the addition task using two other random seeds.} We train transformers on 4-operand addition in reverse output format, using different random seeds than the one used in Figure~\ref{fig:1}. Models always learn digits from the most significant digit to the least significant digit.} 
    \label{fig:addition_random_seeds} 
\end{figure}

We also reran the 4-operand addition experiment in the reverse output format using different random seeds, while keeping the data distribution, model architecture, and training setup fixed. As shown in Figure~\ref{fig:addition_random_seeds}, the results are qualitatively similar across seeds: despite small variation in the timing of error drops, models consistently learn digits in the reversed, non-human order, from higher-place digits to lower-place digits.

\subsubsection{Additional ablation results}

\label{appendix-ablation}
\begin{figure}[t]
    \centering
    \begin{subfigure}[b]{0.49\textwidth}
        \centering
        \includegraphics[width=\linewidth]{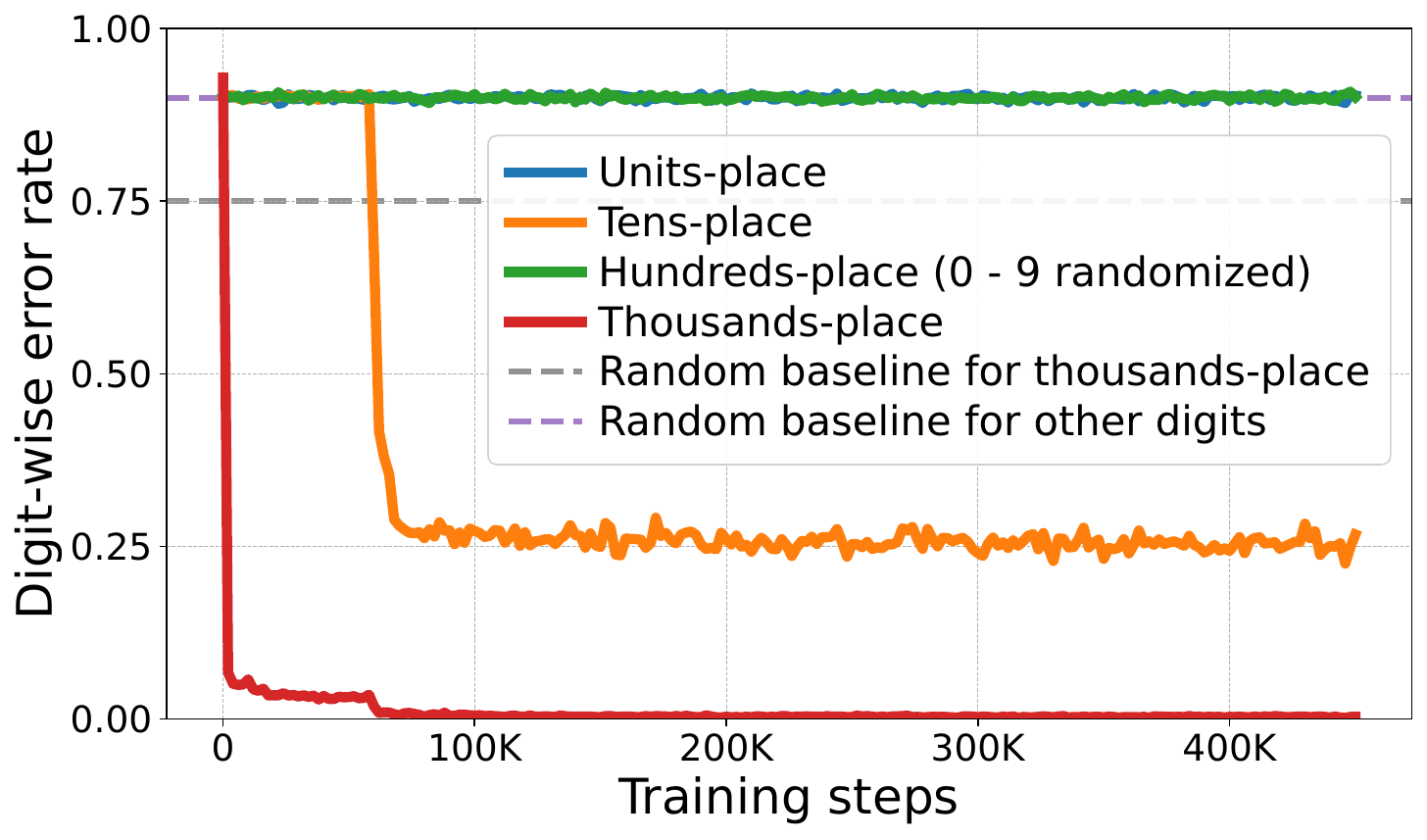}
        \caption{Randomize hundreds-place}
        \label{fig:hundreds-randomized-reverse}
    \end{subfigure}\hfill
    \begin{subfigure}[b]{0.49\textwidth}
        \centering
        \includegraphics[width=\linewidth]{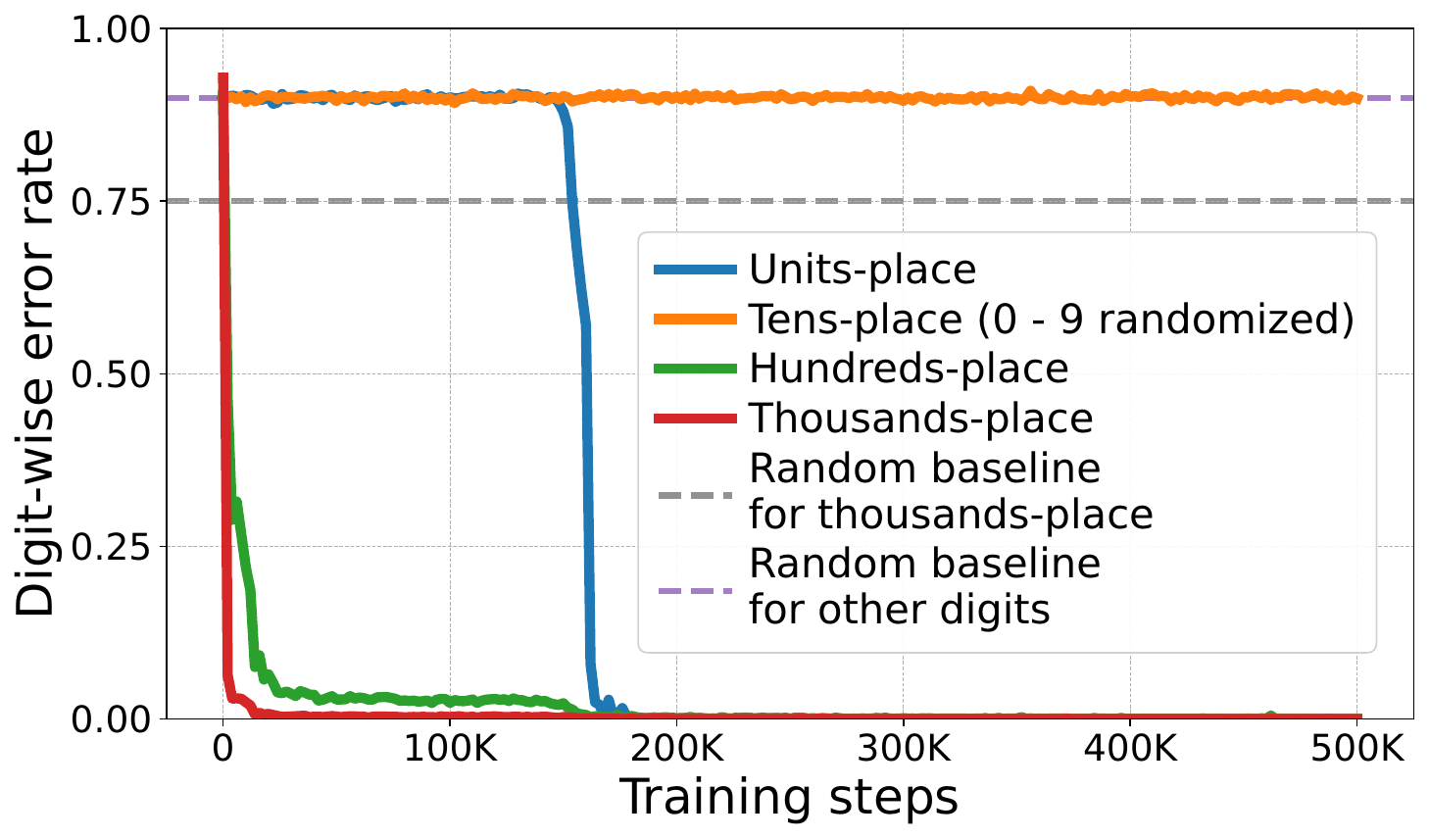}
        \caption{Randomize tens-place}
        \label{fig:tens-randomized-reverse}
    \end{subfigure}\hfill
    \begin{subfigure}[b]{0.49\textwidth}
        \centering
        \includegraphics[width=\linewidth]{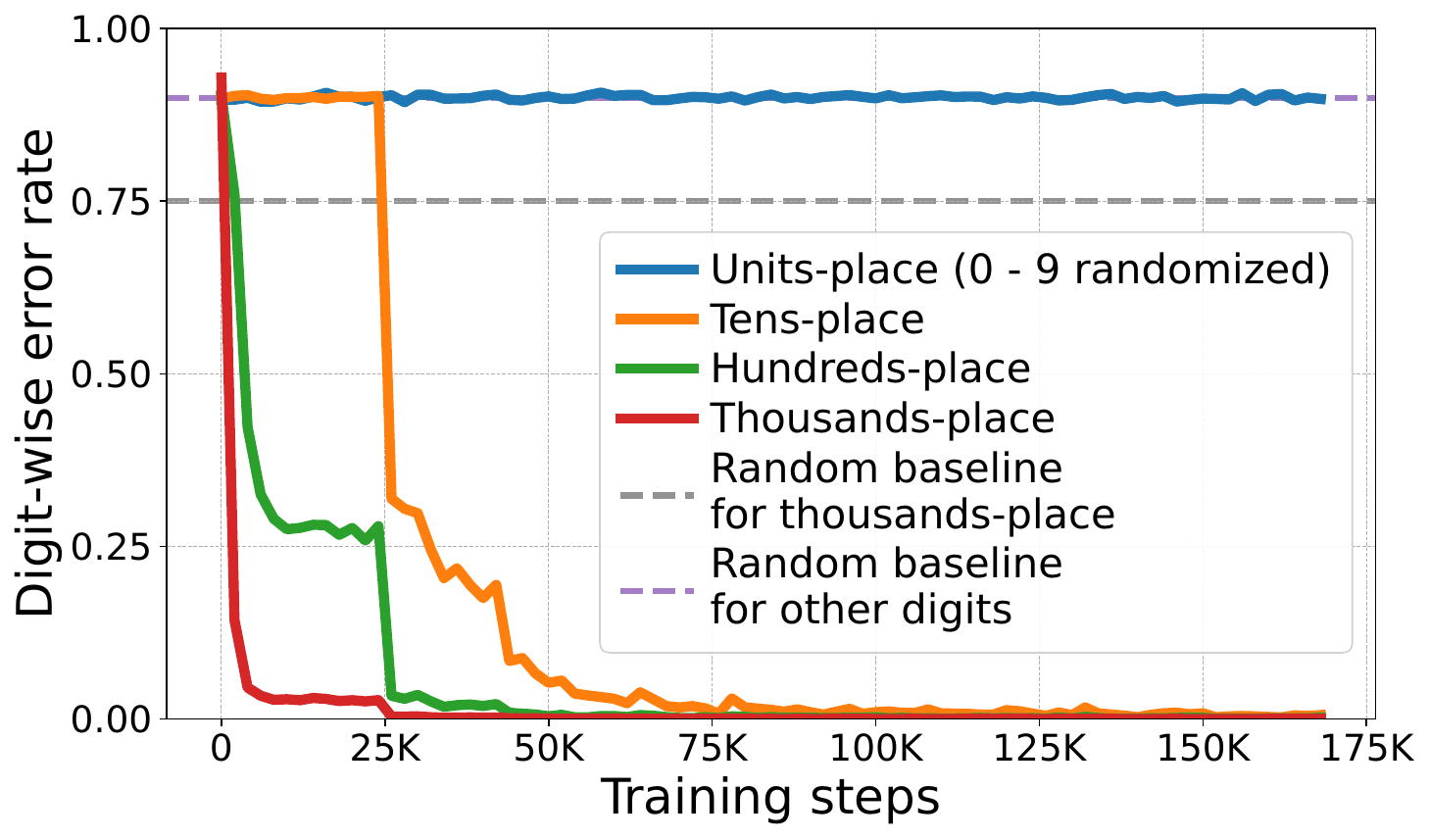}
        \caption{Randomize units-place}
        \label{fig:units-randomized-reverse}
    \end{subfigure}\hfill
   
    \caption{\textbf{Ablation experiment results on hundreds-place, tens-place and units-place with the reverse output format.} We train transformers with modified training data for addition with the reverse format, where the hundreds-place digits (a), tens-place digits (b) and units-place digits (c) are replaced by uniform digits in $\{0,1,\ldots,9\}$. Plot (a) shows that the model struggles to learn tens-place digits (error rates above 0.2) and fails to learn the unit-place digits; Plot (b) shows that the model is able to learn the units-place, hundreds- and thousands-place. Plot (c) shows that the model is able to learn all the higher digits without difficulty.  
    }
    \label{fig:addition_randomized_tens_units_reverse}
\end{figure}

\begin{figure}[t]
    \centering
    \begin{subfigure}[b]{0.49\textwidth}
        \centering
        \includegraphics[width=\linewidth]{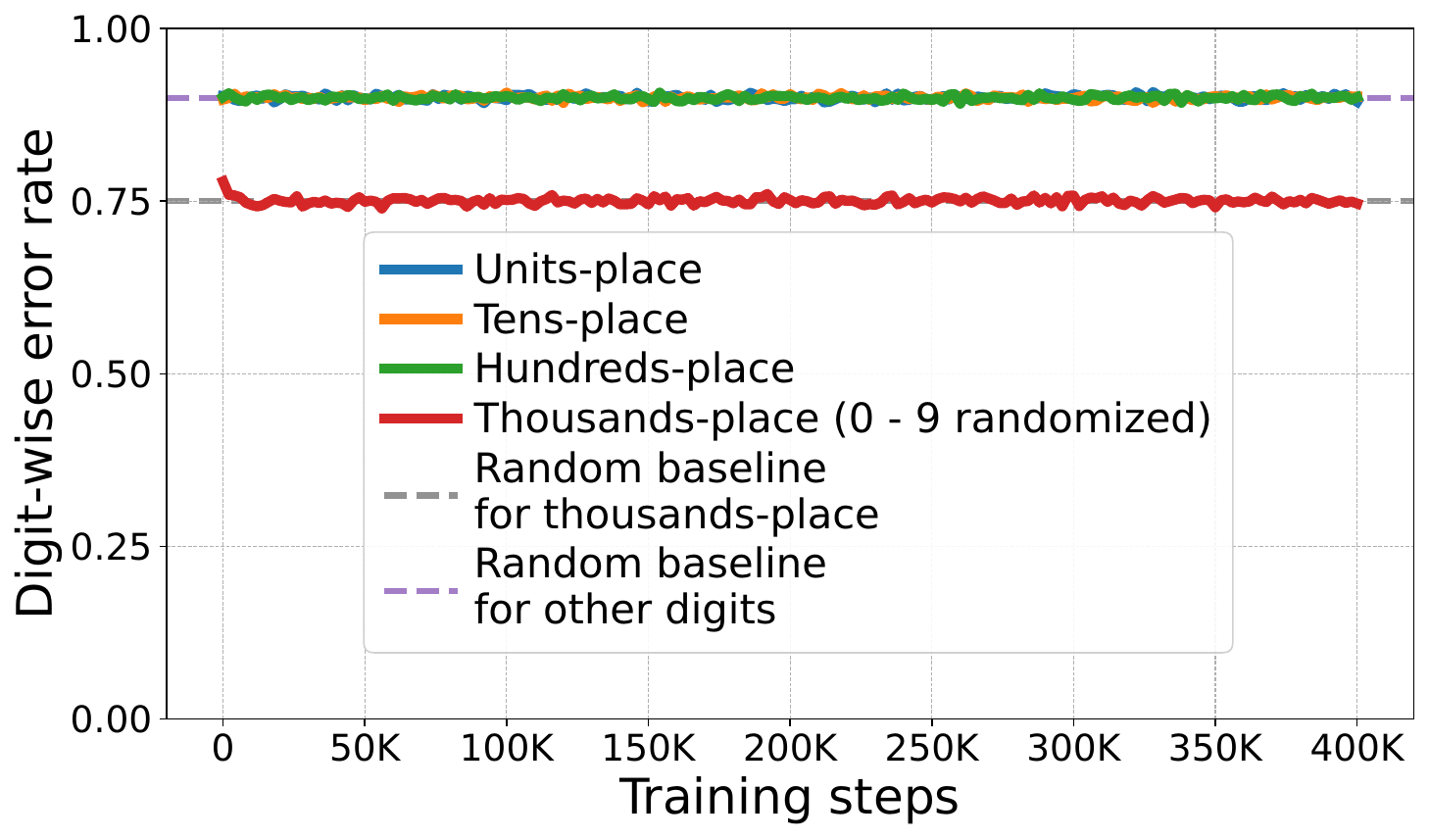}
        \caption{Randomize thousands-place}
        \label{fig:addition_randomized_plain_a}
    \end{subfigure}\hfill
    \begin{subfigure}[b]{0.49\textwidth}
        \centering
        \includegraphics[width=\linewidth]{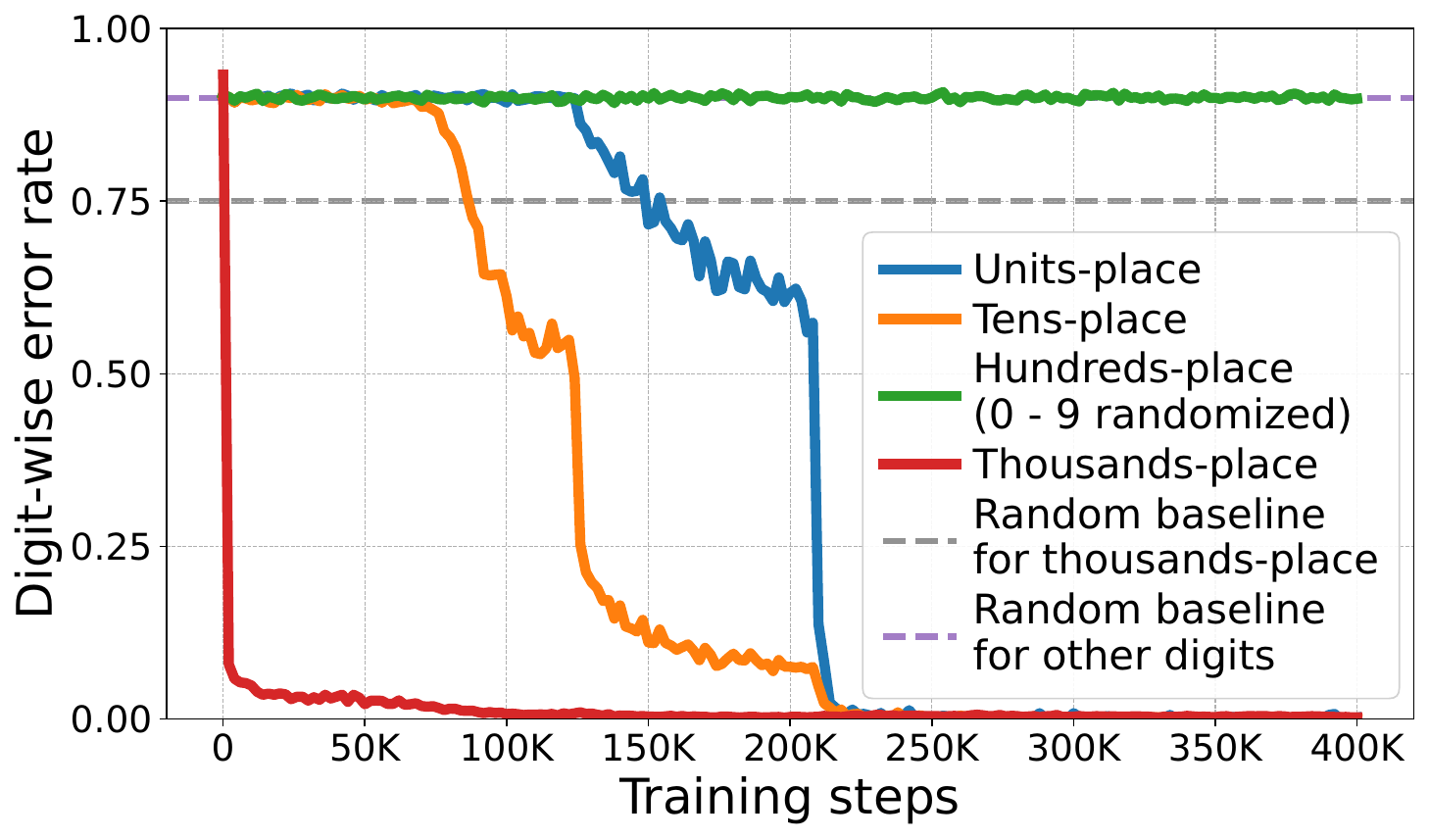}
        \caption{Randomize hundreds-place}
        \label{fig:addition_randomized_plain_b}
    \end{subfigure}\hfill
   
    \caption{\textbf{Ablation experiment results for plain output format addition.} We train transformers with modified training data for addition with the plain format, where the thousands-place digits (left) and hundreds-place digits (right) %\YZ{Is the subcaptions correct?} \PX{ They were blatantly wrong. Updated.}
    are replaced by uniform digits in $\{0,1,2,3\}$ and $\{0,1,\ldots,9\}$ respectively. Left plot shows that the model fails to learn any lower digit, whose digit-wise error rates are not better than random baselines. Right plot shows that when randomizing the hundreds-place, the model is still able to learn the tens-place and units-place, but taking about 5 times longer than without randomization.
    }
    \label{fig:addition_randomized_plain}
\end{figure}
In the main paper we showed the results for randomizing the thousands-place of the output when training addition in the reverse output format. We find similar results when randomizing the hundreds-place of the output. As shown in Figure~\ref{fig:addition_randomized_tens_units_reverse}, the model struggles to learn lower digits (tens-place and units-place). When randomizing tens-place in the reverse output format, we find the model is still able to learn the units-place. If randomizing the units-place, the model can learn all the higher digits without difficulty.

We also experimented with ablation when training in the plain format. As shown in Figure~\ref{fig:addition_randomized_plain}, when the thousands-place is randomized, the model fails to learn any lower digits. However, when we randomize the hundreds-place, the model still manages to learn the tens-place and units-place, but taking about 5 times longer than without randomization. As a comparison, with the hundreds-place randomized, the model takes around 125K and 200K to learn tens-place and units-place respectively, while without randomization, the model takes only around 25K and 40K steps to learn tens-place and units-place respectively, as shown in the left subplot of Figure~\ref{fig:1}.

Back to our digit-place ablation experiment, to rule out irreducible loss and noisy gradients from randomized labels as the main explanation for the lower-digit learning failure, we additionally train on four-operand addition with correct targets
in reverse output order.
We set the prediction loss for either the thousands-place digit
or the hundreds-place digit to zero whenever that digit is present
in the target sequence, keeping all other training settings unchanged.
As shown in Figure~\ref{fig:addition-zero-loss}, thousands-place
loss masking leaves the units-, tens-, and hundreds-place errors
near 90\% across the plotted evaluations. With hundreds-place loss
masking, the units- and tens-place errors likewise remain near
90\%, while the supervised thousands-place error decreases but
fluctuates and remains substantial.
Thus, randomized target labels are not necessary for this observed lower-digit learning failure.

\begin{figure}[t]
    \centering
    \begin{subfigure}[b]{0.49\textwidth}
        \centering
        \includegraphics[width=\linewidth]{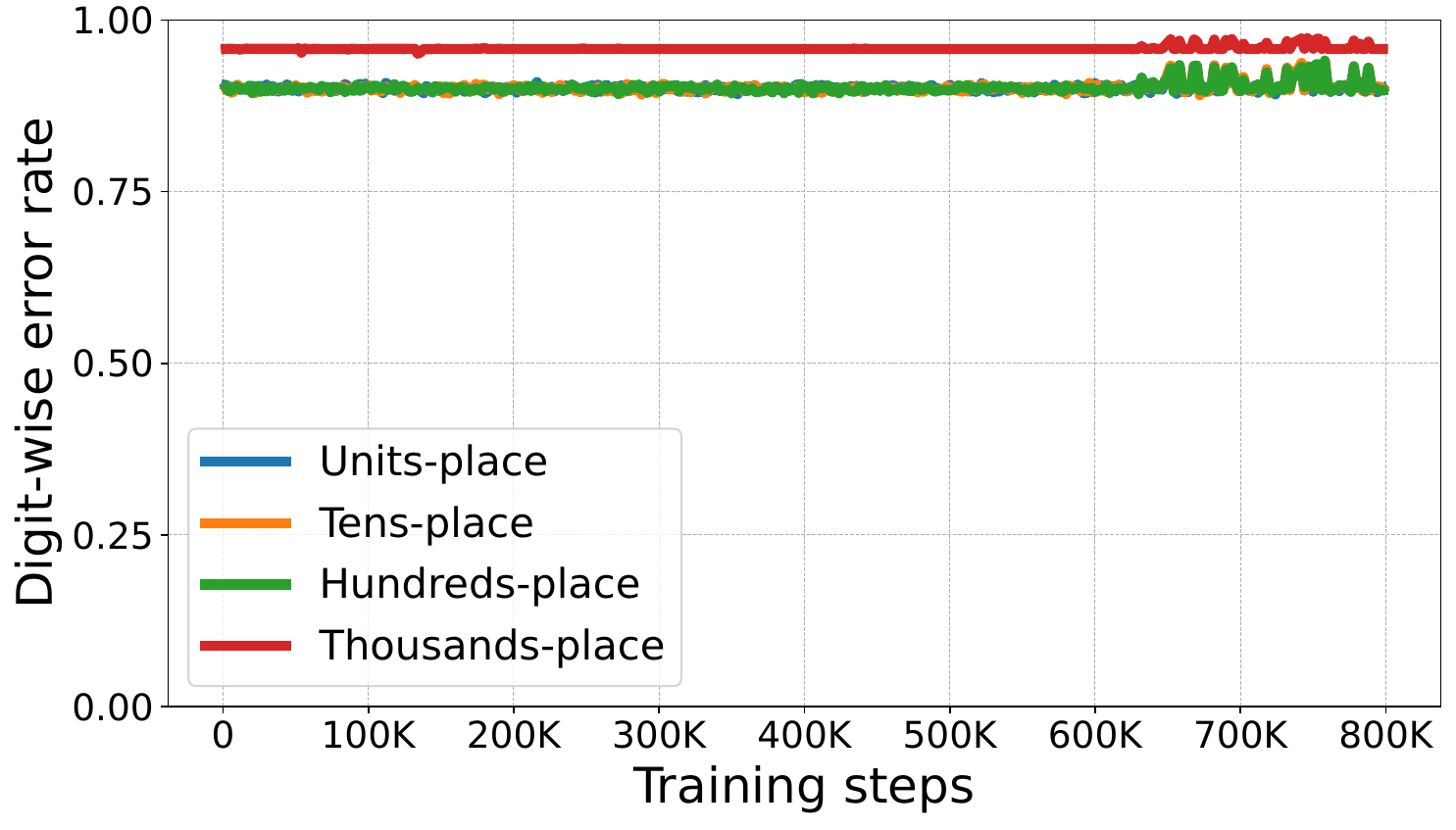}
        \caption{Thousands-place loss set to zero}
        \label{fig:addition-thousands-zero-loss}
    \end{subfigure}\hfill
    \begin{subfigure}[b]{0.49\textwidth}
        \centering
        \includegraphics[width=\linewidth]{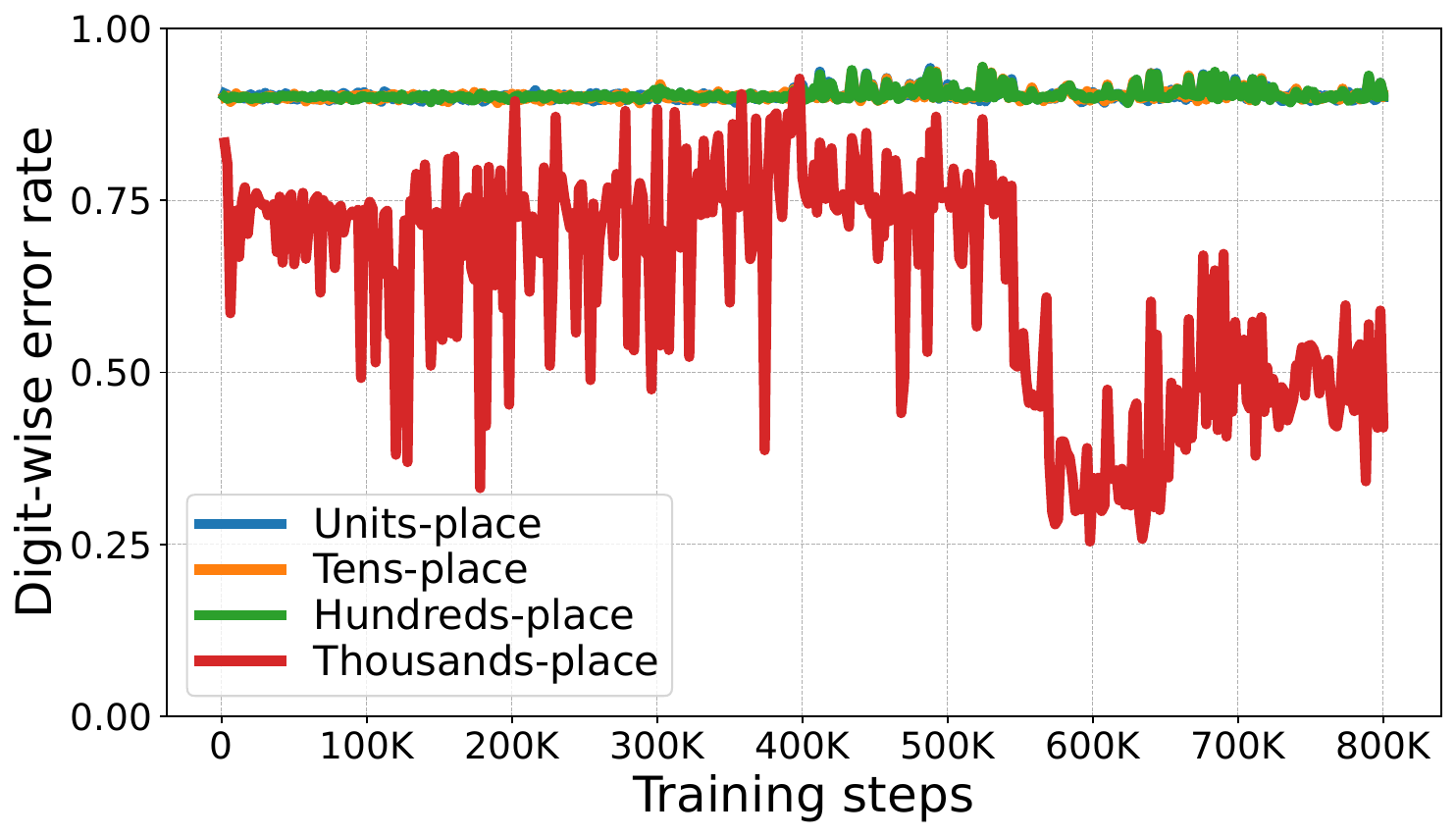}
        \caption{Hundreds-place loss set to zero}
        \label{fig:addition-hundreds-zero-loss}
    \end{subfigure}

    \caption{\textbf{Digit-place loss-masking controls for reverse-output addition.}
    We train transformers for 800K steps on four-operand addition, with each
    operand sampled independently and uniformly from $\{0,\ldots,999\}$.
    Targets remain correct; only the prediction loss for the explicit
    thousands-place digit (left) or hundreds-place digit (right) is set to zero.
    With thousands-place masking, the three lower-place errors remain near
    90\%. With hundreds-place masking, the units- and tens-place errors remain
    near 90\%, while the thousands-place error decreases but remains substantial.}
    \label{fig:addition-zero-loss}
\end{figure}

\subsubsection{Prediction error distribution and a numerical approximator}
\label{appendix-approximator}

\begin{figure*}[ht]
  \begin{center}
    \begin{subfigure}[b]{0.34\textwidth}
        \centering
        \includegraphics[width=0.9\linewidth]{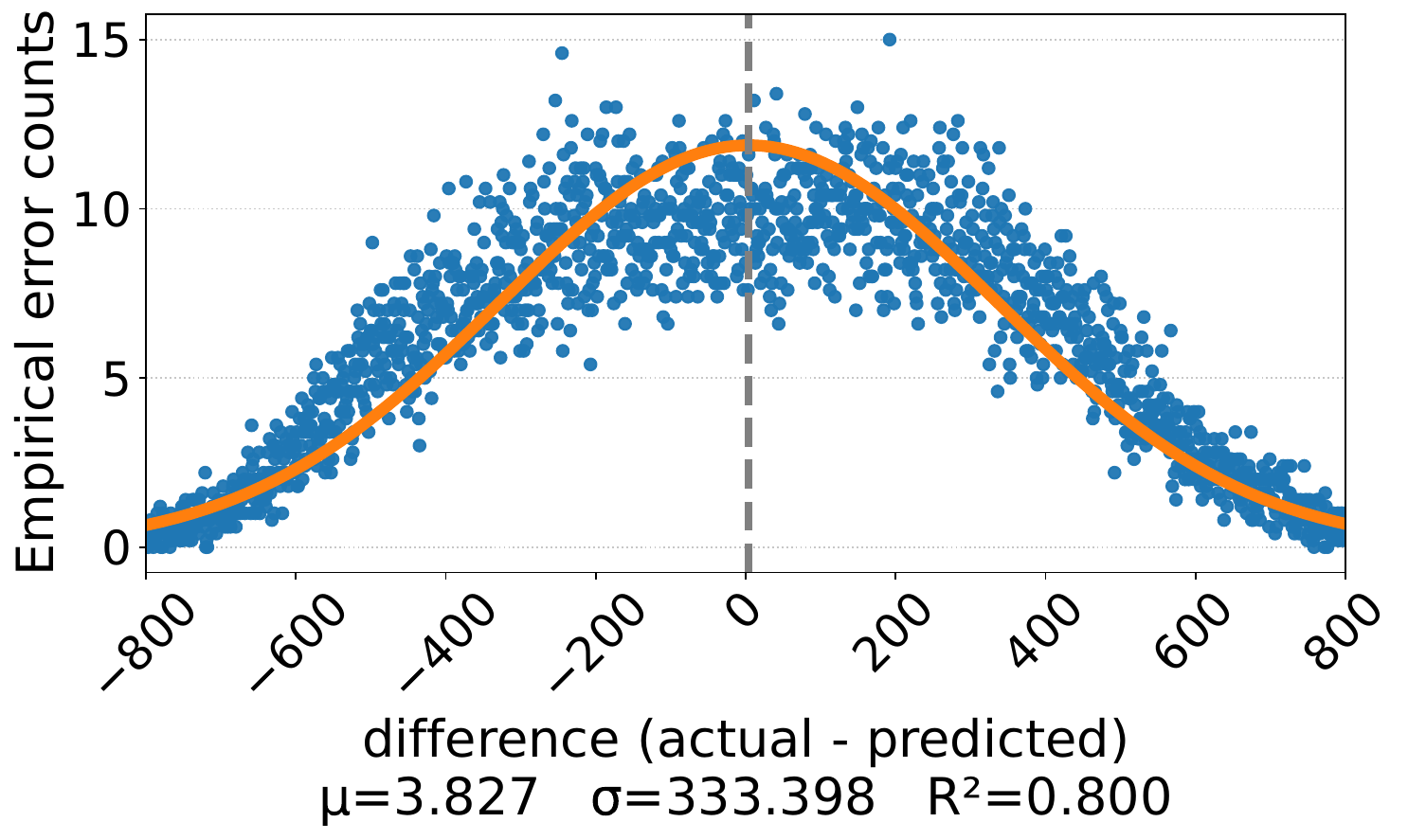}
        \caption{Training step 1,000 to 1,800}
        \label{fig:2a}
    \end{subfigure}\hfill
    \begin{subfigure}[b]{0.32\textwidth}
        \centering
        \includegraphics[width=0.95\linewidth]{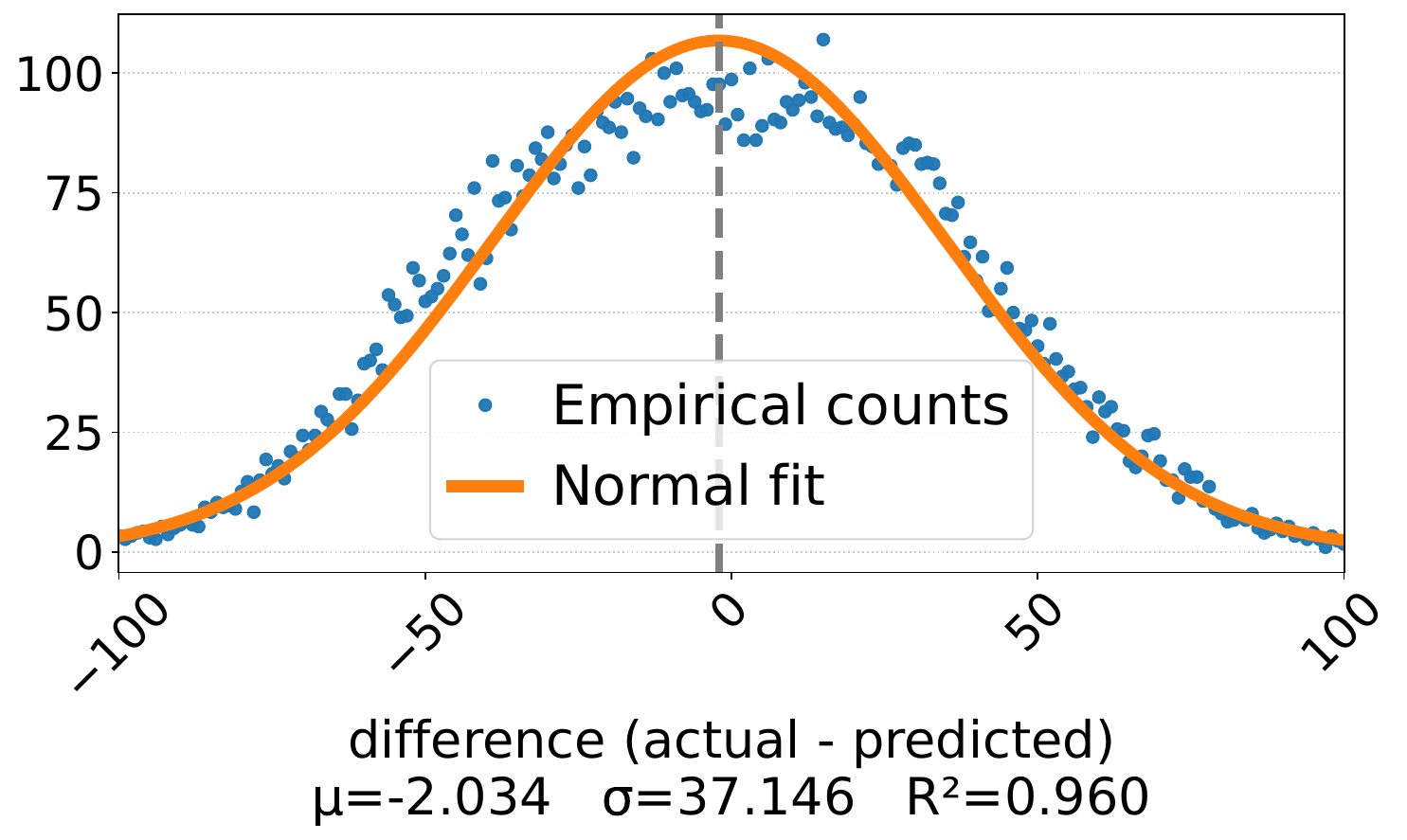}
        \caption{Training step 8,000 to 12,000}
        \label{fig:2b}
    \end{subfigure}\hfill
    \begin{subfigure}[b]{0.32\textwidth}
        \centering
        \includegraphics[width=0.95\linewidth]{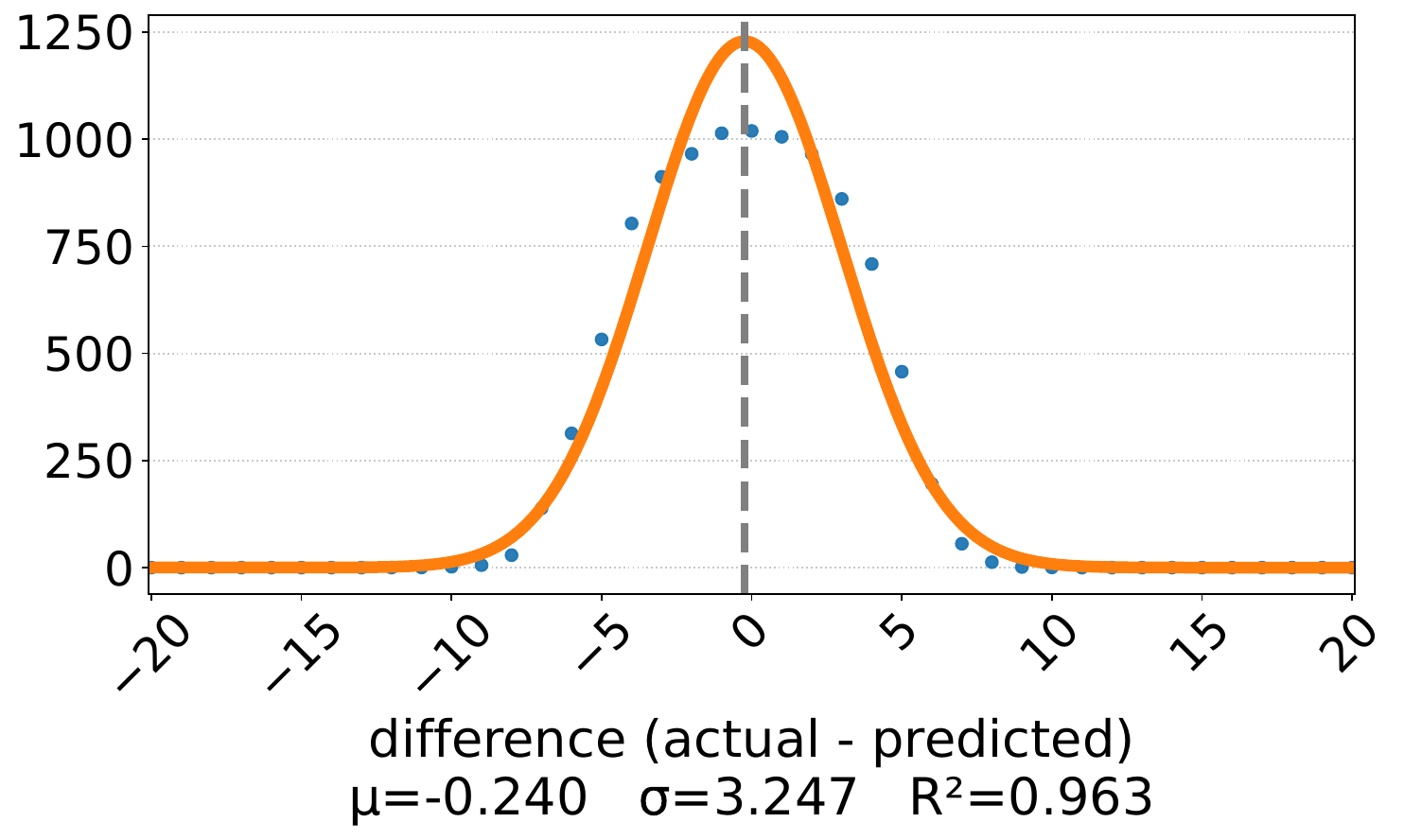}
        \caption{Training step 60,000 to 64,000}
        \label{fig:2c}
    \end{subfigure}
    
    \caption{
\textbf{Transformers learn addition similar to an improving approximation algorithm}, narrowing the error spread as training continues. When training the model for addition in reverse format, we compare the model's predicted integer $\widehat e$ and the groundtruth integer $e'$. The scatterplots are instances of the difference $e' - \widehat e$ collected over three intervals. We then fit a normal distribution for each scatterplot. Despite trained on raw digits, the model learns to approximate the correct solution with increasingly smaller Gaussian error. 
    }
    \label{fig:2}
  \end{center}
  % \vskip -0.2in
\end{figure*}

This analysis complements the digit-wise results with an integer-level view. The reverse learning order already suggests that the typical magnitude of prediction error should shrink as the model learns higher-order digits first. Figure~\ref{fig:2} shows an additional and stronger pattern: even when lower digits are still frequently wrong, the predicted integers are already concentrated around the correct value, rather than being dispersed as if they arose from unrelated token-level mistakes.

For each test example, we compute the signed integer residual $e' - \widehat e$ between the groundtruth integer $e'$ and the model prediction $\widehat e$. To reduce checkpoint-to-checkpoint noise, we aggregate residuals across three evaluation windows corresponding to different stages of training. The resulting empirical residual distributions are well approximated by Gaussians whose variance decreases over training.

This suggests that the model behaves like an increasingly refined approximate predictor at the integer level. In other words, before it masters all output digits exactly, it already tends to produce integers that are numerically close to the target, and training progressively tightens this residual around zero. This observation is complementary to the digit-wise learning order: digit-wise error tells us \emph{which} positions are learned first, whereas the residual distribution shows that intermediate predictions remain globally tied to the correct integer value.

\paragraph{Error counts.}
In Figure~\ref{fig:2}, we aggregate residuals over the following windows: (a) left: iterations $1$K, $1.2$K, $1.4$K, $1.6$K, and $1.8$K; (b) middle: $8$K, $10$K, and $12$K; (c) right: $60$K, $62$K, and $64$K. We use windows rather than single checkpoints to smooth the bumpiness of the training dynamics, since at an individual checkpoint the empirical mean of the residual may fluctuate around $0$. Note also that the x-axis range changes from $[-800,800]$ in the left panel to $[-100,100]$ in the middle panel, and to $[-20,20]$ in the right panel.

\paragraph{Coefficient of determination ($R^2$).}
In Figure~\ref{fig:2}, we report the $R^2$ score for each Gaussian fit. We adopt the standard definition:
\[
SS_{\mathrm{res}}=\sum_{i=1}^n (y_i-\hat y_i)^2,\qquad
SS_{\mathrm{tot}}=\sum_{i=1}^n (y_i-\bar y)^2,
\]
where $\bar y=\frac{1}{n}\sum_{i=1}^n y_i$, and
\[
R^2 = 1-\frac{SS_{\mathrm{res}}}{SS_{\mathrm{tot}}}.
\]

% \subsubsection{Prediction error distribution and a numerical approximator}
% \label{appendix-approximator}

% In the main paper, we aggregated the model's prediction error counts among the evaluated steps in three chosen windows, which correspond to three stages in learning addition. The reason why we choose to aggregate among a window, instead of just reporting the counts at a single evaluated step, is to smooth out the bumpiness of training dynamics, as we find that for a single training step the mean of the prediction error may swing around 0.

% \paragraph{Error counts.} In Figure~\ref{fig:2}, we gather the errors on test data from evaluation in three windows, in order to smooth out ``bumpiness'' of a single checkpoint: (a) left plot: iteration $1$K, $1.2$K, $1.4$K, $1.6$K, $1.8$K; (b) middle plot: $8$K, $10$K \& $12$K. (Note that the x axis range has changed from $[-800,+800]$ to $[-100,+100]$); (c) right plot: iteration $60$K, $62$K \& $64$K. 

% \paragraph{Coefficient of determination ($R^2$).}
% In Figure~\ref{fig:2}, we reported $R^2$ scores for each normal fitting. We adopt the following definition of $R^2$.

% Given observed responses $\{y_i\}_{i=1}^n$ and corresponding model predictions $\{\hat y_i\}_{i=1}^n$, we define the residual sum of squares
% \[
% SS_{\mathrm{res}}=\sum_{i=1}^n (y_i-\hat y_i)^2
% \]
% and the total sum of squares about the sample mean $\bar y=\frac{1}{n}\sum_{i=1}^n y_i$,
% \[
% SS_{\mathrm{tot}}=\sum_{i=1}^n (y_i-\bar y)^2.
% \]
% The coefficient of determination is
% \[
% R^2 = 1-\frac{SS_{\mathrm{res}}}{SS_{\mathrm{tot}}}.
% \]

\subsubsection{Additional mutual information metric results}
\label{app:actual-carry-check}
\begin{figure}[t] 
    \centering 
    \begin{subfigure}{0.5\textwidth} 
        \centering 
        % Replace with your actual image file
        \includegraphics[width=0.9\linewidth]{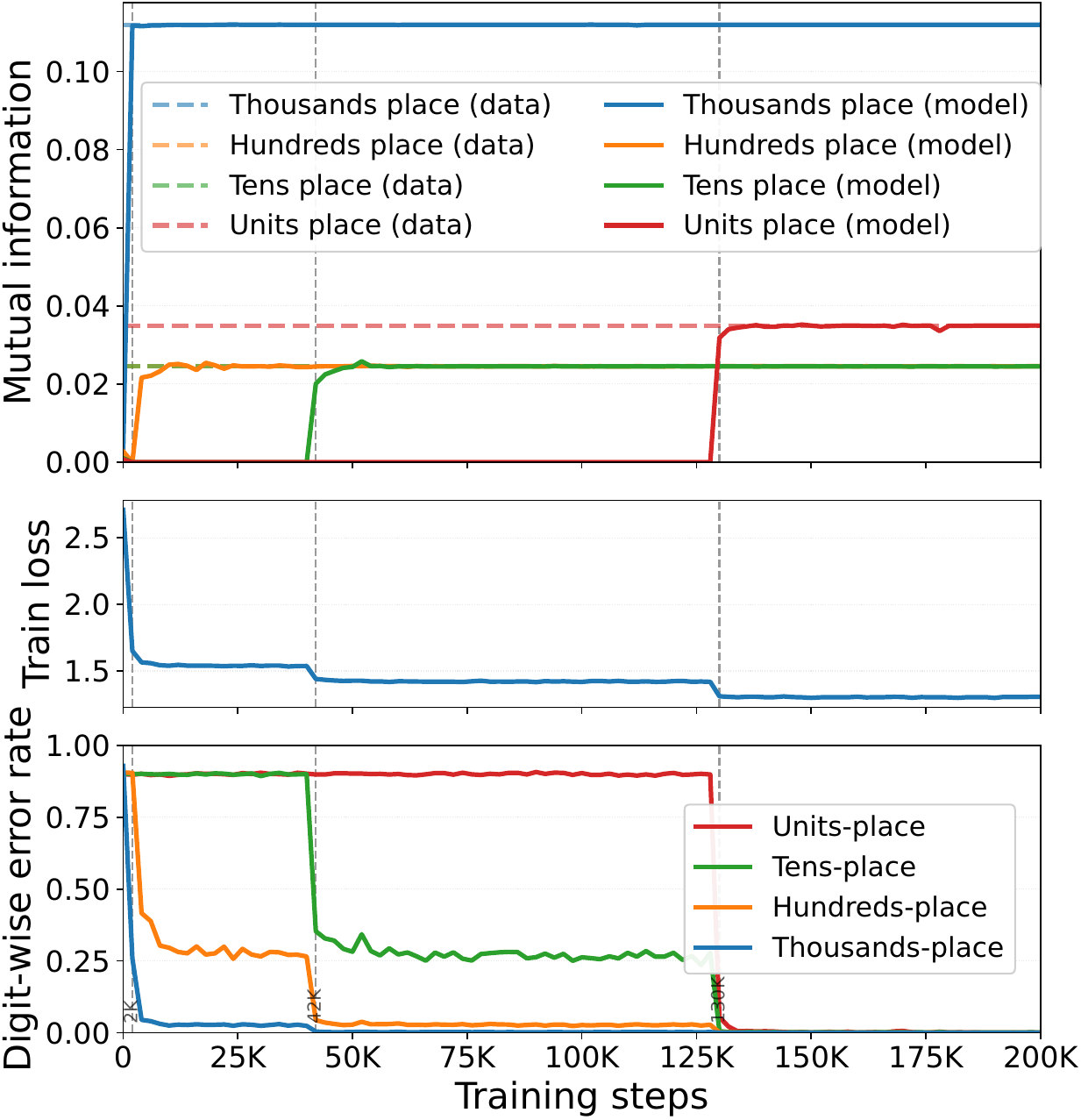} 
        % \caption{MI conditioned on carries} 
        \caption{MI conditioned on local carries}
    \end{subfigure}% 
    \hfill 
    \begin{subfigure}{0.5\textwidth} 
        \centering 
        % Replace with your actual image file
        \includegraphics[width=0.9\linewidth]{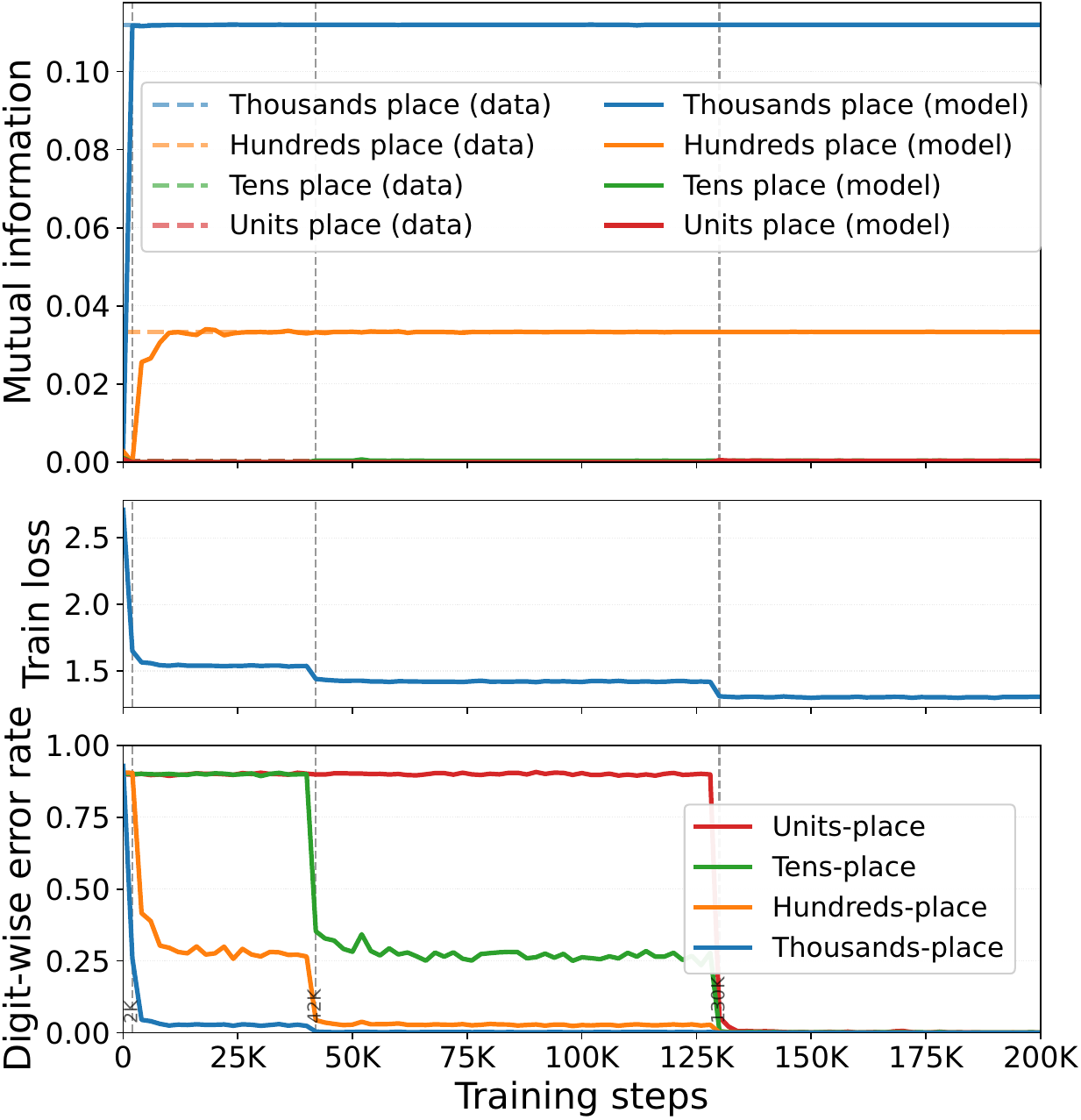} 
        \caption{MI conditioned on immediate higher digit} 
    \end{subfigure} 
    
    \caption{
    \textbf{Left: }In the left top subfigure, we track the MI metrics $I(a_1,\widehat p_0)$ (highest digit) and $I(a_i,\widehat p_i\mid K_{i-1})$ (other digits) using the model's prediction probabilities across training (solid curves), which are compared against the corresponding $I(a_1,e_0)$ and $I(a_i,e_i\mid K_{i-1})$ based on the training data (dashed line), where $K_{i-1}$ denotes local carry.
    % \textbf{Left: }In the left top subfigure, we track the MI metrics $I(a_1,\widehat p_0)$ (highest digit) and $I(a_i,\widehat p_i|c_{i-1})$ (other digits) using the model's prediction probabilities across training (solid curves) , which are compared against the corresponding $I(a_1,e_0)$ and $I(a_i,e_i|c_{i-1})$ based on the training data (dashed line). 
    \textbf{Right: }In the right top subfigure, we track the MI metrics $I(a_1,\widehat p_0)$ (highest digit) and $I(a_i,\widehat p_i|e_{i-1})$ (other digits) using the model's prediction probabilities across training (solid curves), which are compared against the corresponding $I(a_1, e_0)$ and $I(a_i,e_i|e_{i-1})$ based on the training data (dashed lines), with $I(a_2,e_2|e_1) = 0, I(a_3,e_3|e_2) = 0$. The sharp descent of the training loss (middle) and digit-wise error rates (bottom) correspond to the model’s sudden matching with MI metrics at each output digit, indicating skill acquisition. %\PX{Updated to test on MI data of 1M lines, vertical lines adjusted}
    } 

    \label{fig:addition_mi_wth_digitwise_error} 
\end{figure}
\begin{figure}[t]
\centering
\includegraphics[width=0.55\textwidth]
{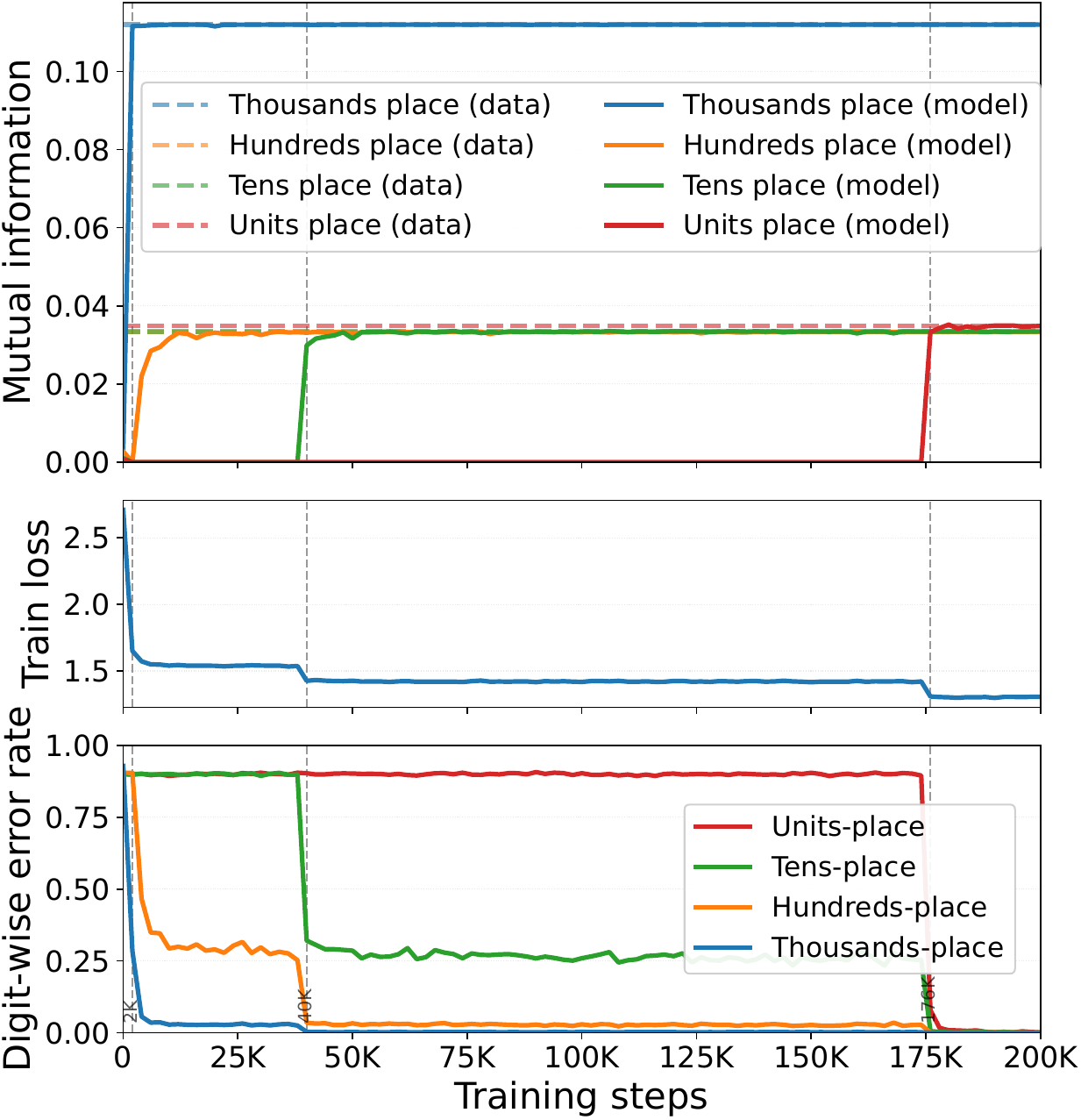}
\caption{
\textbf{Actual-carry MI in an additional reverse-output training run.}
\textbf{Top:} Model MI (solid curves),
$I(a_1,\widehat p_0)$ and
$I(a_i,\widehat p_i\mid c_{i-1})$ for $i=1,2,3$,
compared with the corresponding dataset MI
(dashed lines), $I(a_1,e_0)$ and
$I(a_i,e_i\mid c_{i-1})$.
Here $c_{i-1}$ denotes actual, propagated carry.
\textbf{Middle:} Training loss.
\textbf{Bottom:} Digit-wise test error.
All panels describe the same additional run,
with measurements every 2K steps.
Vertical dashed guides mark checkpoints at
2K, 40K, and 176K steps.
}
\label{fig:actual-carry-check}
\end{figure}
In the main paper, we measured the MI $I(a_1,\widehat p_0)$ (highest digit) and MI conditioned on the local carry-out $I(a_i,\widehat p_i\mid K_{i-1})$ (other digits), which are compared against the corresponding $I(a_1,e_0)$ and $I(a_i,e_i\mid K_{i-1})$ based on the training data. We showed that the sharp decline of training loss corresponds to the model's sudden matching with MI metrics at each output digit.
% In the main paper, we measured the MI $I(a_1,\widehat p_0)$ (highest digit) and MI conditioned on the carry to its immediate higher digit $I(a_i,\widehat p_i|c_{i-1})$ (other digits), which are compared against the corresponding $I(a_1,e_0)$ and $I(a_i,e_i|c_{i-1})$ based on the training data. We showed that the sharp decline of training loss corresponds to the model's sudden matching with MI metrics at each output digit. 

In Figure~\ref{fig:addition_mi_wth_digitwise_error} (right), we also track the MI conditioned on the immediate higher digit $I(a_i,\widehat p_i|e_{i-1})$ (for digits other than the highest one), which are compared against the corresponding $I(a_i,e_i|e_{i-1})$ based on the training data. ($I(a_2,e_2|e_1) = 0, I(a_3,e_3|e_2) = 0$) The figure shows that these MI matchings correspond to the sudden drops of the digit-wise error rates as well.

\paragraph{Additional actual-carry experiment.}
The empirical MI in Figures~\ref{fig:mi-addition}
and~\ref{fig:addition_mi_wth_digitwise_error}(a)
uses local carry as the conditioning variable, which ignores incoming carry from lower columns, whereas the actual carry $c_{i-1}$ in Theorem~\ref{thm:1} can be off by one due to the propagated carry-in from lower places.
Figure~\ref{fig:actual-carry-check} reports the MI analysis using the actual carry as the conditioning variable in a separate reverse-output training run.
The results exhibit the same qualitative
higher-to-lower digit-learning order and
alignment between MI matching, loss decreases,
and digit-wise error reductions.
Transition times differ from the original figures
because this is a separate training trajectory.

\subsubsection{Distribution-shift tests}
\label{app:distribution-shift-tests}

We further evaluate whether the trained models remain accurate on structured out-of-distribution addition problems. 
For each operand count $n \in \{2,3,4\}$, we train a NanoGPT model on $n$-operand addition, where each operand is uniformly sampled from $\{0,\ldots,999\}$ and the result is written in the reverse format described in Section~\ref{sec:experimental-setup}. 
We evaluate the checkpoints whose final in-distribution test accuracy exceeds $97\%$.

Table~\ref{tab:shifted-tests} summarizes the three shifted test sets. 
Each test targets a different kind of structure that is underrepresented by uniform random sampling: unusually difficult carry propagation, highly correlated operands, and sums near digit-boundary edge cases.

\begin{table}[t]
\caption{
Summary of the distribution-shifted test sets. 
Examples are shown in standard human-readable order for clarity.
}
\label{tab:shifted-tests}
\centering
\small
\begin{tabular}{p{0.24\linewidth} p{0.68\linewidth}}
\toprule
\textbf{Test set} & \textbf{Construction} \\
\midrule
Hard carry &
All operands are three-digit numbers, and the carry at each digit position is maximal. 
For $n$ operands, the carry in each column is required to be $n-1$. 
For example, in the two-operand case, $999+999=1998$ has a carry of $1$ in each of the ones, tens, and hundreds columns. \\
\midrule
Close numbers with same prefix &
All operands are three-digit numbers sharing the same hundreds and tens digits; only the ones digit varies. 
For example, $314+316+311+319=1260$. \\
\midrule
Edge-result sums &
The target sum is constrained to lie in $[a-2,a+2]$, where $a$ is a multiple of $1000$. 
For example, $890+109=999$, $928+724+349=2001$, and $456+982+931+631=3000$. \\
\bottomrule
\end{tabular}
\end{table}

Table~\ref{tab:shifted-test-accuracy} reports the resulting test accuracies. 
Although all models remain generally accurate on these shifted evaluations, the accuracy decreases as the number of operands increases,  especially on hard-carry and edge-result examples.

\begin{table}[t]
\caption{
Accuracy on distribution-shifted addition tests for models trained on 2-, 3-, and 4-operand addition. 
Each model is selected from checkpoints whose final in-distribution test accuracy is above $97\%$.
}
\label{tab:shifted-test-accuracy}
\centering
\small
\begin{tabular}{lccc}
\toprule
\textbf{Distribution-shifted test set} 
& \textbf{2 operands} 
& \textbf{3 operands} 
& \textbf{4 operands} \\
\midrule
Hard carry & $99.99\%$ & $99.48\%$ & $95.17\%$ \\
Close numbers with same prefix & $99.98\%$ & $99.15\%$ & $97.07\%$ \\
Edge-result sums & $98.03\%$ & $97.22\%$ & $95.25\%$ \\
\bottomrule
\end{tabular}
\end{table}

% \subsubsection{Distribution shift tests}

% We train transformers on addition of different number of operands, and use different set of tests to evaluate their robustness to distribution shifts. Specifically, we train three NanoGPT models on 2/3/4-operand addition, with each operand uniformly sampled from 0 to 999 and the result written in the reverse format as described in Section \ref{sec:experimental-setup}. In each setting, we evaluate the trained model whose final in-distribution test accuracy is greater than 97\%.
\subsection{Simple Multiplication Task}
\label{appendix_simple_mul}
% \subsubsection{Alternative format}
% %\YZ{explain.} \PX{Updated.}
% \input{Figures/simple_mul_plain}
% In the main paper, we considered training transformers to do 40-digit times 1-digit multiplication in the reverse output format. We also trained models to do the same task but in the plain output format order. As shown in Figure~\ref{fig:simple-mul-plain}, we observed the same learning dynamics as in the reverse output format case--models simultaneously learn the reverse order and normal order.
\subsubsection{Output digit permutations}
\begin{figure}[t]
    \centering
    \begin{subfigure}{0.49\textwidth} % Adjust width as needed
        \centering
        \includegraphics[width=\linewidth]{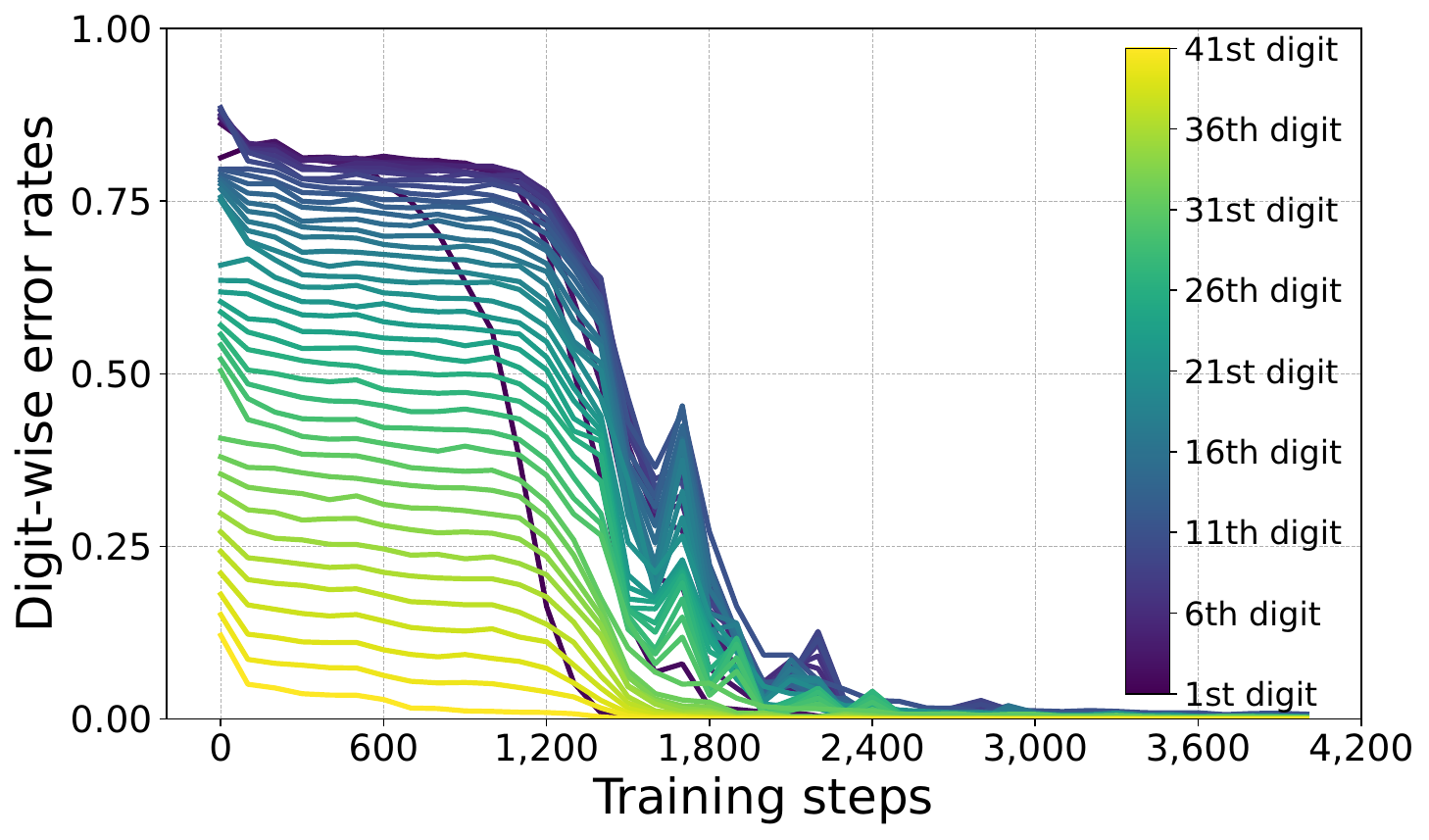} 
        \caption{Permutation 1}
    \end{subfigure}%
    \hfill %
    \begin{subfigure}{0.49\textwidth} % Adjust width as needed
        \centering
        \includegraphics[width=\linewidth]{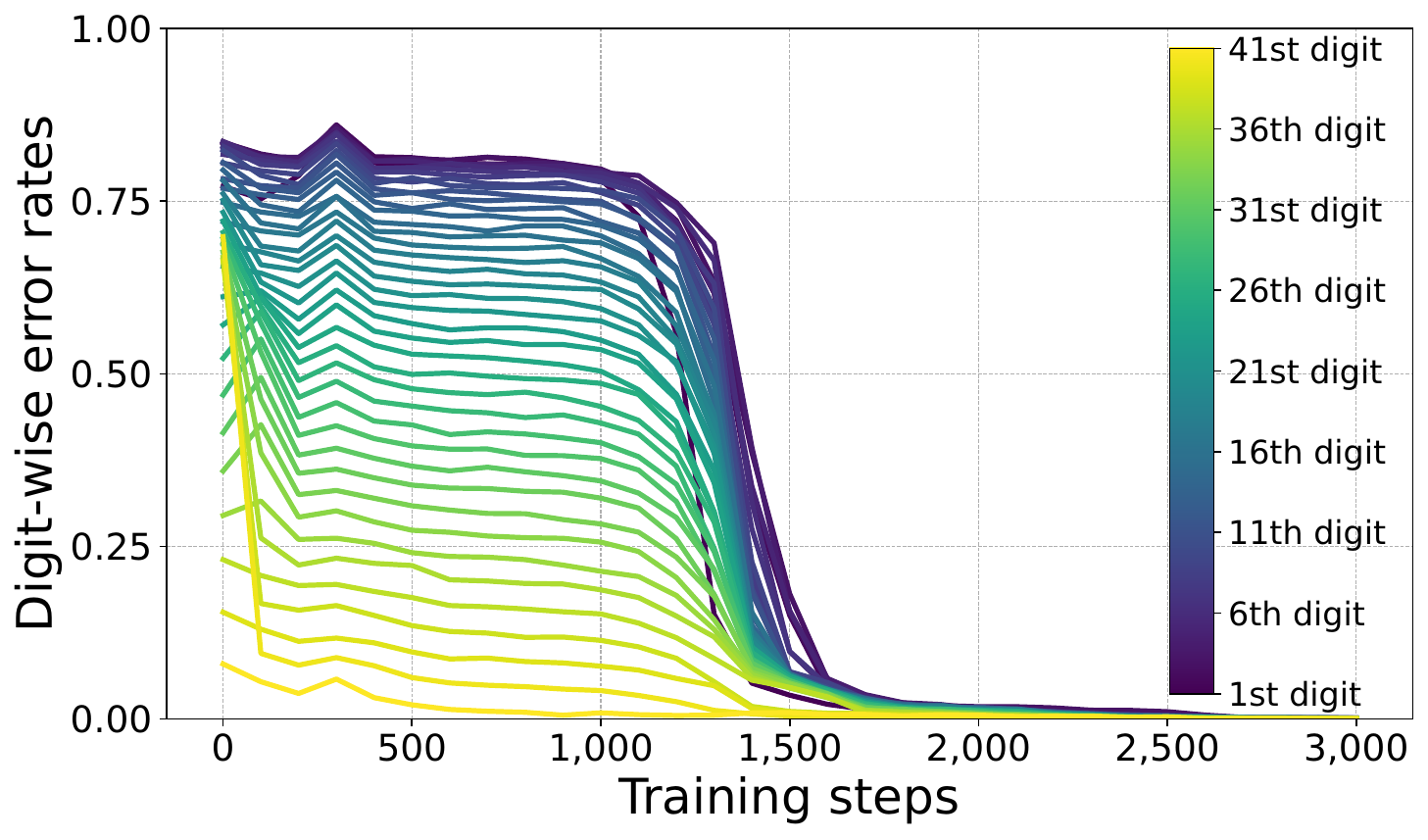} 
        \caption{Permutation 2}
    \end{subfigure}
    \caption{\textbf{Regardless of permutations of the output order, transformers learn digits in two-way order for simple multiplication}. We train transformers on simple multiplication of the format $a \times b = c$ with the digits of $c$ written in different permutation orders. Following Section~\ref{addition_perm}, we consider two permutation orders rather than plain and reverse format. In the left subfigure, we use the permutation "11 10 9 8 7 6 5 4 3 2 1 21 20 19 18 17 16 15 14 13 12 31 30 29 28 27 26 25 24 23 22 41 40 39 38 37 36 35 34 33 32" (block-reversal with rotation). In the right subfigure, we use the permutation "1 3 5 7 9 11 13 15 17 19 21 23 25 27 29 31 33 35 37 39 41 40 38 36 34 32 30 28 26 24 22 20 18 16 14 12 10 8 6 4 2" (odd positions then even positions, reversed evens).
    As in plain and reverse output format, we find two opposite sequential orders are learned simultaneously. (i) Reverse order: starting from the 41st (highest) digit to the 3rd digit; (ii) Normal order: starting from the 1st (lowest) digit. 
    } 
    \label{fig:simple_mul_perm}
\end{figure}

In the main paper, we presented the 40-digit times 1-digit training dynamics results using plain and reverse output format. Following Section~\ref{addition_perm}, we also consider two other permutations of output digits order for our simple multiplication task. One is "block-reversal with rotation", with the permutation being "11 10 9 8 7 6 5 4 3 2 1 21 20 19 18 17 16 15 14 13 12 31 30 29 28 27 26 25 24 23 22 41 40 39 38 37 36 35 34 33 32". Another is "odd positions then even positions, reversed evens", with the permutation being "1 3 5 7 9 11 13 15 17 19 21 23 25 27 29 31 33 35 37 39 41 40 38 36 34 32 30 28 26 24 22 20 18 16 14 12 10 8 6 4 2". As shown in Figure~\ref{fig:simple_mul_perm}, regardless of the permutation of the output digits, models learn digits in two opposite sequential orders, reverse order and normal order, simultaneously.

\subsubsection{Alternative lengths}
\begin{figure}[t]
    \centering
    \begin{subfigure}{0.49\textwidth} % Adjust width as needed
        \centering
        \includegraphics[width=\linewidth]{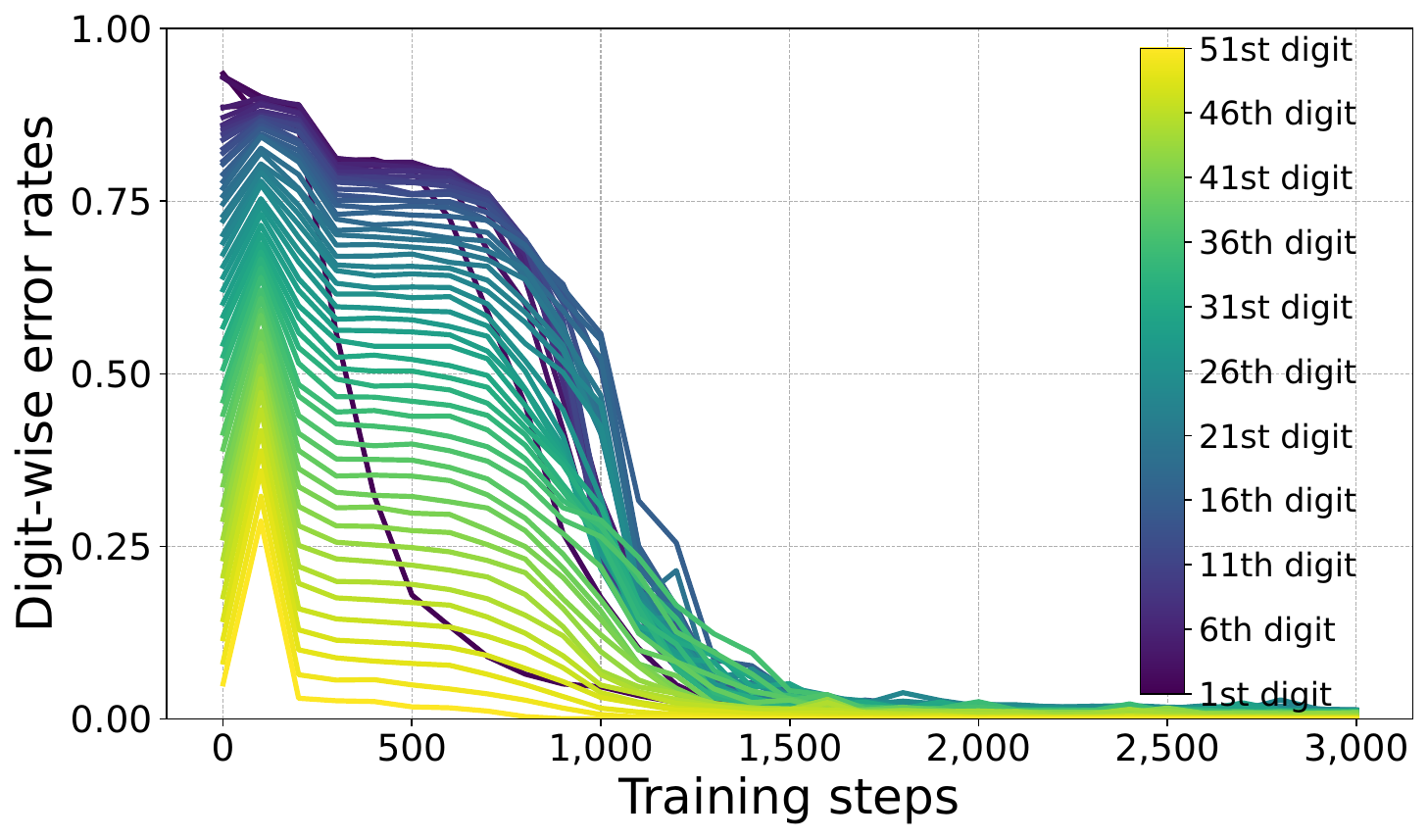} 
        \caption{Plain output format}
    \end{subfigure}%
    \hfill %
    \begin{subfigure}{0.49\textwidth} % Adjust width as needed
        \centering
        \includegraphics[width=\linewidth]{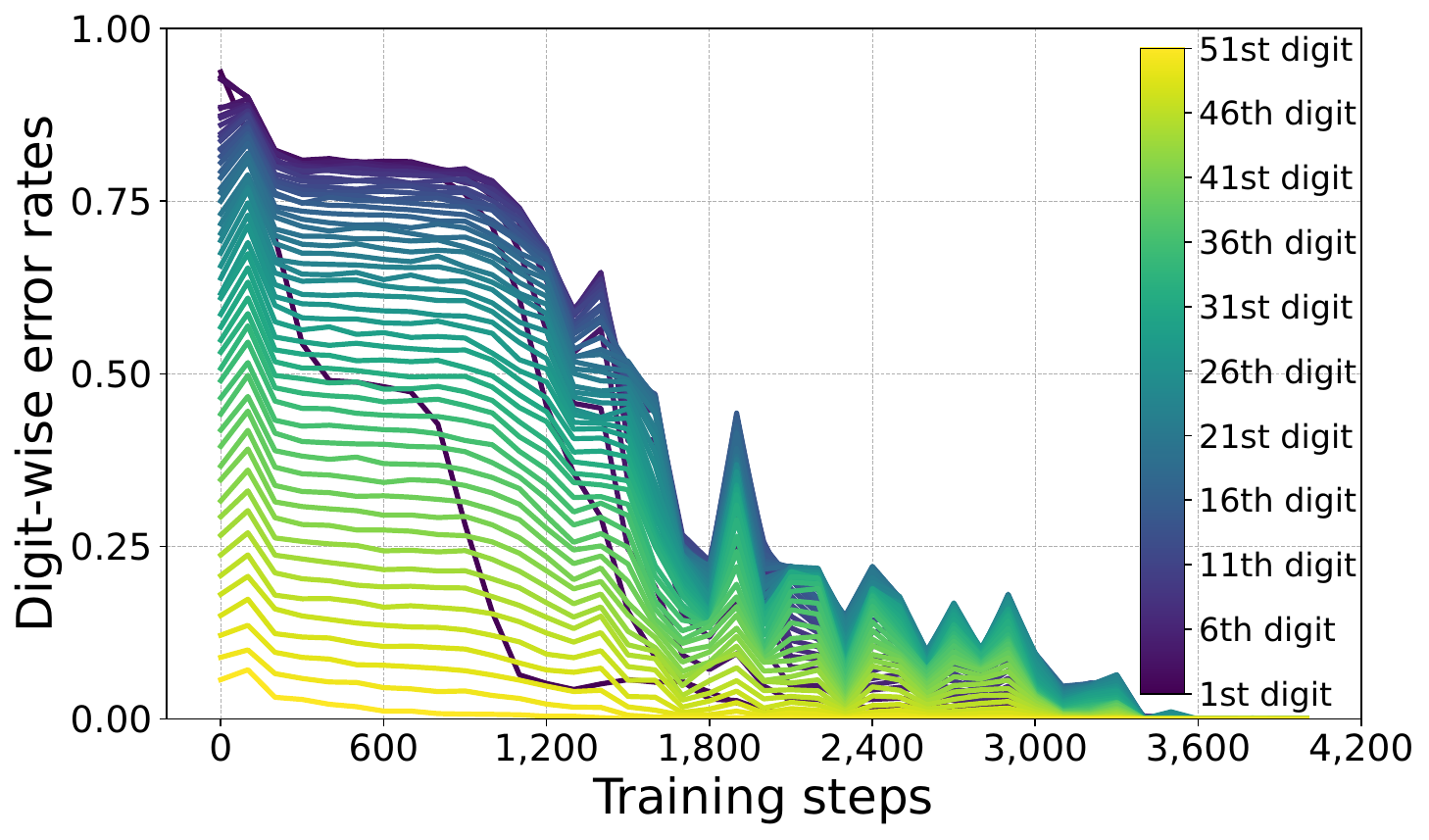} 
        \caption{Reverse output format}
    \end{subfigure}
    \caption{\textbf{Simple multiplication: 50-digit times 1-digit results}. Training transformers with 50-digit times 1-digit multiplication, we also find that reverse order and normal order are learned simultaneously.
    } 
    \label{fig:simple_mul_50_digit}
\end{figure}

We also train transformers with 50-digit times 1-digit multiplication. As in the 40-digit times 1-digit case, transformers learn the reverse order and normal order simultaneously, as shown in Figure~\ref{fig:simple_mul_50_digit}.

\subsubsection{Mutual information metric results}
\label{MI_mul}
\begin{figure}[t]
    \centering
    \includegraphics[width=\linewidth]{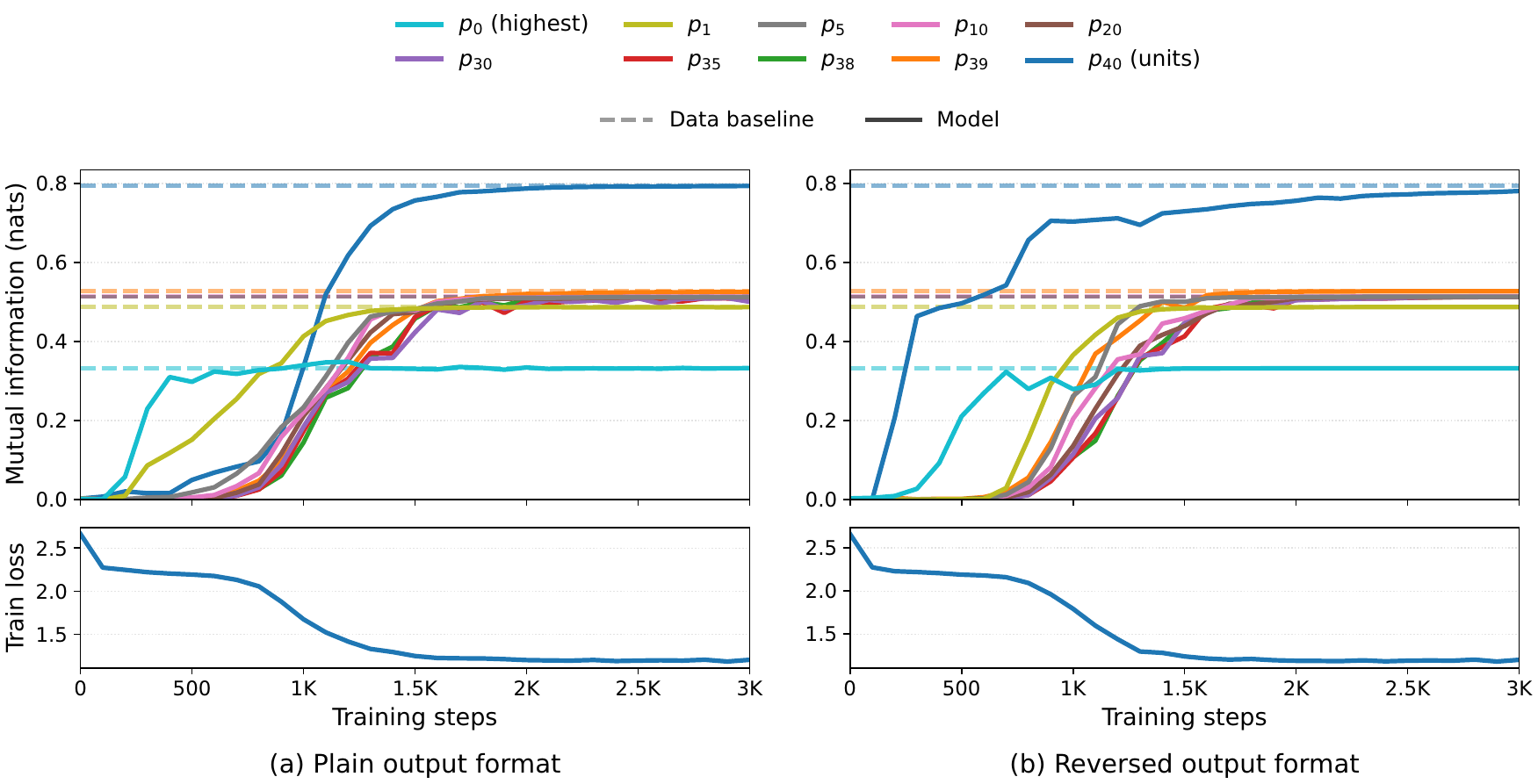}
    \caption{
        \textbf{Mutual information tracks bidirectional learning in multiplication.}
        We compare plain (left) and reversed (right) output formats over 3K training updates.
        Solid curves show MI from teacher-forced prediction distributions; dashed lines show data baselines.
        For each probe, both estimates use the same eligible examples from an independent 1M-example dataset.
        Indices run from the highest possible output place ($p_0$) to units ($p_{40}$).
        The $p_0$ probe is unconditional; other probes condition on their column's carry-out.
        Training loss is shown below.
    }
    \label{fig:MI_mul_conditioned_on_carries}
\end{figure}

In the main paper, we showed transformer's learning dynamics in addition align with the available correlational signals, measured by conditional mutual information. We similarly track the MI conditioned on carries for multiplication task. 
As shown in Figure~\ref{fig:MI_mul_conditioned_on_carries},
different output positions approach their data baselines at
different stages of training. Dependencies near both ends of
the product emerge early, consistent with the bidirectional acquisition pattern observed in the digit-wise errors.
The main decreases in training loss are matched by increases in model MI toward the corresponding data baselines.

\subsection{Comparison task.}
\label{appendix-comparison}
\subsubsection{Subskills definition}
Solving the 4-digit number comparison task requires the model to master several subskills, combine and apply them in a logical way. These subskills include comparing the most significant differing digit, deciding the two numbers equal if all their digits equal. A more detailed description of these subskills and their logical orders is presented in Table~\ref{tab:comparison_subskills}. 
\begin{table}[h]
\centering
\caption{Subskills required in 4-digit number comparison.}
\label{tab:comparison_subskills}
\begin{tabular}{@{}p{0.18\linewidth} p{0.72\linewidth}@{}}
\toprule
\textbf{Subskill} & \textbf{Description} \\ 
\midrule
Subskill 1 & Compare the 1st digit \\
Subskill 2 & Compare the 2nd digit (apply only if Subskill 1 fails to decide). \\
Subskill 3 & Compare the 3rd digit (apply only if Subskills 1-2 fail to decide) \\
Subskill 4 & Compare the 4th digit (apply only if Subskills 1-3 fail to decide) \\
Subskill 5 & 	If all digits are identical, output '='
 \\
\bottomrule
\end{tabular}
\end{table}

\subsubsection{Comparison Contrast Pair Dataset}
\label{contrast_dataset}
We generate four sets of contrast pairs to isolate the contribution of specific digit positions (see Table~\ref{tab:contrast_pairs}). The testing pairs $(a, b)$ in set $k \in \{1, \dots, 4\}$ differ \textit{strictly} at position $k$, while all other digits remain identical. To generate a pair, we first sample a number $a$ uniformly from $[1000, 9999]$. We then construct $b$ by sampling a digit $b_k \in \{0, \dots, 9\} \setminus \{a_k\}$ for the $k$-th position, setting $b_i = a_i$ for all $i \neq k$. Formally, the condition for set $k$ is:
\begin{equation*}
    a_k \neq b_k \quad \land \quad a_i = b_i, \; \forall i \neq k
\end{equation*}

\begin{table}[h]
\centering
\caption{Examples of generated pairs $(a, b)$ for Comparison Contrast Pair Dataset.  For each set $k$, the pairs differ \textit{strictly} at the $k$-th position (highlighted in bold), while all other digits remain identical.}
\label{tab:contrast_pairs}
\begin{tabular}{@{}c l l l l c@{}}
\toprule
\textbf{Set ($k$)} & \textbf{Active Position} & \textbf{Constraint} & \textbf{Example $a$} & \textbf{Example $b$} & \textbf{Label} \\ \midrule
1 & Thousands & $a_1 \neq b_1$; $a_{j \neq 1} = b_j$ & $\mathbf{9}\,1\,5\,7$ & $\mathbf{1}\,1\,5\,7$ & $>$ \\
2 & Hundreds & $a_2 \neq b_2$; $a_{j \neq 2} = b_j$ & $3\,\mathbf{5}\,5\,9$ & $3\,\mathbf{9}\,5\,9$ & $<$ \\
3 & Tens & $a_3 \neq b_3$; $a_{j \neq 3} = b_j$ & $7\,3\,\mathbf{8}\,9$ & $7\,3\,\mathbf{3}\,9$ & $>$ \\
4 & Units & $a_4 \neq b_4$; $a_{j \neq 4} = b_j$ & $5\,3\,5\,\mathbf{0}$ & $5\,3\,5\,\mathbf{4}$ & $<$ \\ \bottomrule
\end{tabular}
\end{table}

\subsubsection{Additional subskill test}
%\PX{Updated.}
\begin{figure}[t]
    \centering

        \centering
        \includegraphics[width=0.6\linewidth]{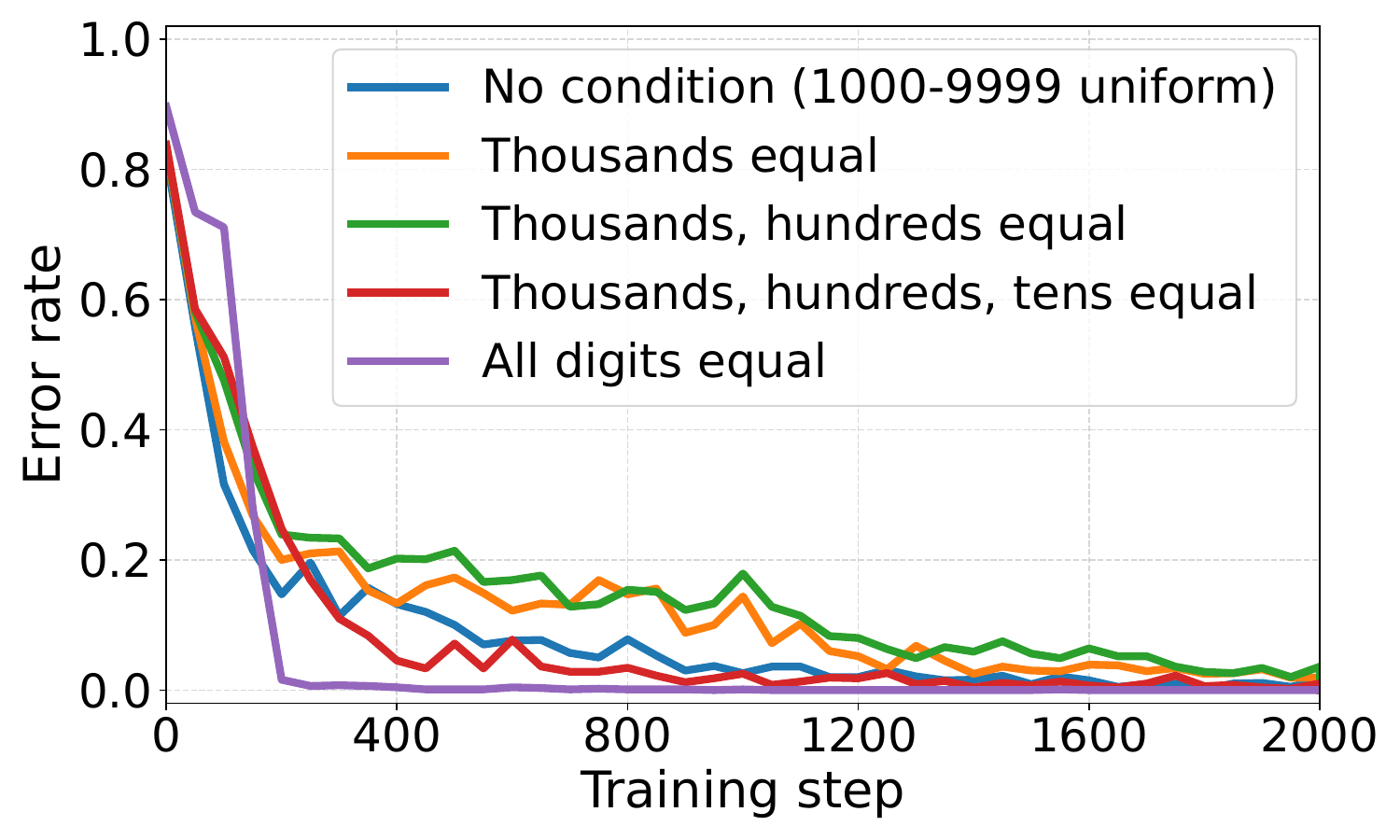} 

    % \begin{subfigure}{0.5\textwidth} % Adjust width as needed
    %     \centering
    %     \includegraphics[width=0.9\linewidth]{Figures/Images/staircase_digitwise_test.pdf} 

    % \end{subfigure}%
    \caption{\textbf{Transformers learn to compare numbers of different closeness levels concurrently.}. When we use test data of the same distribution as training data (NCID Group 0 to 4), we find transformers gain accuracy improvement in all NCID groups simultaneously, which indicates transformers learn to compare numbers of different closeness levels (compare 6183 \& 6189 vs compare 7291 \& 4813) at the same time.}
    \label{fig:comparison_staircase_test}
\end{figure}

Besides the Contrast Pair Dataset we designed for assessing subskill acquisition, we also use test data of the same distribution as training data (NCID Group 0 to 4). Different NCID groups basically contain numbers of different closeness levels (e.g. models compare 6183 \& 6189 in NCID Group 3 while comparing  7291 \& 4813 in NCID Group 0). As shown in Figure~\ref{fig:comparison_staircase_test}, we find that subskills of comparing numbers of different closeness levels are learned simultaneously.

\subsubsection{Different random seeds for contrast pairs}
\label{app:comparison-random-seeds}

\begin{figure}[t]
    \centering
    \begin{subfigure}{0.49\textwidth}
        \centering
        \includegraphics[width=\linewidth]{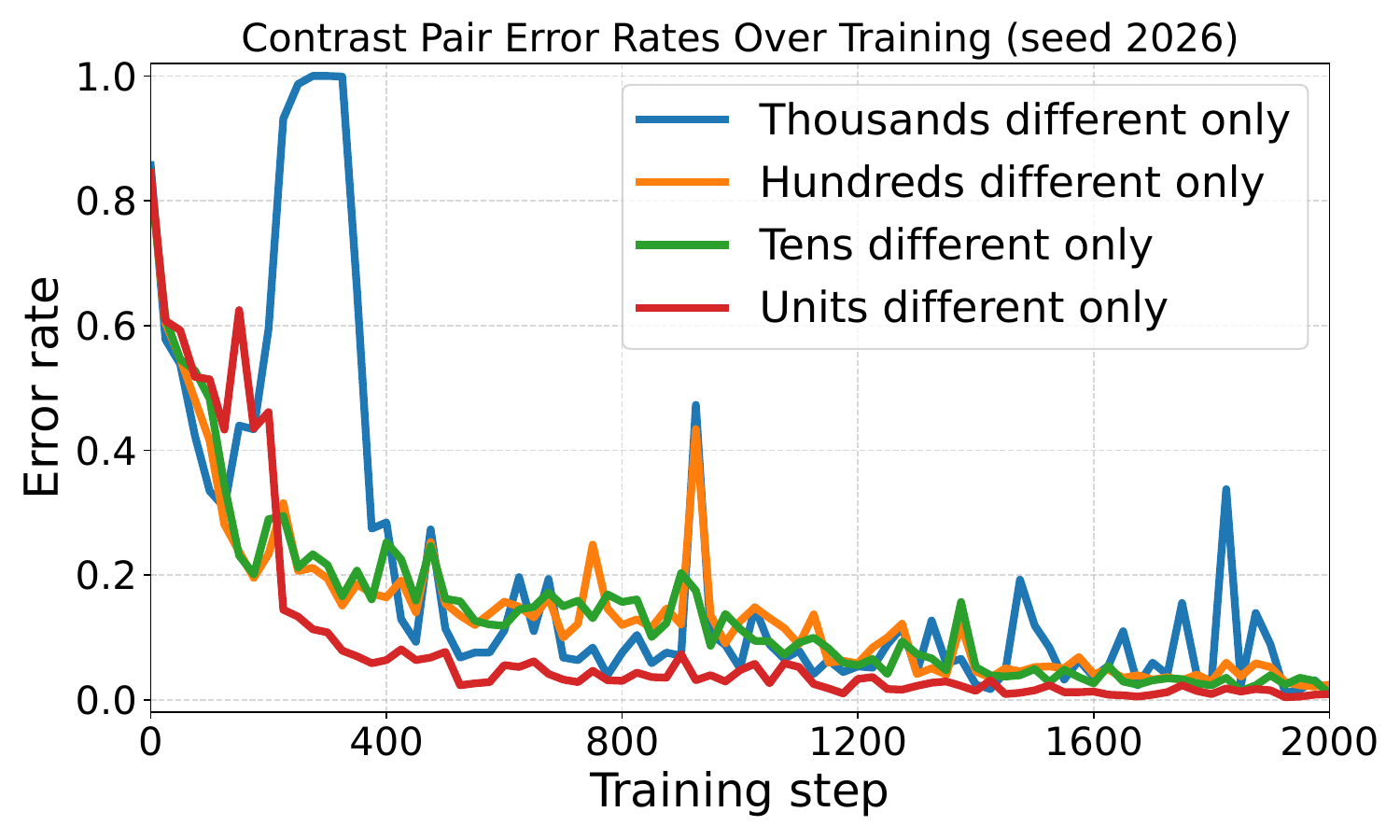}
        \caption{Using random seed 2026}
        \label{fig:comparison_contrast_pair_seed_2026}
    \end{subfigure}
    \hfill
    \begin{subfigure}{0.49\textwidth}
        \centering
        \includegraphics[width=\linewidth]{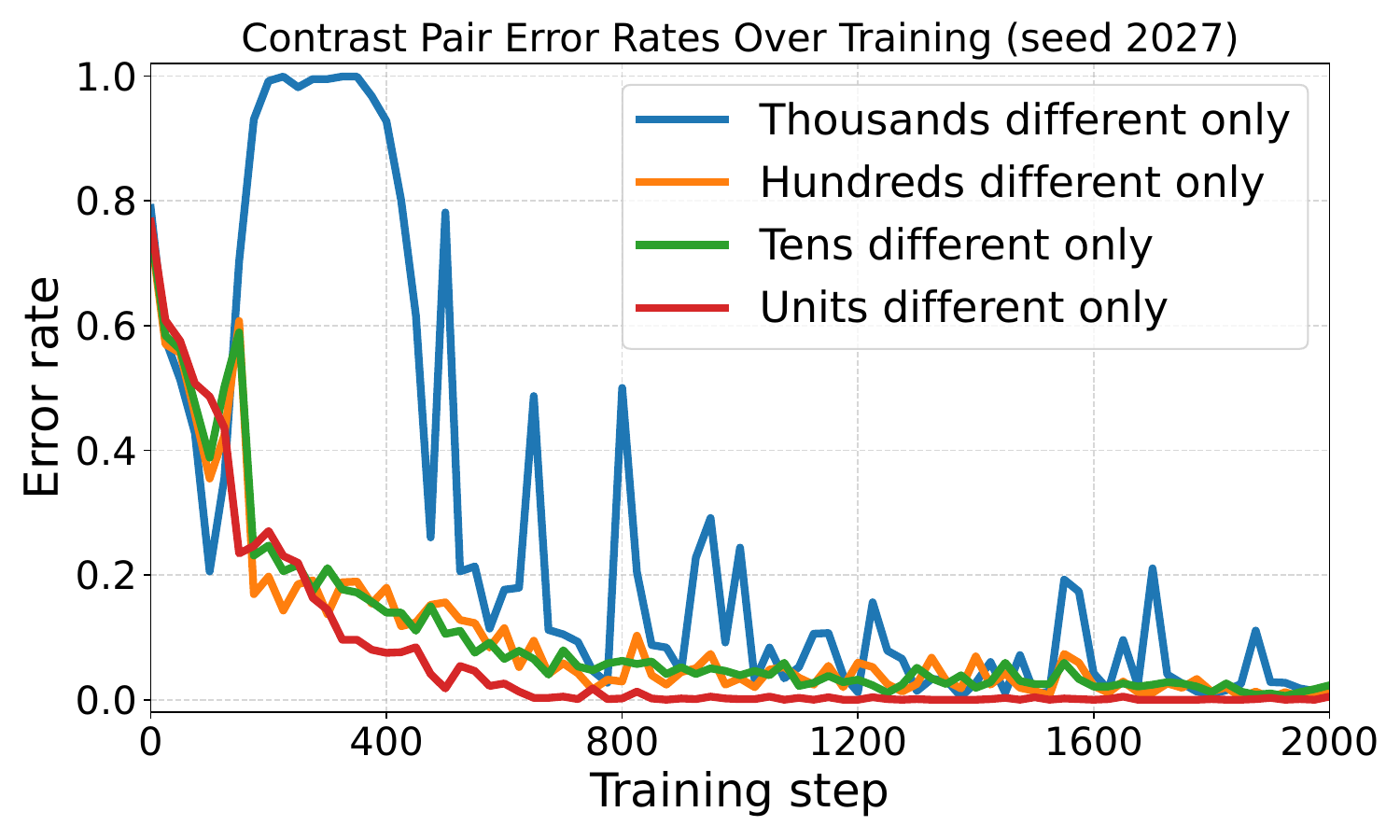}
        \caption{Using random seed 2027}
        \label{fig:comparison_contrast_pair_seed_2027}
    \end{subfigure}

    \caption{
    \textbf{Additional contrast-pair results using different random seeds.}
    We reran the comparison contrast-pair evaluation with different random seeds. 
    Across seeds, the same qualitative pattern persists: error rates on contrast pairs exhibit temporary non-monotonicity, indicating that partially learned comparison subskills can interfere with one another before stabilizing.
    }
    \label{fig:comparison_contrast_pair_random_seeds}
\end{figure}

We further test whether the contrast-pair interference pattern is sensitive to training randomness by rerunning the same comparison experiment with different random seeds. As shown in Figure~\ref{fig:comparison_contrast_pair_random_seeds}, the qualitative behavior is consistent across seeds. The contrast-pair evaluations again show temporary non-monotonicity, supporting that independently useful comparison subskills can transiently compete before the model reaches a stable solution.
\subsection{Sorting task.}

\subsubsection{Uniform sampling}\label{sec:append-sorting-uniform}

%\YZ{Present a few results and discussions here. Why is uniform sampling not an interesting sampling strategy? Provide some rationales.} \PX{Updated.}
\begin{figure}[t]
    \centering

    %{0.8\textwidth} % Adjust width as needed
        \centering
        \includegraphics[width=0.6\linewidth]{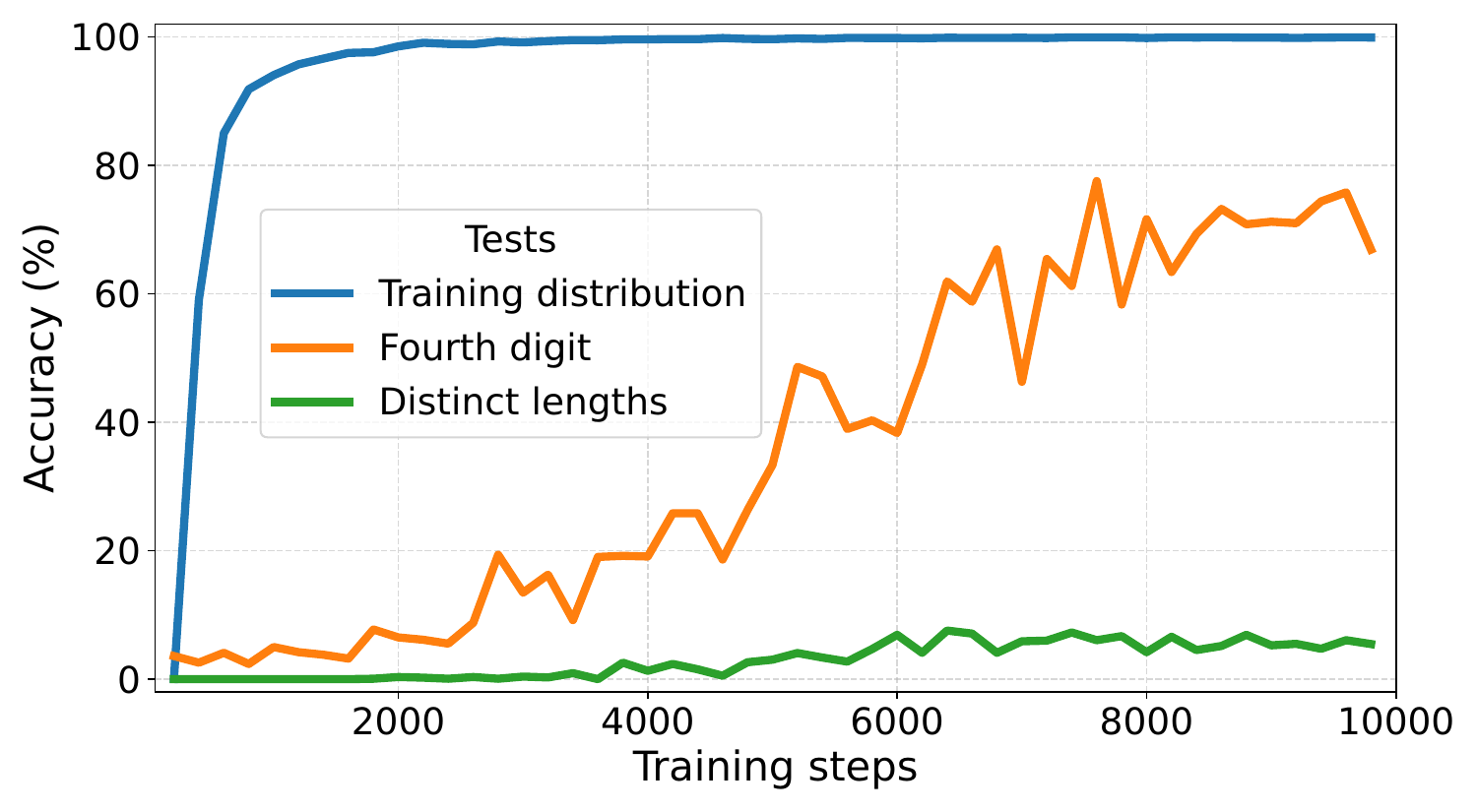}

    \caption{\textbf{Under uniform sampling, transformers fail to effectively learn subskills essential in the sorting task}. We train NanoGPTs on 4 operand sorting task. Under uniform sampling, each input number is uniformly sampled from 0 to 9999. Despite achieving a perfect accuracy on the training distribution early in training, the model struggles to compare numbers of distinct lengths and to acquire the 4th digit comparing subskill.}
    \label{fig:sorting_uniform}
\end{figure}

In the sorting task, all training examples take the following format $a,b,c,d \to \mathrm{sorted}(a,b,c,d) $,
\[
\underbrace{9312}_{a}, \underbrace{4661}_{b}, \underbrace{405}_{c}, \underbrace{6252}_{d} \to \underbrace{405}_{c}, \underbrace{4661}_{b}, \underbrace{6252}_{d}, \underbrace{9312}_{a}. 
\]
We experimented with different sampling strategies for generating the training dataset. One possible approach is uniform sampling, where $a,b,c,d$ are uniformly sampled from 0 to 9999. We find under uniform sampling, despite achieving perfect accuracy on the training data distribution, the model performs poorly on sorting the following two types of numbers.

\begin{enumerate}
    \item \textbf{Close numbers: } As described in Section~\ref{sec:doubly_bal_sorting_details}, we adopt a suite of specially designed test examples to assess models' subskill learning. Particularly, to measure the 4th digit comparing accuracy we use 4-digit numbers as input, where we controlled the first three digits to be equal. As shown in Figure~\ref{fig:sorting_uniform}, we find the model's 4th digit accuracy stagnates at around 70\%, failing to achieve satisfactory subskill accuracy.
    \item \textbf{Numbers of distinct lengths: } Considering that under uniform sampling the input numbers range from 0 to 9999, to evaluate the model's length comparing ability, we design the "distinct lengths" test as follows: The 4 input numbers have distinct lengths, being 1-digit, 2-digit, 3-digit, 4-digit respectively. A test example is $15,4,5303,920$. However, we find the model fails to acquire this length comparison subskill throughout the training with an accuracy below 10\%.
\end{enumerate}

Based on these results, we conclude that uniform sampling is not suitable for models to learn different subskills essential in solving the general sorting task. To expose model more often to training examples of different closeness levels and of varying length we introduced the "doubly balanced" training dataset described in Section~\ref{sec:append-data-gen-sorting}.

\subsubsection{Further details of Section~\ref{sec:sorting}}  \label{sec:doubly_bal_sorting_details}

Solving the 4 operand 3-digit \& 4-digit sorting task requires the model to master several subskills, combine and apply them in a logical way. These subskills include identifying, then comparing the length of each input number and digit-wise comparison. A more detailed description of these subskills and their logical orders is presented in Table~\ref{tab:sorting_subskills}. We note that in sorting these subskills have a hierarchical precedence, meaning that the model should apply one subskill only if all previous subskills fail to resolve the ordering.

% \begin{enumerate}
%     \item Identifying and comparing the length of each input number
%     \item Compare 1st digit
%     \item Compare 2nd digit
%     \item Compare 3rd digit
%     \item Compare 4th digit
% \end{enumerate}

\begin{table}[h]
\centering
\caption{Subskills required in 3-digit \& 4-digit number sorting.}
\label{tab:sorting_subskills}
\begin{tabular}{@{}p{0.18\linewidth} p{0.72\linewidth}@{}}
\toprule
\textbf{Subskill} & \textbf{Description} \\ 
\midrule
Subskill 1 & Identify and compare the length of each number. \\
Subskill 2 & Compare the 1st digit (apply only if Subskill 1 fails to decide). \\
Subskill 3 & Compare the 2nd digit (apply only if Subskills 1--2 fail to decide). \\
Subskill 4 & Compare the 3rd digit (apply only if Subskills 1--3 fail to decide). \\
Subskill 5 & Compare the 4th digit (apply only if the number is 4-digit and Subskills 1--4 fail to decide). \\
\bottomrule
\end{tabular}
\end{table}

\paragraph{Crude length learned first, individual digits learned in parallel} To determine the learning order of these subskills, we use a set of specially targeted tests to assess model's acquisition of each subskill. We summarize each test file and our corresponding evaluation criterion for model's prediction to be counted as correct in Table~\ref{tab:subskill_testfiles}.

\begin{table}[h]
\centering
\caption{Test files we use to assess the model's acquisition of different subskills and the corresponding criterion for a prediction to be counted as correct. $l(a,b,c,d)$ denotes the length of the four numbers.}
\label{tab:subskill_testfiles}
\begin{tabular}{@{}p{5.0cm} p{3.0cm} p{5.2cm}@{}}
\toprule
\textbf{Test file description} & \textbf{Targeted subskill} & \textbf{Criterion} \\ 
\midrule
Doubly balanced test & Length comparison & All output numbers are of correct lengths \\

NCID = 0, $l(a,b,c,d)=(4,4,4,4)$ & 1st digit comparison & Output the 1st digits of the 4 input numbers in correct order \\

NCID = 1, $l(a,b,c,d)=(4,4,4,4)$ & 2nd digit comparison & Output the 2nd digits of the 4 input numbers in correct order \\

NCID = 2, $l(a,b,c,d)=(4,4,4,4)$ & 3rd digit comparison & Output the 3rd digits of the 4 input numbers in correct order \\

NCID = 3, $l(a,b,c,d)=(4,4,4,4)$ & 4th digit comparison & Output the 4th digits of the 4 input numbers in correct order \\
\bottomrule
\end{tabular}
\end{table}

As shown in Figure~\ref{fig:sorting_subskill_order}, the model learns the "crude" length early in training before any digit can be predicted correctly. In other words, at this stage, the model can output the position of the comma, the separating character, correctly. Later, the model learns each individual digit mostly in parallel, with the fourth digit lagging behind.

\paragraph{Learning different subskills creates competitions.} Solving sorting task requires learning distinct subskills, one sometimes even conflicting with another. We hypothesize that learning one subskill poses challenges to learning other subskills. As a control group to the previously adopted "doubly balanced" training dataset, we design the  "closeness-balanced only" training dataset, where all numbers are 4-digit. For each training example, we randomly draw the NCID group from $k \in\{0,1,2\}$, each with probability $1/3$. See Table~\ref{tab:sorting_closeness_bal_only} for details.

\begin{table}[h]
\centering
\caption{"Closeness-balanced only" training dataset for sorting task. Except that all numbers are 4-digit, all other constraints are the same as in the "doubly balanced" dataset. $l(a,b,c,d)$ denotes the length of the four numbers.}
\label{tab:sorting_closeness_bal_only}
\begin{tabular}{@{}c l l l@{}}
\toprule
\textbf{NCID ($k$)}  & \textbf{Constraint} & \textbf{Example input} \\ \midrule
0 & $l(a,b,c,d) = (4,4,4,4) $ & $3621,7916,2040,7304$ \\
1 & $l(a,b,c,d) = (4,4,4,4); a_1=b_1=c_1=d_1 $ & $7370,7224,7141,7521$ \\
2 & $l(a,b,c,d) = (4,4,4,4); a_{1:2}=b_{1:2}=c_{1:2}=d_{1:2} $ & $8168,8133,8195,8174$ \\
 \bottomrule
\end{tabular}
\end{table}

% \begin{enumerate}
%     \item \textbf{NCID = 0: No closeness enforcement.} Each number is uniformly drawn from 1000 to 9999.
%     \item \textbf{NCID = 1: Identical 1st digit}  The single common digit is uniformly drawn from 1 to 9. The rest 3 digits for each 4-digit number is uniformly drawn from 000 to 999. (e.g. 5810,5099,5812,5285)
%     \item \textbf{NCID = 2: Identical 1st \& 2nd digit} The two common digits are uniformly drawn from 10 to 99. The two rest digits for each 4-digit number is uniformly drawn from 00 to 99. (e.g. 2552,2506,2563,2586)
% \end{enumerate}

\begin{figure}[t]
    \centering
    \begin{subfigure}{0.6\textwidth} % Adjust width as needed
        \centering
        \includegraphics[width=\linewidth]{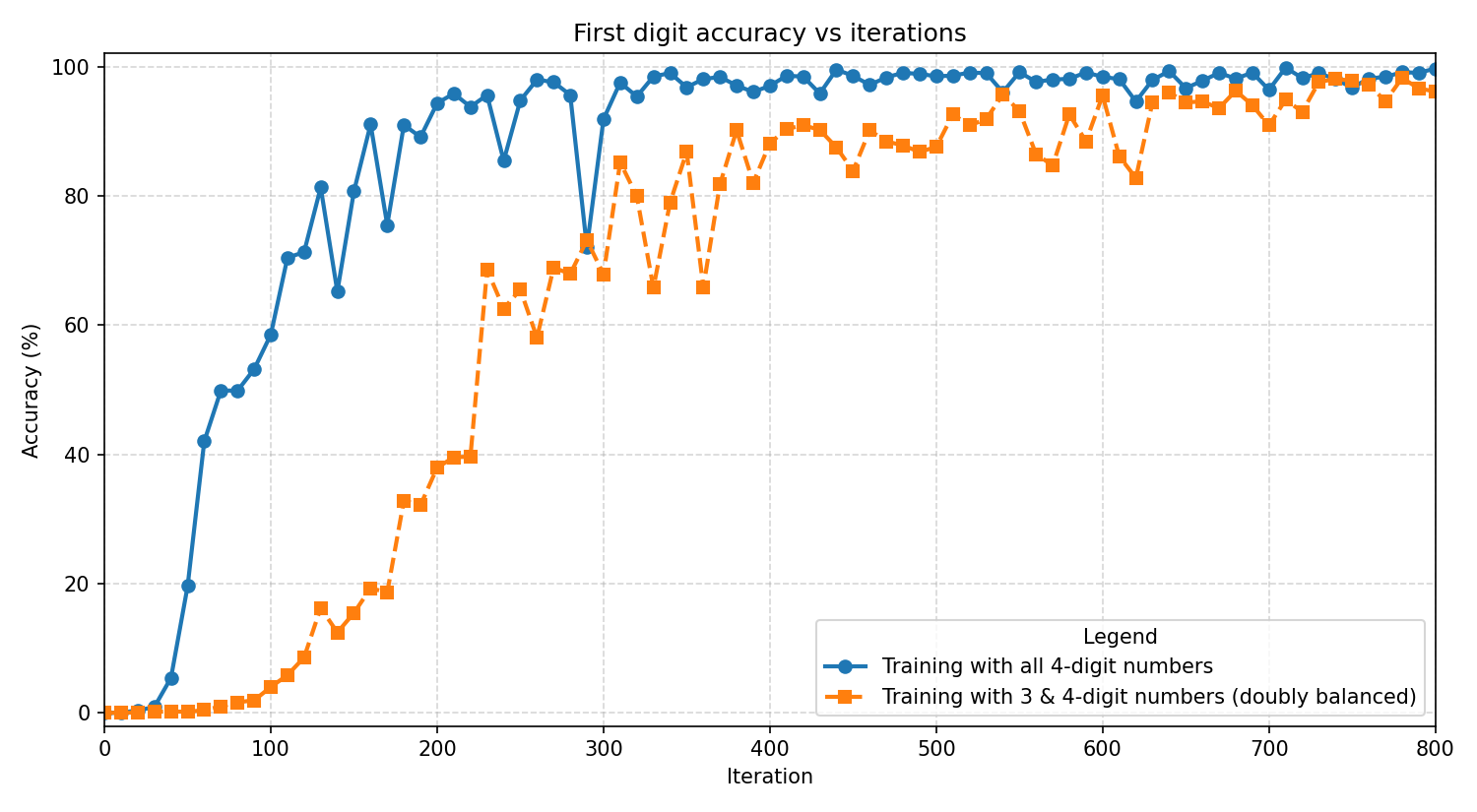} 
    \end{subfigure}%
    \caption{\textbf{Learning different subskills in sorting creates competition.} We train NanoGPTs using either closeness balanced only dataset (\textbf{Blue line}) or doubly balanced dataset (\textbf{Orange line}). Unlike in the doubly balanced dataset where the model needs to learn both comparing length and comparing individual digits, in the closeness balanced only dataset, the model does not need to learn comparing the length, since all numbers are 4-digit. As a result, the model takes much shorter time to achieve 80\%+ accuracy in first digit comparison.}
    \label{fig:competition_length_first_digit}
\end{figure}
Since all input numbers are 4-digit, the model does not need to learn the length comparison subskill. As shown in Figure~\ref{fig:competition_length_first_digit}, the model learns 1st digit comparison (subskill 2) significantly earlier when it does not need to learn length comparison (subskill 1). Specifically, when trained with "closeness-balanced only" dataset, it only takes 130 iterations to achieve a $80\%+$ accuracy. As a comparison, it takes the model 310 iterations to achieve similar performance when trained in the "doubly balanced" dataset.

\paragraph{Models sometimes apply different subskills in undesired ways.} Sorting not only requires models to master different subskills individually, but also to apply them in a logical way. We find during training models are likely to combine different subskills in undesired ways. We design the following three "skewed" test datasets, in all cases one of the input number is 4-digit, the other three numbers are 3-digit:

\begin{enumerate}
    \item \textbf{First digit skewed} The first digit of the 4-digit number is drawn from 1 to 4, while the first digit of the 3-digit numbers is drawn from 5 to 9. The other digits are drawn from 0 to 9. A skewed error in this case is defined as the model outputting the first digit of the 4-digit number as the start of its predicted smallest number. A test example and its corresponding skewed error would be \(2774,524,996,875 = 2\ldots\)
    \item \textbf{Second digit skewed} The 4 input numbers share the same first digit, which is drawn from 1 to 9. The second digit of the 4-digit number is drawn from 0 to 4, while the second digit of the 3-digit numbers is drawn from 5 to 9. The other digits are drawn from 0 to 9. A skewed error in this case is defined as the model outputting the first two digits of the 4-digit number as the start of its predicted smallest number. A test example and its corresponding skewed error would be \(197,1234,165,183 = 12\ldots\)
    \item \textbf{Third digit skewed} The 4 input numbers share the same 1st and 2nd digit, which are drawn from 10 to 99. The third digit of the 4-digit number is drawn from 0 to 4, while the third digit of the 3-digit numbers is drawn from 5 to 9. A skewed error in this case is defined as the model outputting the first three digits of the 4-digit number as the start of its predicted smallest number. A test example and its corresponding skewed error would be \(787,789,7815,786 = 781\ldots\)
\end{enumerate}
\begin{figure}[t]
    \centering
    \begin{subfigure}[b]{0.33\textwidth}
        \centering
        \includegraphics[width=\linewidth]{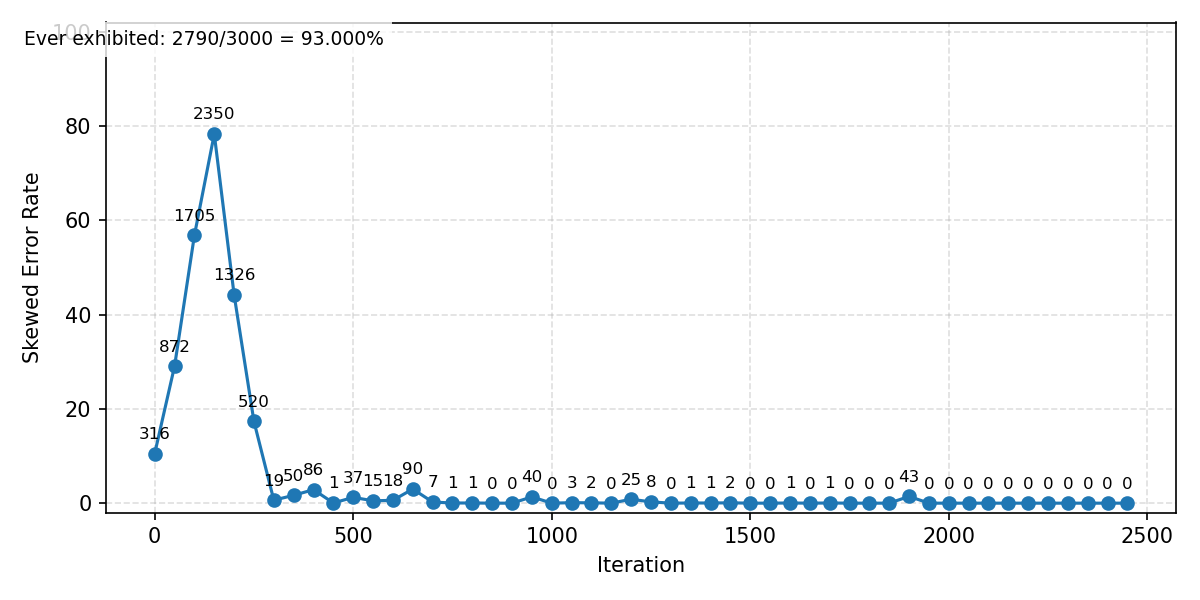}
        \caption{First digit skewed test}
        \label{fig:skewed_error_a}
    \end{subfigure}\hfill
    \begin{subfigure}[b]{0.33\textwidth}
        \centering
        \includegraphics[width=\linewidth]{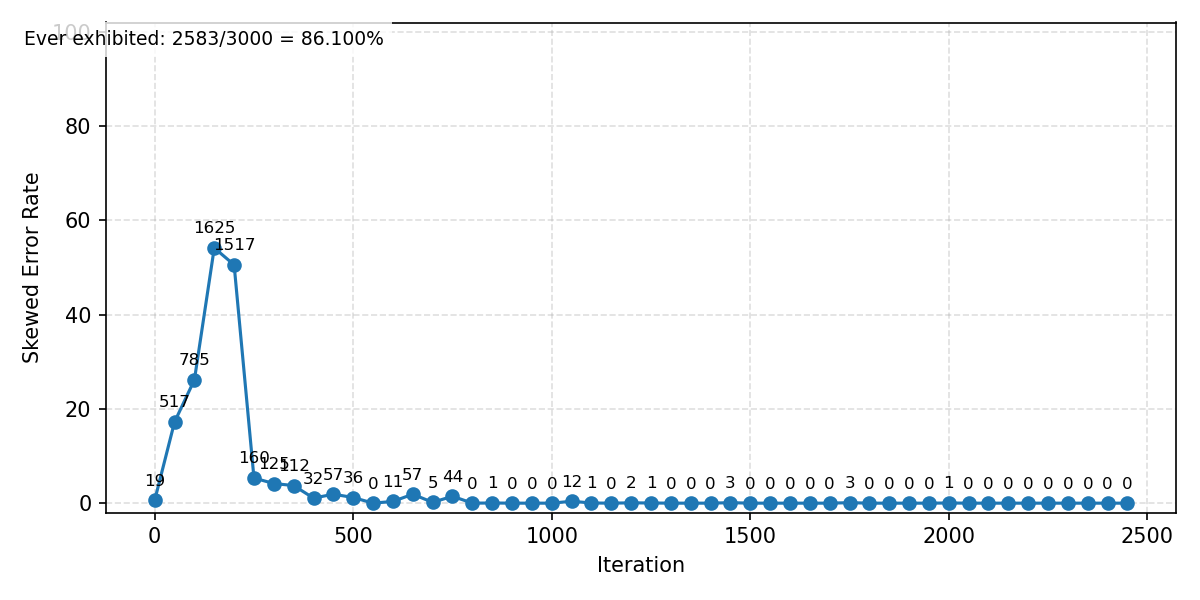}
        \caption{Second digit skewed test}
        \label{fig:skewed_error_b}
    \end{subfigure}\hfill
    \begin{subfigure}[b]{0.33\textwidth}
        \centering
        \includegraphics[width=\linewidth]{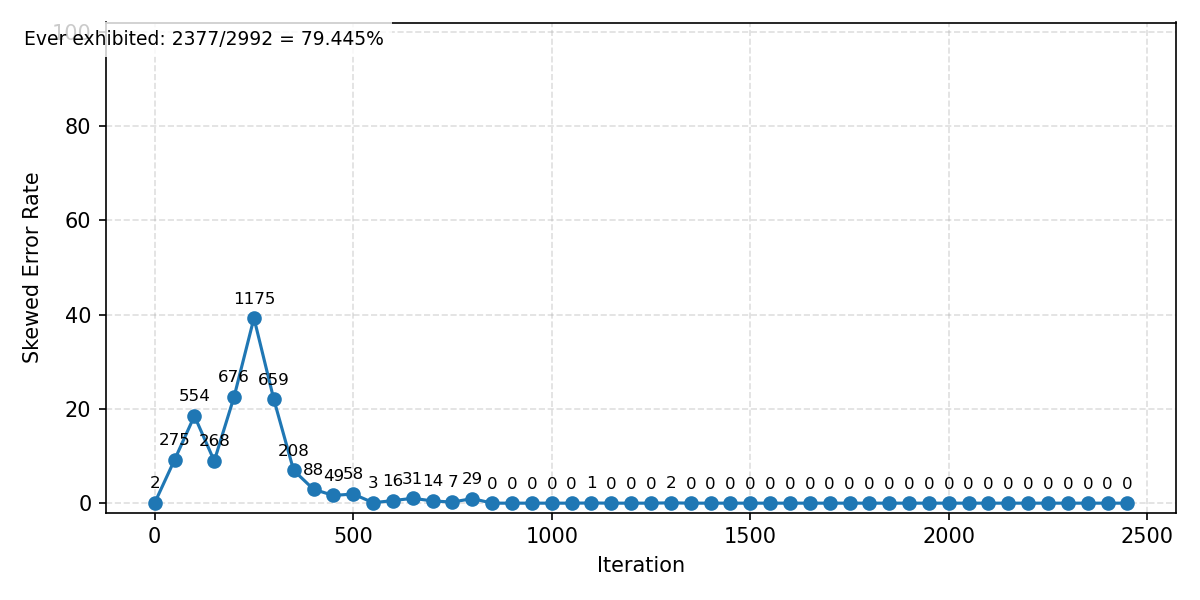}
        \caption{Third digit skewed test}
        \label{fig:skewed_error_c}
    \end{subfigure}

    \caption{\textbf{Transformers incorrectly combine learned sorting subskills}. As in Figure~\ref{fig:sorting_subskill_order}, we train the model using doubly balanced dataset, which requires the model to master both length comparison and digit-wise comparison. We test the model on three datasets, "First digit skewed", "Second digit skewed", and "Third digit skewed". \textbf{(a)} On "First digit skewed" dataset, the model makes errors of the kind \(2774,524,996,875 = 2\ldots\). At iter 150, the error rate peaks at 78\%.  \textbf{(b)} On "Second digit skewed" dataset, the model makes errors of the kind \(197,1234,165,183 = 12\ldots\). At iter 150, the error rate peaks at 54\%. \textbf{(c)} On "Third digit skewed" dataset, the model makes errors of the kind \(787,789,7815,786 = 781\ldots\). At iter 250, the error rate peaks at 39\%.}
    \label{fig:skewed_error}
\end{figure}
As shown in Figure~\ref{fig:skewed_error}, the model indeed makes the three kinds of skewed error, with the first kind peaks at even 79\% at iteration 150. All three error rates increases at first as the model learns length comparison and digit-wise comparison but illogically combining them, then drops as the model figures out how to correctly apply these subskills.

\paragraph{Additional Mixing Error Analysis}
As shown in Table~\ref{tab:mixing}, we identified two types of counterintuitive errors models usually make while learning the sorting task. In the main paper, we showed that conflicting digit pairs will exacerbate the mixing error occurrences. Here we investigate the role that same digit pairs play in producing the mixing error effects.

\subsubsection*{Same-digit tests (SD)}
Table~\ref{tab:same-digit-tests} summarizes the same-digit tests.
These tests vary whether the first and third digit positions are equal across $b$ and $c$, while keeping the $(2,4)$ conflict present.
The goal is to see whether equality at particular positions encourages the model to copy those digits across numbers, producing swaps or repeats.

\begin{table}[t]
\caption{\textbf{Same-digit tests for mixing errors in sorting.}
We vary whether the first and third digit positions are equal across the two middle numbers $b$ and $c$, while keeping the $(2,4)$ conflict present. Mixing errors occur only when $b_3=c_3$, and are amplified when $b_1=c_1$.}
\label{tab:same-digit-tests}
\centering
\small
\begin{tabular}{@{}p{0.9cm}@{\hspace{2pt}}p{3.0cm} c c c p{1.5cm} c c@{}}
\hline
Test id & Example input & $b_1=c_1$ & $b_3=c_3$ & $(2,4)$ conf. & Label & Mix Err. & SD \\
\hline
SD1 & \texttt{1000,6589,6682,9999} & yes & yes & yes & baseline & 6.63\% & 0.46\% \\
SD2 & \texttt{1000,6589,7682,9999} & no  & yes & yes & change $b_1$  & 3.61\% & 0.32\% \\
SD3 & \texttt{1000,6589,6672,9999} & yes & no  & yes & change $b_3$  & 0.0\% & 0.0\% \\
SD4 & \texttt{1000,6589,5672,9999} & no  & no  & yes & both changed   & 0.0\% & 0.0\% \\
\hline
\end{tabular}
\end{table}
% \begin{center}
% \begin{tabular}{@{}p{0.9cm}@{\hspace{2pt}}p{3.0cm} c c c p{1.5cm} c c@{}}
%  % 9 columns total: adjust to match header cells
% \hline
% Test id & Example input & $b_1=c_1$ & $b_3=c_3$ & $(2,4)$ conf. & Label & Mix Err. & SD \\
% \hline
% SD1 & \texttt{1000,6589,6682,9999} & yes & yes & yes & baseline & 6.63\% & 0.46\% \\
% SD2 & \texttt{1000,6589,7682,9999} & no  & yes & yes & change $b_1$  & 3.61\% & 0.32\% \\
% SD3 & \texttt{1000,6589,6672,9999} & yes & no  & yes & change $b_3$  & 0.0\% & 0.0\% \\
% SD4 & \texttt{1000,6589,5672,9999} & no  & no  & yes & both changed   & 0.0\% & 0.0\% \\
% \hline
% \end{tabular}
% \end{center}

\noindent\textbf{Interpretation:} Compare mixing frequencies across SD1--SD4. Equality at the third digit position ($b_3=c_3$) appears \emph{necessary} for mixing to occur. Equality at the first digit ($b_1=c_1$) amplifies mixing.
\\
\begin{figure}[t]
    \centering

\begin{subfigure}{0.5\textwidth} % Adjust width as needed
        \centering
        \includegraphics[width=\linewidth]{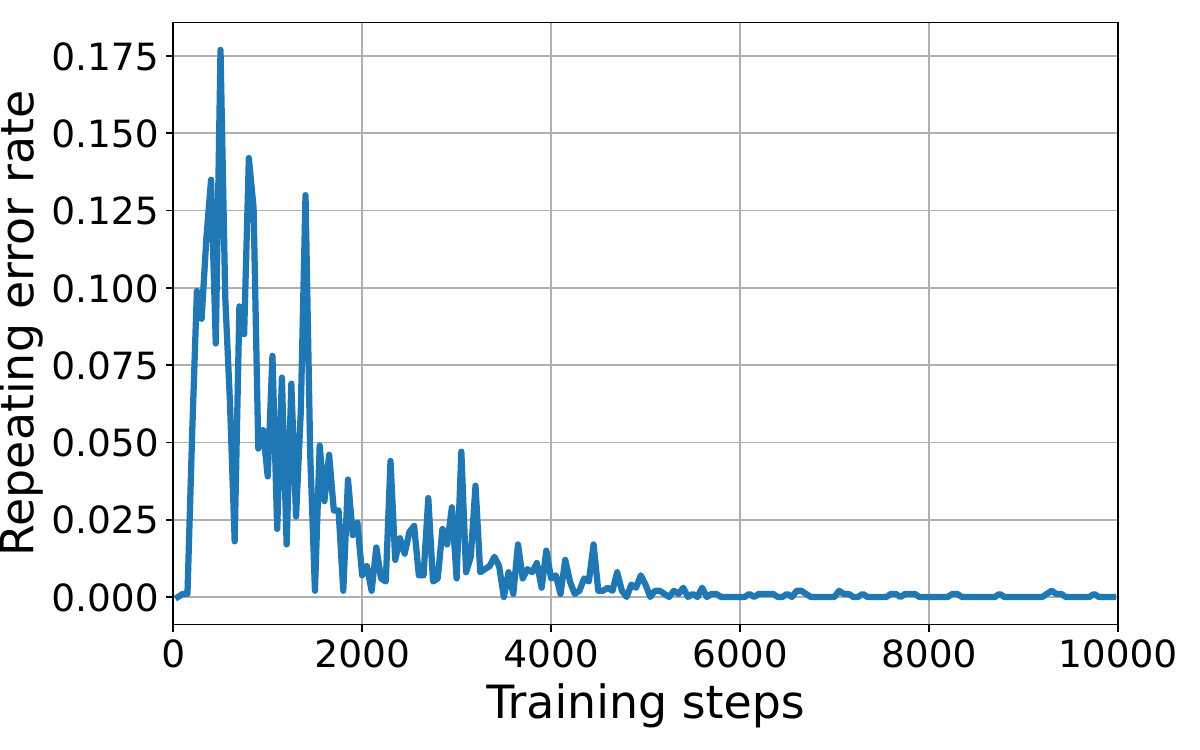} 
        \caption{Repeating error in conflicting group}
    \end{subfigure}%
    \hfill % Adds horizontal space between subfigures
    \begin{subfigure}{0.5\textwidth} % Adjust width as needed
        \centering
        \includegraphics[width=\linewidth]{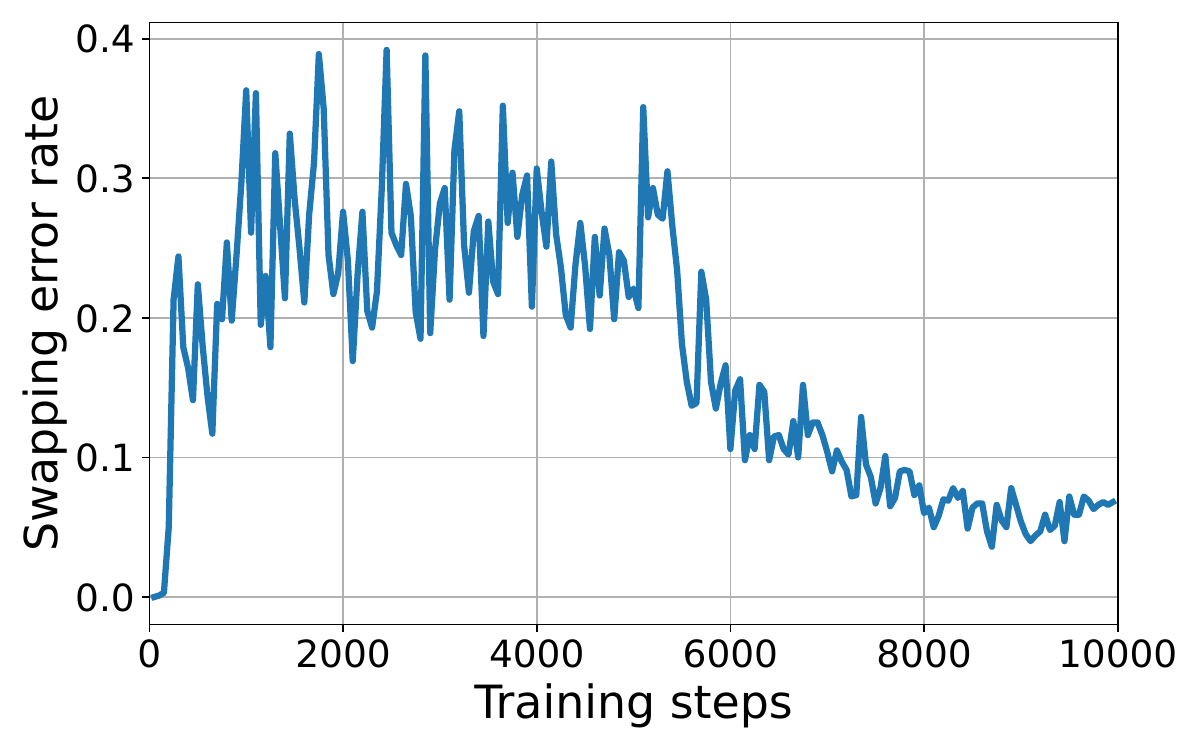}
        \caption{Swapping error in conflicting group}

    \end{subfigure}

\caption{\textbf{Repeating errors disappear as training progresses.} In the conflicting test group, models first experience a high rate of repeating error, then gradually eliminates this subcategory of mixing error as we increase the training steps. This contrasts with the swapping error, which persists and converges to around 6\% at the end of training.
    } 
    \label{fig:repeating_error_disappear}
\end{figure}

\noindent\textbf{Dynamics of repeating error} When subdividing mixing error into two finer subcategories, swapping and repeating, we find that unlike swapping errors which persist throughout the training process the repeating errors are gradually eliminated. The dynamics of the two types of errors are shown in Figure~\ref{fig:repeating_error_disappear}.

\section{Details about GSM8K-based question templates}\label{sec:append-clauses}

Most clauses we add can be classified into two types based on how that information is used in the summation. Let $x_1 + x_2 +...+x_i$ be the multi-operand addition we need to perform.
\begin{enumerate}
    \item \textbf{Type 1: }The new term $x_i$ is a product of two numbers $a_i$ and $b_i$ which are directly provided in the new clause. For a concrete example, one original question is "There are 6 boxes of crayons that hold 8 orange crayons. There are 7 boxes of crayons that have 5 blue crayons. There is 1 box of 11 red crayons. How many crayons are there in total?" A $k=1$ extended question would add the clause "There are 4 boxes of crayons that hold 6 green crayons" before asking the total number of crayons.
    \item \textbf{Type 2: }The new term $x_i$ must be derived by using information from earlier clauses, thus requiring the model to form a reasoning chain to compute the correct term. For a concrete example, if the original question is "McKenna has 34 stuffed animals. Kenley has twice as many as McKenna. Tenly has 5 more than Kenley . How many stuffed animals do the three girls have in all?" A $k=1$ extended question would add the clause "Riley has 10 fewer stuffed animals than Tenly" before asking how many stuffed animals in total these people have.
\end{enumerate}

\subsection{Extended question examples}
Here we provide five examples of our extended question with $k=6$.
\begin{enumerate}
    \item \textbf{Example 1: } Original question:
\begin{Verbatim}[breaklines=true,fontsize=\small,xleftmargin=1em]
Toulouse has twice as many sheep as Charleston. Charleston has 4 times as many sheep as Seattle. How many sheep do Toulouse, Charleston, and Seattle have together if Seattle has 20 sheep?
\end{Verbatim}
    $k = 6$ extended question:
\begin{Verbatim}[breaklines=true,fontsize=\small,xleftmargin=1em]
    Toulouse has 2 times as many sheep as Charleston. Charleston has 4 times as many sheep as Seattle. Seattle has 2 times as many sheep as Alice. Alice has 5 times as many sheep as Bob. Bob has 3 times as many sheep as David. David has 2 times as many sheep as Zoe. Zoe has 4 times as many sheep as John. John has 5 times as many sheep as Rachel. How many sheep do Toulouse, Charleston, Seattle, Alice, Bob, David, Zoe, John and Rachel have together if Rachel has 20 sheep?
\end{Verbatim}
    \item \textbf{Example 2: }
    Original question:
\begin{Verbatim}[breaklines=true,fontsize=\small,xleftmargin=1em]
    Henrietta is repainting her house. The walls in the living room take up 600 square feet. She has three bedrooms. The walls in each bedroom take up 400 square feet. If one gallon of paint can cover 600 square feet, how many gallons of paint does Henrietta need to paint her house?
\end{Verbatim}
    $k = 6$ extended question:
\begin{Verbatim}[breaklines=true,fontsize=\small,xleftmargin=1em]
    The living room walls take up 600 square feet. The bedroom walls (each) take up 400 square feet and there are 3 of them. The kitchen walls (each) take up 500 square feet and there are 6 of them. The dining room walls (each) take up 240 square feet and there are 5 of them. The bathroom walls (each) take up 150 square feet and there are 8 of them. The garage walls (each) take up 480 square feet and there are 5 of them. The hallway walls (each) take up 900 square feet and there are 8 of them. The sunroom walls (each) take up 300 square feet and there are 6 of them. If one gallon of paint covers 600 square feet, how many gallons of paint does Henrietta need to paint her house?
\end{Verbatim}
    \item \textbf{Example 3: }
    Original question:
\begin{Verbatim}[breaklines=true,fontsize=\small,xleftmargin=1em]
    Bill is making omelets for his family's breakfast. It takes him 3 minutes to chop a pepper, 4 minutes to chop an onion, and 1 minute to grate enough cheese for one omelet. It takes him 5 minutes to assemble and cook the omelet. If he needs to chop up four peppers, two onions, and also grates cheese for cooking each of five omelets, how long will he spend preparing for and cooking the five omelets?  
\end{Verbatim}
    $k = 6$ extended question:
\begin{Verbatim}[breaklines=true,fontsize=\small,xleftmargin=1em]
    Bill is making omelets for his family's breakfast. It takes him 3 minutes to chop a pepper, 4 minutes to chop an onion, and 1 minute to grate enough cheese for one omelet. It takes him 5 minutes to assemble and cook one omelet. He needs to chop up 4 peppers, chop 2 onions, and grate cheese for cooking each of 5 omelets. He also needs to chop 3 tomatoes, and it takes him 4 minutes to chop a tomato. It takes him 2 minutes to slice enough mushrooms for one omelet, and he adds mushrooms to each of the 5 omelets. It takes him 3 minutes to dice enough ham for one omelet, and he adds ham to each of the 5 omelets. It takes him 1 minute to rinse a handful of spinach for one omelet, and he adds spinach to each of the 5 omelets. It takes him 2 minutes to crack and whisk the eggs for one omelet, and he does this separately for each of the 5 omelets. After he finishes cooking, it takes him 6 minutes to wash the cutting board and knife. How long will he spend preparing for and cooking the 5 omelets?
\end{Verbatim}
    \item \textbf{Example 4: }
    Original question:
\begin{Verbatim}[breaklines=true,fontsize=\small,xleftmargin=1em]
    In a spelling contest held in her school, Drew got 20 questions correct, winning her the competition. She got six questions wrong, and her competitor Carla got 14 questions correct, and twice as many questions wrong as the number of questions Drew got wrong. If each competitor was asked a different question, how many questions were asked in the competition? 
\end{Verbatim}
    $k = 6$ extended question:
\begin{Verbatim}[breaklines=true,fontsize=\small,xleftmargin=1em]
    In a spelling contest held in her school, Drew got 20 questions correct, winning her the competition. She got six questions wrong, and her competitor Carla got 14 questions correct and twice as many questions wrong as the number of questions Drew got wrong. Another competitor, Maya, got 3 fewer questions correct than Drew and 2 more questions wrong than Drew. Another competitor, Eli, got 5 more questions correct than Carla and half as many questions wrong as Carla. Another competitor, Sasha, got twice as many questions correct as the number of questions Drew got wrong, and she got 4 fewer questions wrong than Carla. Another competitor, Noah, got the same number of questions correct as Carla and 3 fewer questions wrong than Drew. Another competitor, Priya, got 2 fewer questions correct than Eli and 1 more question wrong than Sasha. Another competitor, Jamal, got 4 fewer questions correct than Maya and twice as many questions wrong as Noah. If each competitor was asked a different question, how many questions were asked in the competition?
\end{Verbatim}
    \item \textbf{Example 5: }
    Original question:
\begin{Verbatim}[breaklines=true,fontsize=\small,xleftmargin=1em]
    Buffy has a sailboat with a planing hull that she sails in the Gulf of Mexico. Under optimal wind conditions, the use of two sails brings the ship to the speed of 50 knots, whereas under the same conditions, the use of one sail brings her ship to the speed of 25 knots. A knot is a nautical term for speed equal to 1 nautical mile per hour, and one nautical mile equals 1.15 land miles. If she travels in her sailboat under optimal wind conditions for 4 hours with one sail and then for another 4 hours with two sails, what is the total distance, in land miles, that she will travel?
\end{Verbatim}
    $k = 6$ extended question:
\begin{Verbatim}[breaklines=true,fontsize=\small,xleftmargin=1em]
    Buffy has a sailboat with a planing hull that she sails in the Gulf of Mexico. Under optimal wind conditions, the use of two sails brings the ship to the speed of 50 knots, whereas under the same conditions, the use of one sail brings her ship to the speed of 25 knots. A knot is a nautical term for speed equal to 1 nautical mile per hour, and one nautical mile equals 1.15 land miles. If she travels in her sailboat under optimal wind conditions for 4 hours with one sail and then for another 4 hours with two sails, Under the same conditions, the use of three sails brings the ship to the speed of 60 knots, and she sails for 1 hour using three sails. Under the same conditions, the use of four sails brings the ship to the speed of 70 knots, and she sails for 2 hours using four sails. Under the same conditions, sailing with one reefed sail brings her ship to the speed of 15 knots, and she sails for 4 hours with a reefed sail. Under the same conditions, the use of two sails while carrying extra cargo brings her ship to the speed of 40 knots, and she sails for 1 hour using two sails with the extra cargo. Under the same conditions, the use of one sail with a favorable current brings her ship to the speed of 30 knots, and she sails for 2 hours with that setup. Under the same conditions, if she lowers all sails and just drifts, her ship moves at 5 knots, and she drifts for 4 hours. what is the total distance, in land miles, that she will travel?
\end{Verbatim}
\end{enumerate}
\subsection{Common Error Types}
\label{gsm-error-type}
We examined some failure cases of LLMs in our clause-extended question test. Common types of error include the following.

\begin{enumerate}
    \item \textbf{Direct calculation mistakes in multi-operand addition} Among the error cases, often the model attempts to perform direct addition on multiple operands (in contrast to adding one by one), but made a mistake in its calculation. An example is $4 + 3 + 2 + 1 + 5 + 3 + 4 + 6 = 32 \text{ inches}$, where the correct sum should be 28.
    \item \textbf{Repeating one item} The model sometimes double-counted one term. For example in earlier step it correctly writes the summation it needs to perform as $64 + 256 + 32 + 32 + 16 + 48 + 64 + 48 + 32$, but later when adding one by one it writes
    \begin{verbatim}
        **Add the numbers step by step:**
  - \( 64 + 256 = 320 \)
  - \( 320 + 32 = 352 \)
  - \( 352 + 32 = 384 \)
  - \( 384 + 16 = 400 \)
  - \( 400 + 48 = 448 \)
  - \( 448 + 64 = 512 \)
  - \( 512 + 48 = 560 \)
  - \( 560 + 48 = 608 \) // It double-counted this "48"
  - \( 608 + 32 = 640 \)
    \end{verbatim}
    
    \item \textbf{Missing one item} Sometimes the LLM does not add all the relevant terms. An example is when earlier it calculated the produce of guavas is 100 kg, but this did not get added in the final summation. When it calculate the total produce of all fruits, it misses this item and writes $\text{Total produce} = 400 + 800 + 600 + 700 + 300 + 350 + 200 + 450 = 4000 \text{ kg}$.
    \item \textbf{Incorrectly interpreting the units of intermediate results} For example, in earlier step the model calculates one intermediate result to be 42 pills/week. Later when calculating others it uses as 42 pills/day.
    \item \textbf{Fail to locate effective information} Sometimes the model quotes as irrelevant piece of information as its argument for performing one reasoning. For example, one question says "“Wanda brings half as many treats as Jane”, but the model analyzes that "Jane brings 75\% as many pieces of bread as treats. Let \( T_J \) be the number of treats Jane brought. Then:
\(T_J = T_W = 30\)". To correctly calculate $T_J$, it should use the information “Wanda brings half as many treats as Jane” from the question and writes $T_J = 2 * T_W = 60$ instead.
    
\end{enumerate}

The "direct calculation mistakes" emphasize the challenges for transformer-based models to perform long arithmetic. The other types, such as "repeating/missing one item", bear close resemblance to the mixing errors we observed in the sorting task, indicating the difficulty for transformers to master human-rule compositions.
\section{Details about scaling tests}
\label{scaling_test}
\subsection{Pure Model scaling}
\label{app:pure_model_scaling}

To investigate the effect of model scaling on learning order, we scaled the NanoGPT architecture to 20M and 100M parameters. The architectural details for each model size are presented in Table \ref{tab:model_scaling}. All three variants were trained using the identical hyperparameters listed in Table \ref{tab:hyperparams_comparison}.

We observed that the number of iterations required to learn the addition task increased as we scaled up the model parameters. However, Figure~\ref{fig:addition-scaling_scratch} illustrates that despite this difference in capacity, both the 20M and 100M models exhibited the same reverse digit-wise learning order. These results suggest that parameter scaling does not fundamentally alter the underlying stagewise learning dynamics for arithmetic tasks.

\begin{figure}[t]
    \centering
    \begin{subfigure}{0.5\textwidth} % Adjust width as needed
                \centering
        \includegraphics[width=\linewidth]{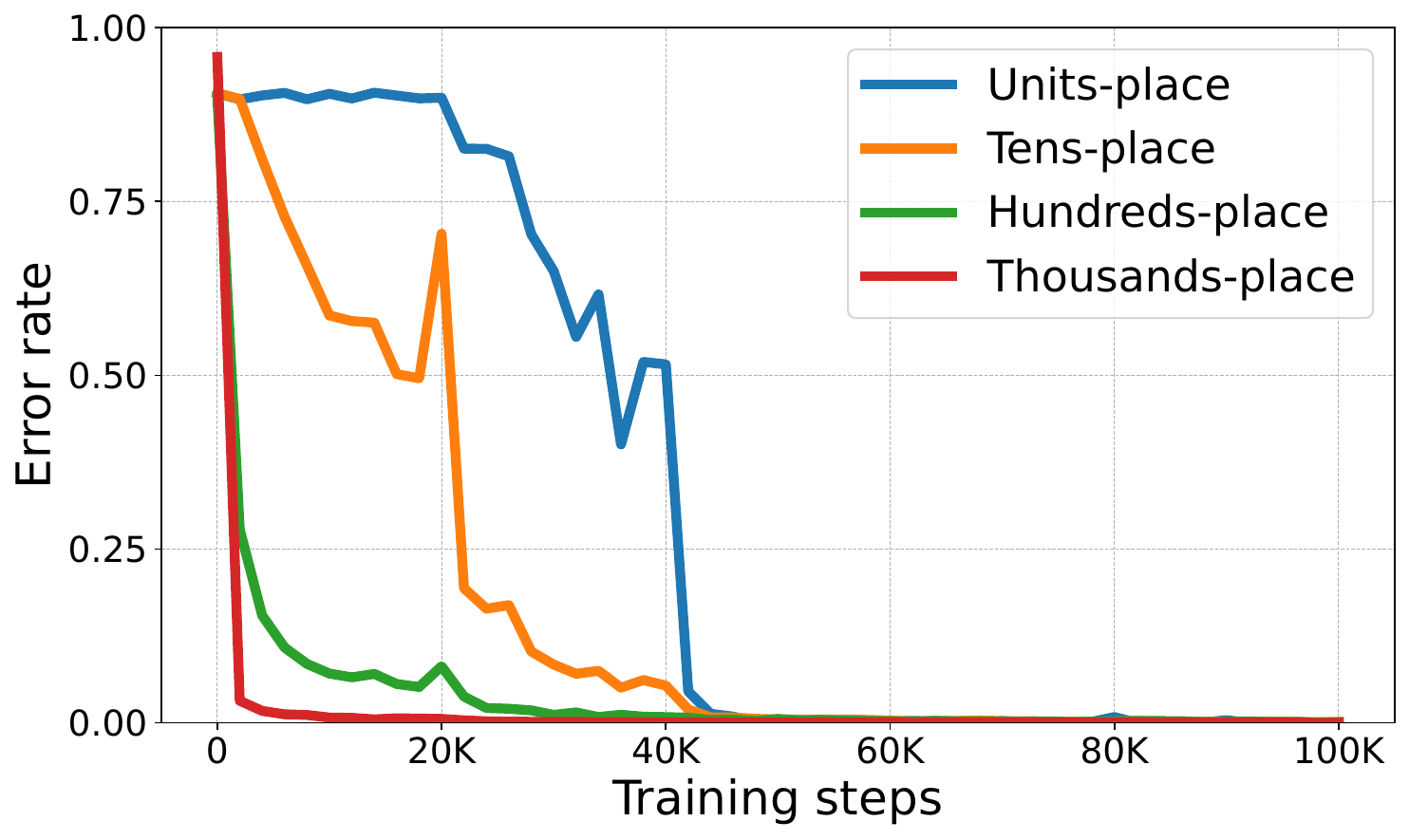}
        \caption{Digit-wise error rate for 20M NanoGPT model}
    \end{subfigure}%
    \hfill % Adds horizontal space between subfigures
    \begin{subfigure}{0.5\textwidth} % Adjust width as needed
          \centering
            \includegraphics[
            width=\linewidth,
            keepaspectratio
          ]{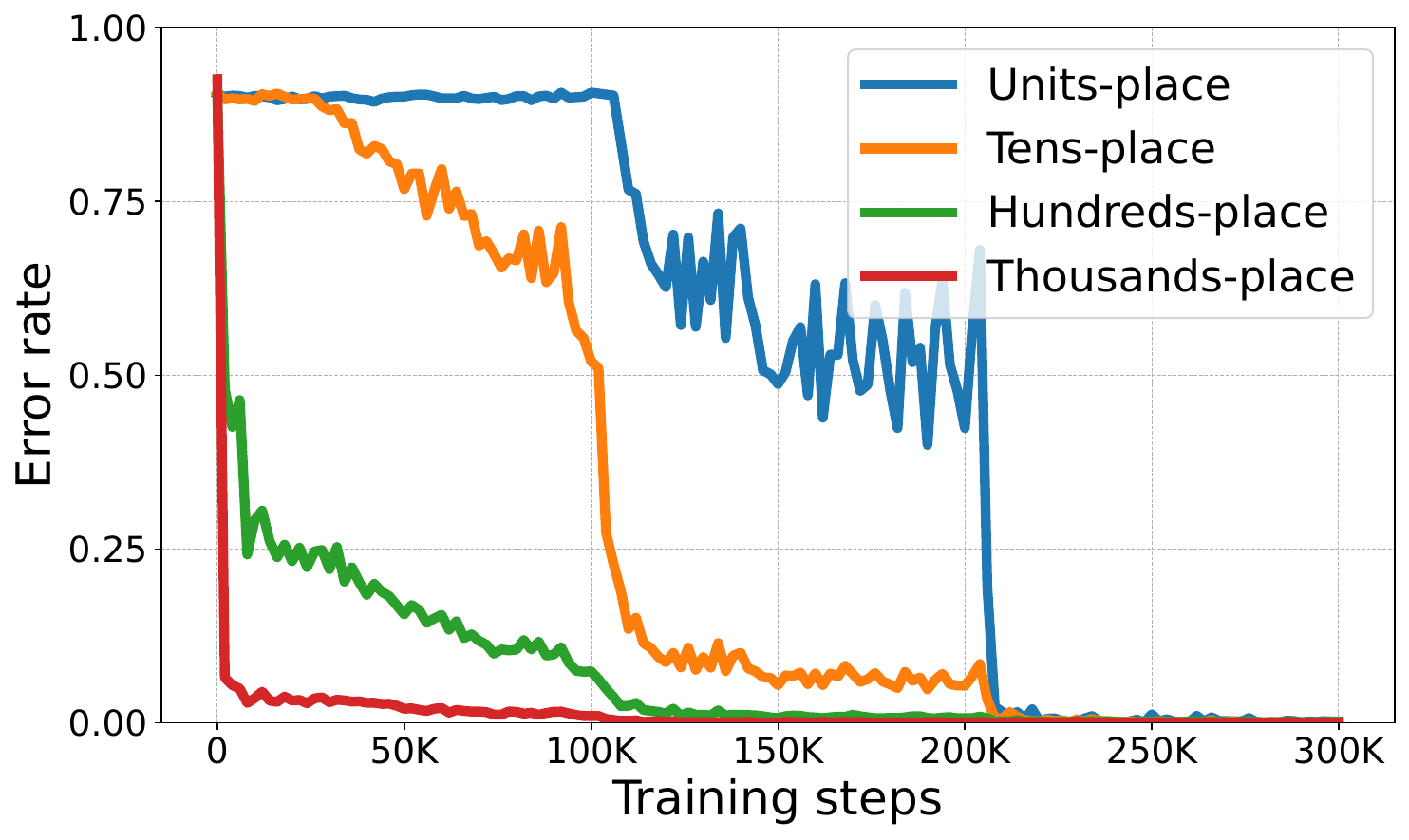}
          \caption{Digit-wise error rate for 100M NanoGPT model}
    \end{subfigure}
    \caption{\textbf{Using larger models to train the addition task from scratch model still learn digits in reverse format}. The results are for 20M and 100M NanoGPT models for addition of the format $a+b+c+d=e$. }
    \label{fig:addition-scaling_scratch}
\end{figure}

\subsection{Pretrained Model Fine-tuning}
\label{app:pretraining_model}
To investigate the effect of pretraining on arithmetic learning dynamics, we fine-tuned the Pythia-1B model on a 4-operand addition task. Prior to fine-tuning, we evaluated the off-the-shelf model on the addition task and the model achieved $0\%$ accuracy, confirming that it possessed no zero-shot capability for this specific task.

For the fine-tuning process, we utilized a learning rate of $1 \times 10^{-4}$. Full details of the training configuration and hyperparameters are provided in Table~\ref{tab:hyperparams_comparison}.

We limited the fine-tuning of the pretrained Pythia-1B model to 200,000 iterations. Although the model did not reach 100\% accuracy within this budget, Figure~\ref{fig:addition-pretrained} illustrates that the learning dynamics remained consistent with our previous observations. Specifically, the model exhibited the characteristic reverse learning order, prioritizing the acquisition of the most significant digit, signifying that pre-training does not fundamentally alter this learning dynamic.

\begin{figure}[t]
    \centering
    % Change 1.0\textwidth to 0.6\textwidth here
    \begin{subfigure}{0.6\textwidth} 
        \centering
        \includegraphics[width=\linewidth]{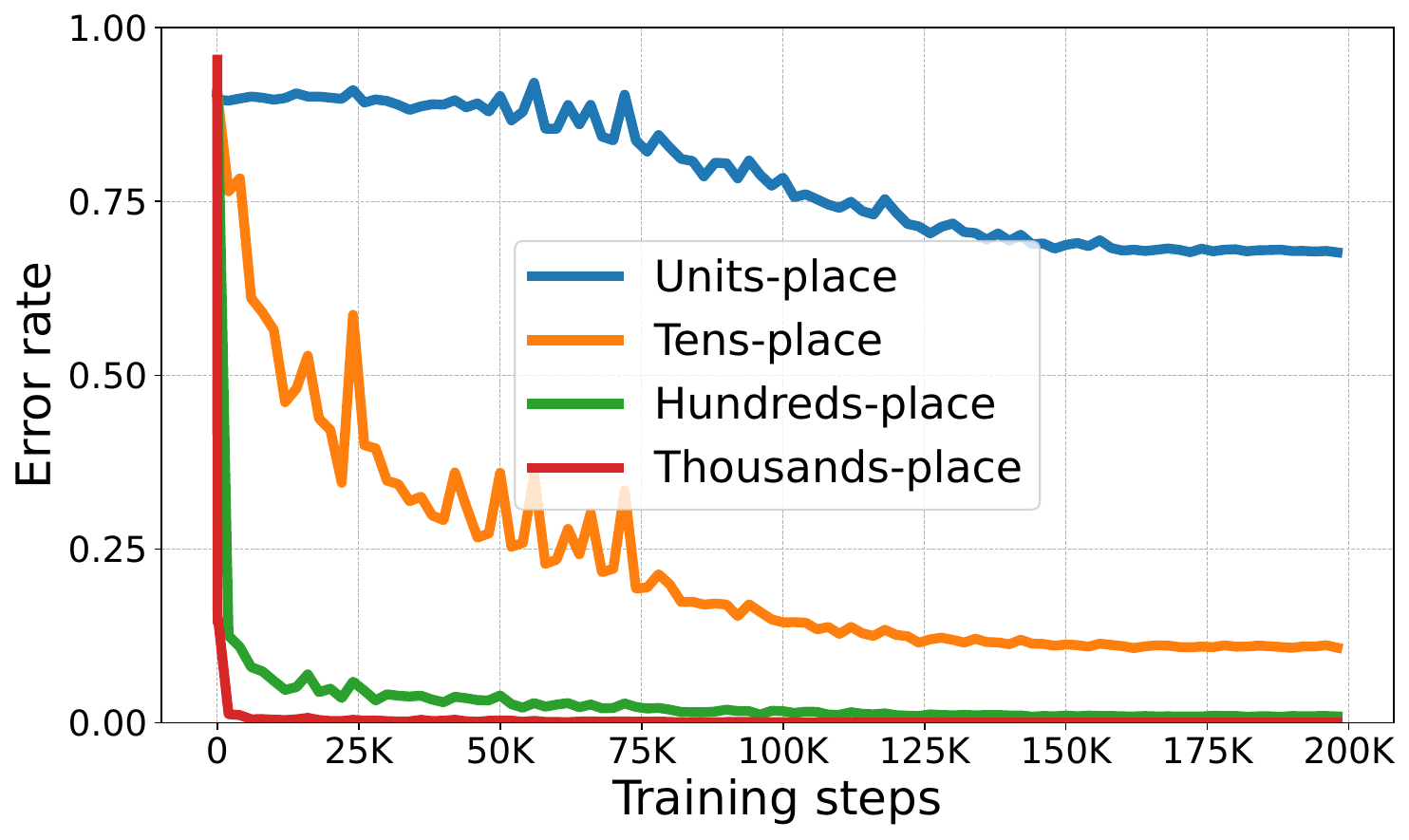}
        \caption{Digit-wise error rate for Pythia-1B model}
    \end{subfigure}
    \caption{\textbf{Pretrained model fine-tuned for addition learn digits in non-human order}. We fine-tune Pythia-1B model on addition of the format $a+b+c+d=e$ for plain-format. 
    The result shows that pretraining on diverse data doesn't affect the model learning arithmetic tasks in the higher-place-first order.
    % The result shows that pre-training on diverse data doesn't affect how the model internally learns arithmetic tasks.
    }
    \label{fig:addition-pretrained}
\end{figure}

\subsection{Scratchpad-based Reasoning Training}\label{sec:append-scratchpad}

%To evaluate the impact of explicit reasoning steps on learning dynamics, we employed a scratchpad training strategy. 

We employed a scratchpad training strategy to induce intermediate reasoning steps similar to CoT reasoning.
We utilized our standard NanoGPT architecture trained with default hyperparameters (refer to Table~\ref{tab:model_scaling} and Table~\ref{tab:hyperparams_comparison} for full configuration). Here we provide further explanations for the two scratchpad formats.

%We introduced two distinct formatting strategies for the 4-operand addition reasoning traces:

\begin{enumerate}
    \item \textbf{\underline{D}irect Decomposition (D scratchpad):} The model decomposes the addends by place value and immediately computes the final sum.
    \begin{equation*}
        811+856+239+313 = 100(8+8+2+3) + 10(1+5+3+1) + 1(1+6+9+3) = 2219
    \end{equation*}
    
    \item \textbf{\underline{D}ecomposition plus \underline{A}ggregation (D+A scratchpad):} The model decomposes the addends and explicitly calculates the sum for each place value group before converging on the final result.
    \begin{equation*}
    \begin{aligned}
        811+856+239+313 &= 100(8+8+2+3) + 10(1+5+3+1) + 1(1+6+9+3) \\
        &= 100(21)+10(10)+1(19) \\
        &= 2219
    \end{aligned}
    \end{equation*}
\end{enumerate}

Both scratchpad formats significantly outperformed standard training, requiring substantially fewer iterations to reach convergence. Specifically, the \textbf{Decomposition plus Aggregation (D+A)} scratchpad required only 1,000 iterations and the \textbf{Direct Decomposition (D)} format 10,000 iterations, representing a dramatic reduction compared to the approximately 60,000 iterations needed for the standard model.
Figure~\ref{fig:scratchpad_dynamics} illustrates the distinct learning dynamics for each format:
\begin{itemize}
    \item \textbf{Direct Decomposition (D):} As shown in Figure~\ref{fig:scratchpad_dynamics} (a), the model still exhibited the reverse learning order, prioritizing higher-order digits before lower-order ones, mirroring the behavior of the base model.
    \item \textbf{Decomposition plus Aggregation (D+A):} In contrast, Figure~\ref{fig:scratchpad_dynamics} (b) reveals that the D+A format altered this dynamic slightly; while it still learned the Most Significant Digit (MSD) first, the remaining three positions were learned nearly simultaneously rather than sequentially.
    \begin{figure}[t]
    \centering
    \begin{subfigure}{0.47\textwidth} % Adjust width as needed
                \centering
        \includegraphics[width=\linewidth]{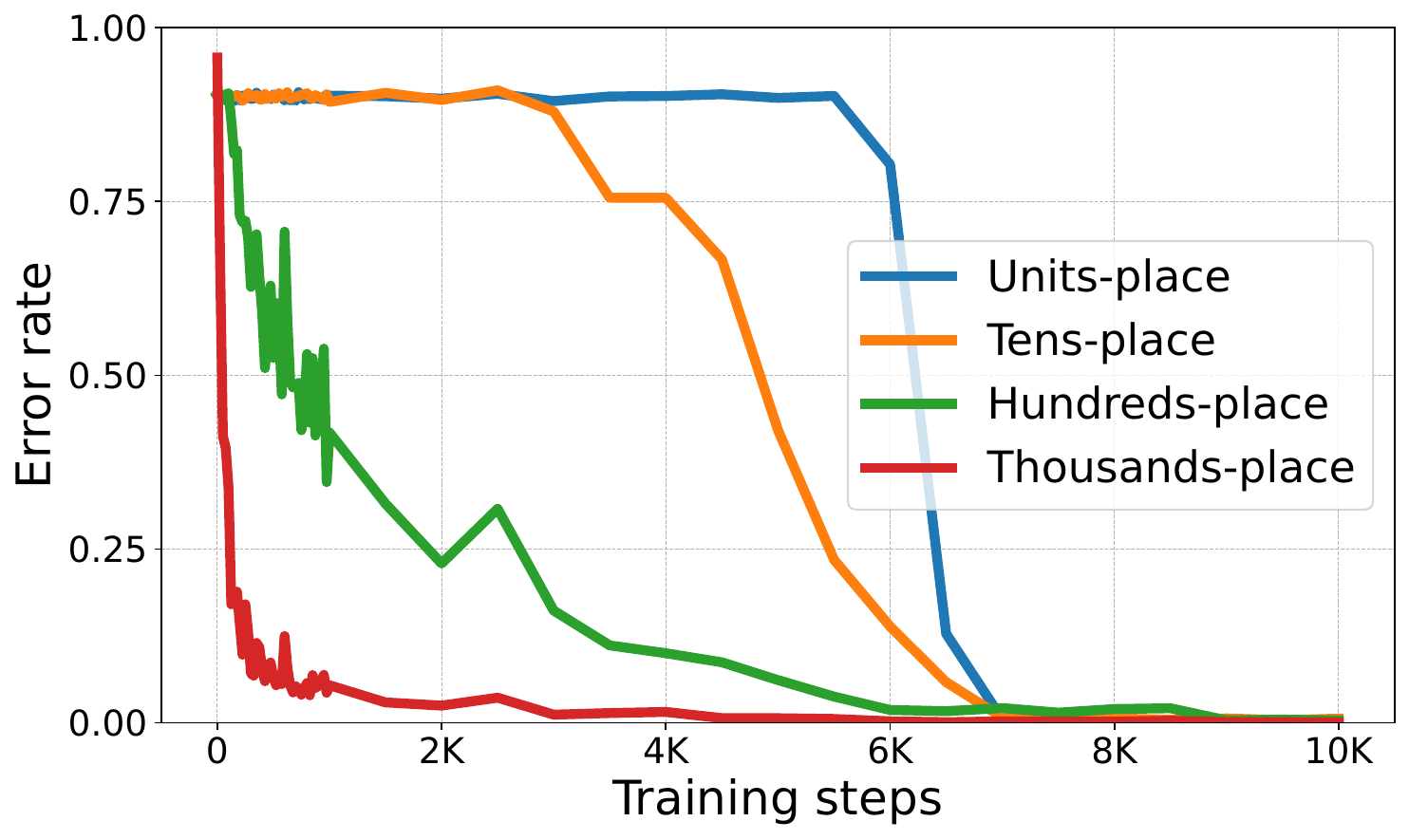}
        \caption{Digit-wise error rate for Direct Decomposition (D)}
    \end{subfigure}%
    \hfill % Adds horizontal space between subfigures
    \begin{subfigure}{0.47\textwidth} % Adjust width as needed
          \centering
            \includegraphics[
            width=\linewidth,
            keepaspectratio
          ]{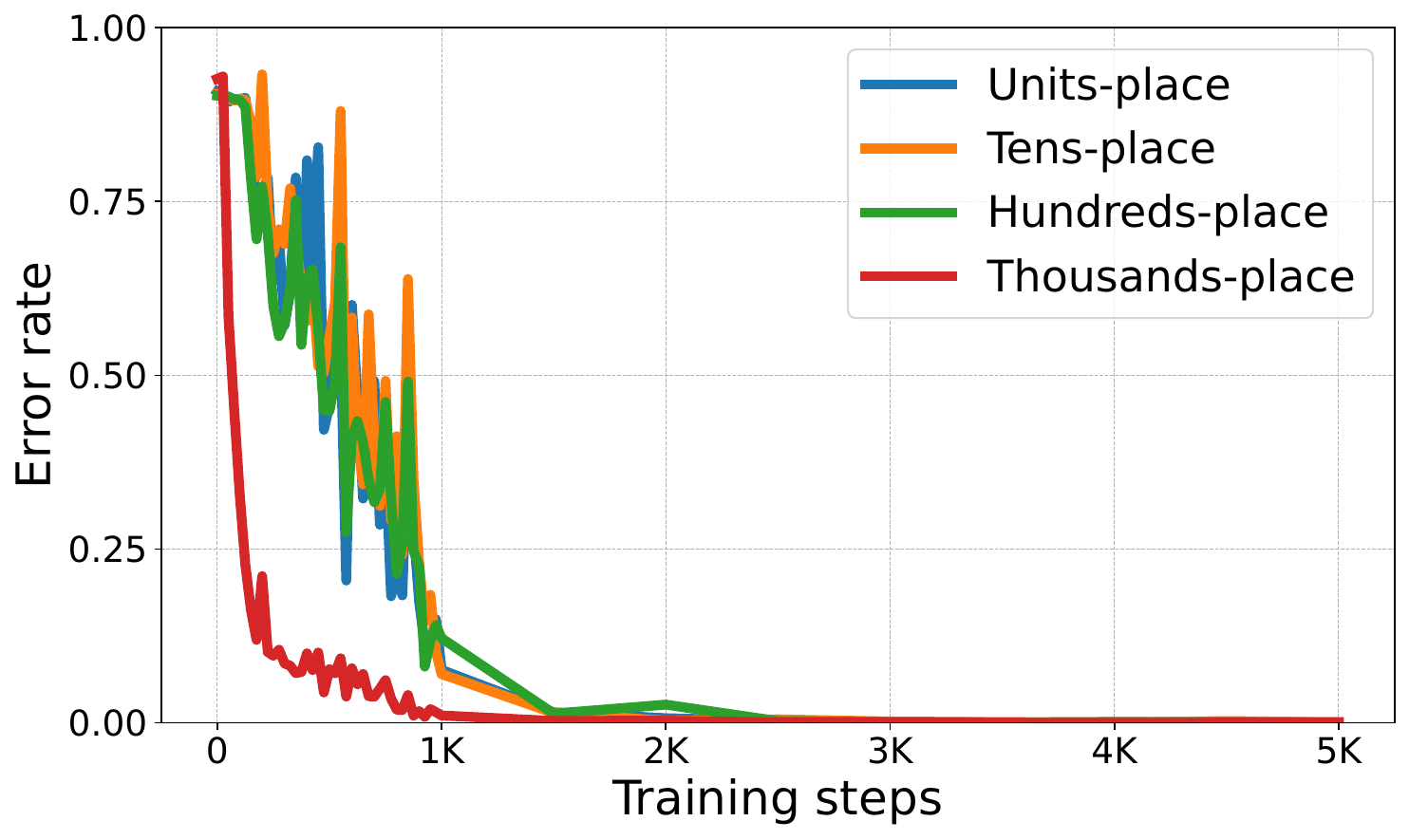}
          \caption{Digit-wise error rate for Decomposition plus Aggregation (D+A)}
    \end{subfigure}
    \caption{\textbf{Scratchpads accelerate convergence but fail to enforce human learning order.} Although scratchpads reduce training steps, neither format successfully restores the Least-to-Most Significant Digit progression inherent to the addition algorithm. }
    \label{fig:scratchpad_dynamics}
\end{figure}

\end{itemize}
\section{Proof of Theorem~\ref{thm:1}}\label{app:proof-thm1}
\begin{proof}
\medskip
\textbf{\underline{Part 1}: we will show $I(a_i;e_0)>0$ for $i \in\{1,2,3\}$.}

Write the full numbers as $A,B,C,D\in\{0,\dots,999\}$. Let $S := A + B + C + D.$ By definition $e_0 \ge 1$ iff $S \ge 1000.$

\medskip\noindent\textbf{Case 1: $a_1$ (hundreds digit).}
Write
\[
S = 100a_1 + R,\qquad R := 10a_2 + a_3 + B + C + D,
\]
so $R$ is independent of $a_1$. Then
\[
\Pr(e_0\ge1\mid a_1=0)=\Pr(R\ge1000),\qquad
\Pr(e_0\ge1\mid a_1=9)=\Pr(R\ge100).
\]
Since $\{R\ge1000\}\subseteq\{R\ge100\}$, these two probabilities are equal only if
$\Pr(100\le R\le 999)=0$. But that event has positive probability: for example the concrete assignment
\[
a_2=0,\ a_3=0,\ B=500,\ C=0,\ D=0
\]
gives $R=500$, so $\Pr(R=500)>0$. Hence
\[
\Pr(R\ge1000)<\Pr(R\ge100),
\]
so $\Pr(e_0\ge1\mid a_1=0)\neq\Pr(e_0\ge1\mid a_1=9)$. Therefore $a_1$ and $e_0$ are not independent and
$I(a_1;e_0)>0$.

\medskip\noindent\textbf{Case 2: $a_2$ (tens digit).}
Write
\[
S = 10 a_2 + R,\qquad R := 100a_1 + a_3 + B + C + D,
\]
so $R$ is independent of $a_2$. Then
\[
\Pr(e_0\ge1\mid a_2=0)=\Pr(R\ge1000),\qquad
\Pr(e_0\ge1\mid a_2=9)=\Pr(R\ge910).
\]
Again $\{R\ge1000\}\subseteq\{R\ge910\}$; these probabilities differ if
$\Pr(910\le R\le999)>0$. A concrete realization shows this event has positive probability:
take
\[
a_1=9,\ a_3=0,\ B=10,\ C=0,\ D=0,
\]
which yields $R=900+0+10+0+0=910$. Thus $\Pr(R=910)>0$, so
\[
\Pr(R\ge1000)<\Pr(R\ge910),
\]
hence $\Pr(e_0\ge1\mid a_2=0)\neq\Pr(e_0\ge1\mid a_2=9)$, and $I(a_2;e_0)>0$.

\medskip\noindent\textbf{Case 3: $a_3$ (units digit).}
Write
\[
S = a_3 + R,\qquad R := 100a_1 + 10a_2 + B + C + D,
\]
so $R$ is independent of $a_3$. Then
\[
\Pr(e_0\ge1\mid a_3=0)=\Pr(R\ge1000),\qquad
\Pr(e_0\ge1\mid a_3=9)=\Pr(R\ge991).
\]
These differ if $\Pr(991\le R\le999)>0$. A concrete example again shows positive probability:
take
\[
a_1=9,\ a_2=0,\ B=91,\ C=0,\ D=0,
\]
which yields $R=900+0+91+0+0=991$. Thus $\Pr(R=991)>0$, so
\[
\Pr(R\ge1000)<\Pr(R\ge991),
\]
hence $\Pr(e_0\ge1\mid a_3=0)\neq\Pr(e_0\ge1\mid a_3=9)$, and $I(a_3;e_0)>0$.

\medskip
\noindent\textbf{\underline{Part 2}: We will show that $I(a_i;e_j)=0$ for \(i,j\in\{1,2,3\}\).}

To avoid a name clash with the digit $c_i$, for each place \(i\) let \(C_{\mathrm{in}}^{(i)}\) denote the carry coming into that place from the less-significant places (so \(C_{\mathrm{in}}^{(3)}=0\)).

\noindent\textbf{Group 1: $I(a_2;e_1), I(a_3;e_1), I(a_3;e_2)$}

We prove the representative case $I(a_2;e_1)=0$. The other two follow by similar arguments. By definition,
\[
e_1 = \left( a_1 + b_1 + c_1 + d_1 + C_{\mathrm{in}}^{(1)} \right) \pmod{10}
\]
where $C_{\mathrm{in}}^{(1)} = \lfloor \frac{(A \bmod 100) + (B \bmod 100) + (C \bmod 100) + (D \bmod 100)}{100} \rfloor$.

Let $S_{100} = (a_1 + b_1 + c_1 + d_1) \pmod{10}$. Since $a_1, b_1, c_1, d_1 \sim \text{Uniform}\{0, \dots, 9\}$, their sum modulo 10 is also uniformly distributed: $S_{100}\sim \text{Uniform}\{0, \dots, 9\}.$
Observe that the hundreds digits $\{a_1, b_1, c_1, d_1\}$ are independent of $a_2$ and the other lower-order digits. Thus:
\[
\{a_1, b_1, c_1, d_1\} \perp C_{\mathrm{in}}^{(1)} \quad \text{and} \quad \{a_1, b_1, c_1, d_1\} \perp a_2
\]
So $S_{100}$ is independent of $C_{\mathrm{in}}^{(1)}$ and $a_2.$ We evaluate the distribution of $e_1$ conditioned on $a_2$. 

\begin{align*}
    P(e_1 = k \mid a_2 = x) &= \sum_{c} P(e_1 = k \mid C_{\mathrm{in}}^{(1)} = c, a_2 = x) P(C_{\mathrm{in}}^{(1)} = c \mid a_2 = x) \\
    &= \sum_{c} P(S_{100} \equiv k - c \pmod{10}) P(C_{\mathrm{in}}^{(1)} = c \mid a_2 = x) \\
    &= \sum_{c} \left(\frac{1}{10}\right) P(C_{\mathrm{in}}^{(1)} = c \mid a_2 = x) \\
    &= \frac{1}{10} \sum_{c} P(C_{\mathrm{in}}^{(1)} = c \mid a_2 = x) \\
    &= \frac{1}{10}
\end{align*}
Since $P(e_1 = k \mid a_2 = x) = P(e_1 = k) = 1/10$, the variables $e_1$ and $a_2$ are independent. Hence, $I(a_2, e_1) = 0.$

\noindent\textbf{Group 2: $I(a_i;e_i) = 0$ for $i\in\{1,2,3\}$}

Since all input digits are independent and uniformly distributed, in
\[
e_i = (a_i+b_i+c_i+d_i+C_{\mathrm{in}}^{(i)}) \pmod{10},
\]
we have the following independence structure. For a fixed position
\(i\in\{1,2,3\}\), the digit \(a_i\) is independent of the other
digits in the same column, namely \(b_i,c_i,d_i\). Moreover,
\(C_{\mathrm{in}}^{(i)}\) is a function only of digits in strictly
lower-order columns \(j>i\), and is therefore independent of all
current-column digits \(a_i,b_i,c_i,d_i\).

Let $X = (b_i + c_i + d_i + C_{\mathrm{in}}^{(i)}) \pmod{10}$. 
Since $b_i, c_i, d_i$ are independent uniform variables on $\{0, \dots, 9\}$, their sum modulo 10 is uniformly distributed. Adding $C_{\mathrm{in}}^{(i)}$ (which is independent of $b_i, c_i, d_i$) results in a cyclic shift of a uniform distribution, which remains uniform. Thus:
\[
X \sim \text{Uniform}\{0, \dots, 9\} \quad \text{and} \quad X \perp a_i
\]

We can express $e_i$ as \(e_i = (a_i + X) \pmod{10}\).
Because $X$ is a uniform random variable independent of $a_i$, the resulting sum $e_i$ is also uniform and independent of $a_i$. Formally, for any $k, x \in \{0, \dots, 9\}$:
\[
P(e_i = k \mid a_i = x) = P(X \equiv k - x \pmod{10}) = \frac{1}{10} = P(e_i = k)
\]
Since the conditional distribution does not depend on $a_i$, $I(a_i, e_i) = 0$.

\noindent\textbf{Group 3: $I(a_1;e_2), I(a_1;e_3), I(a_2;e_3)$}

For every pair in this group, $a$ belong to a higher digit place than $e$. We prove the representative case $I(a_1;e_2)=0$. As before, 
\[
e_2 = (a_2 + b_2 + c_2 + d_2 + C_{\mathrm{in}}^{(2)}) \pmod{10},
\]
where $C_{\mathrm{in}}^{(2)}$ is the carry from the units place, determined by $\{a_3, b_3, c_3, d_3\}$. Since $e_2$ is completely determined by digits from the tens place and the ones place, and by definition $a_1 \perp \{a_2, b_2, c_2, d_2, a_3, b_3, c_3, d_3\}$, $a_1$ and $e_2$ are independent. Hence, $I(a_1;e_2)=0$.

\medskip
\noindent\textbf{\underline{Part 3}: We will show  $I(a_i;e_i\mid c_{i-1})>0$ for \(i=1,2,3\).}

Fix \(i\in\{1,2,3\}\). Write the local sum at position \(i\) as
\[
S \;=\; a_i + b_i + c_i + d_i + C_{\mathrm{in}}^{(i)},
\]
and set
\[
U := b_i + c_i + d_i + C_{\mathrm{in}}^{(i)},
\]
so $S = a_i + U$. The output digit and carry-out are
\[
e_i \;=\; S \bmod 10,\qquad
c_{i-1} := \Big\lfloor \frac{S}{10}\Big\rfloor \in\{0,1,2,3\}.
\]

First note 
\[
I(a_i;e_i\mid c_{i-1}) \;=\; \mathbb{E}_{c_{i-1}}\big[I(a_i;e_i\mid c_{i-1}=c)\big]
\;=\; \sum_{c \in\{0,1,2,3\} } \Pr(c_{i-1}=c)\,I(a_i;e_i\mid c_{i-1}=c).
\]
Since each term is nonnegative, it suffices to find a single carry value \(c\in\{0,1,2,3\}\) with $I(a_i;e_i\mid c_{i-1}=c)>0$. We take \(c=0\).

Condition on $c_{i-1}=0$, i.e. $S=a_i+U\in\{0,1,\dots,9\}$. Consider two values of \(a_i\):

\medskip
\noindent\emph{(i) $a_i=9$ and $c_{i-1}=0$.}  
The inequality $9+U\le 9$ forces $U\le 0$, hence $U=0$. Therefore $e_i = (9+0)\bmod 10 = 9$, and in particular
\[
\Pr\big(e_i=0\mid a_i=9,\,c_{i-1}=0\big)=0.
\]

\noindent\emph{(ii) $a_i=0$ and $c_{i-1}=0$.}  
By uniform sampling there is positive probability to have $b_i=c_i=d_i=C_{\mathrm{in}}^{(i)}=0$, so $U=0$ and thus $e_i=0$ occurs with positive probability. Hence
\[
\Pr\big(e_i=0\mid a_i=0,\,c_{i-1}=0\big)>0.
\]

The two conditional probabilities above are different, so the conditional
distribution of \(e_i\) given \((a_i, c_{i-1}=0)\) depends on \(a_i\).
Therefore,
\[
I(a_i; e_i \mid c_{i-1}=0) > 0.
\]
Since \(\Pr(c_{i-1}=0) > 0\), it follows that
\[
\begin{aligned}
I(a_i; e_i \mid c_{i-1})
&= \sum_{c \in \{0,1,2,3\}} \Pr(c_{i-1}=c)\, I(a_i; e_i \mid c_{i-1}=c) \\
&\ge \Pr(c_{i-1}=0)\, I(a_i; e_i \mid c_{i-1}=0) \\
&> 0.
\end{aligned}
\]

% The two conditional probabilities above are different, so the conditional distribution of $e_i$ given $(a_i,c_{i-1}=0)$ depends on \(a_i\). Thus \(I(a_i;e_i\mid c_{i-1}=0)>0\). Since $\Pr(c_{i-1}=0)>0$, we conclude
% \[
% I(a_i;e_i\mid c_{i-1}) \;=\; \sum_{c \in\{0,1,2,3\} } \Pr(c_{i-1}=c)\,I(a_i;e_i\mid c_{i-1}=c)
% \;\ge\; \Pr(c_{i-1}=0)\,I(a_i;e_i\mid c_{i-1}=0) \;>\; 0.
% \]

This argument applies for each \(i=1,2,3\), completing the proof of the theorem.
\end{proof}

\section{Broader Impacts}
\label{app:broader-impacts}

This work studies how transformers acquire and combine arithmetic
subskills. Its potential positive impact is to improve evaluation
of reasoning reliability: high end-task accuracy can conceal
brittle compositions that fail when task conditions change.
Our diagnostics may help researchers identify these failures and
develop more informative robustness evaluations.

A potential negative impact is overgeneralization from synthetic
arithmetic tasks and a limited collection of models and templates.
This could produce misleading assessments of model capabilities
or unsupported conclusions about deployment readiness. Conversely,
failure on these tests does not establish a general inability to
reason. Our black-box measurements also do not establish that
modern LLMs share the internal mechanisms of the small models
studied here. We therefore restrict our conclusions to the
evaluated settings and encourage validation across additional
tasks, model families, and deployment-relevant conditions.

%%%%%%%%%%%%%%%%%%%%%%%%%%%%%%%%%%%%%%%%%%%%%%%%%%%%%%%%%%%%%%%%%%%%%%%%%%%%%%%
%%%%%%%%%%%%%%%%%%%%%%%%%%%%%%%%%%%%%%%%%%%%%%%%%%%%%%%%%%%%%%%%%%%%%%%%%%%%%%%

\ifarxiv
  % No NeurIPS checklist in the arXiv version
\else
  \newpage
  \section*{NeurIPS Paper Checklist}

\begin{enumerate}

\item {\bf Claims}
    \item[] Question: Do the main claims made in the abstract and introduction accurately reflect the paper's contributions and scope?
    \item[] Answer: \answerYes{}
    \item[] Justification: The abstract and Section~\ref{sec:introduction} state the paper's scope, questions, and contributions, and the subsequent sections substantiate the claims through synthetic arithmetic experiments, robustness checks, and an external GSM8K-style evaluation.
    \item[] Guidelines:
    \begin{itemize}
        \item The answer \answerNA{} means that the abstract and introduction do not include the claims made in the paper.
        \item The abstract and/or introduction should clearly state the claims made, including the contributions made in the paper and important assumptions and limitations. A \answerNo{} or \answerNA{} answer to this question will not be perceived well by the reviewers. 
        \item The claims made should match theoretical and experimental results, and reflect how much the results can be expected to generalize to other settings. 
        \item It is fine to include aspirational goals as motivation as long as it is clear that these goals are not attained by the paper. 
    \end{itemize}

\item {\bf Limitations}
    \item[] Question: Does the paper discuss the limitations of the work performed by the authors?
    \item[] Answer: \answerYes{}
    \item[] Justification: Section~\ref{sec:conclusion} explicitly discusses limitations, including the focus on small transformers, synthetic tasks, and black-box diagnostics rather than circuit-level mechanisms; Section~\ref{sec:external-robustness} also notes that the LLM results are qualitative external validation rather than mechanistic identification.
    \item[] Guidelines:
    \begin{itemize}
        \item The answer \answerNA{} means that the paper has no limitation while the answer \answerNo{} means that the paper has limitations, but those are not discussed in the paper. 
        \item The authors are encouraged to create a separate ``Limitations'' section in their paper.
        \item The paper should point out any strong assumptions and how robust the results are to violations of these assumptions (e.g., independence assumptions, noiseless settings, model well-specification, asymptotic approximations only holding locally). The authors should reflect on how these assumptions might be violated in practice and what the implications would be.
        \item The authors should reflect on the scope of the claims made, e.g., if the approach was only tested on a few datasets or with a few runs. In general, empirical results often depend on implicit assumptions, which should be articulated.
        \item The authors should reflect on the factors that influence the performance of the approach. For example, a facial recognition algorithm may perform poorly when image resolution is low or images are taken in low lighting. Or a speech-to-text system might not be used reliably to provide closed captions for online lectures because it fails to handle technical jargon.
        \item The authors should discuss the computational efficiency of the proposed algorithms and how they scale with dataset size.
        \item If applicable, the authors should discuss possible limitations of their approach to address problems of privacy and fairness.
        \item While the authors might fear that complete honesty about limitations might be used by reviewers as grounds for rejection, a worse outcome might be that reviewers discover limitations that aren't acknowledged in the paper. The authors should use their best judgment and recognize that individual actions in favor of transparency play an important role in developing norms that preserve the integrity of the community. Reviewers will be specifically instructed to not penalize honesty concerning limitations.
    \end{itemize}

\item {\bf Theory assumptions and proofs}
    \item[] Question: For each theoretical result, does the paper provide the full set of assumptions and a complete (and correct) proof?
    \item[] Answer: \answerYes{}
    \item[] Justification: The theoretical result is stated as Theorem~\ref{thm:1}, with its uniform-sampling assumptions stated in the theorem and surrounding text; a complete proof is provided in Appendix~\ref{app:proof-thm1}.
    \item[] Guidelines:
    \begin{itemize}
        \item The answer \answerNA{} means that the paper does not include theoretical results. 
        \item All the theorems, formulas, and proofs in the paper should be numbered and cross-referenced.
        \item All assumptions should be clearly stated or referenced in the statement of any theorems.
        \item The proofs can either appear in the main paper or the supplemental material, but if they appear in the supplemental material, the authors are encouraged to provide a short proof sketch to provide intuition. 
        \item Inversely, any informal proof provided in the core of the paper should be complemented by formal proofs provided in appendix or supplemental material.
        \item Theorems and Lemmas that the proof relies upon should be properly referenced. 
    \end{itemize}

    \item {\bf Experimental result reproducibility}
    \item[] Question: Does the paper fully disclose all the information needed to reproduce the main experimental results of the paper to the extent that it affects the main claims and/or conclusions of the paper (regardless of whether the code and data are provided or not)?
    \item[] Answer: \answerYes{}
    \item[] Justification: Section~\ref{sec:experimental-setup} and the appendix specify the task formats, data-generation procedures, model architecture, optimizer, hyperparameters, decoding, and evaluation metrics needed to reproduce the main experimental results.
    \item[] Guidelines:
    \begin{itemize}
        \item The answer \answerNA{} means that the paper does not include experiments.
        \item If the paper includes experiments, a \answerNo{} answer to this question will not be perceived well by the reviewers: Making the paper reproducible is important, regardless of whether the code and data are provided or not.
        \item If the contribution is a dataset and\slash or model, the authors should describe the steps taken to make their results reproducible or verifiable. 
        \item Depending on the contribution, reproducibility can be accomplished in various ways. For example, if the contribution is a novel architecture, describing the architecture fully might suffice, or if the contribution is a specific model and empirical evaluation, it may be necessary to either make it possible for others to replicate the model with the same dataset, or provide access to the model. In general. releasing code and data is often one good way to accomplish this, but reproducibility can also be provided via detailed instructions for how to replicate the results, access to a hosted model (e.g., in the case of a large language model), releasing of a model checkpoint, or other means that are appropriate to the research performed.
        \item While NeurIPS does not require releasing code, the conference does require all submissions to provide some reasonable avenue for reproducibility, which may depend on the nature of the contribution. For example
        \begin{enumerate}
            \item If the contribution is primarily a new algorithm, the paper should make it clear how to reproduce that algorithm.
            \item If the contribution is primarily a new model architecture, the paper should describe the architecture clearly and fully.
            \item If the contribution is a new model (e.g., a large language model), then there should either be a way to access this model for reproducing the results or a way to reproduce the model (e.g., with an open-source dataset or instructions for how to construct the dataset).
            \item We recognize that reproducibility may be tricky in some cases, in which case authors are welcome to describe the particular way they provide for reproducibility. In the case of closed-source models, it may be that access to the model is limited in some way (e.g., to registered users), but it should be possible for other researchers to have some path to reproducing or verifying the results.
        \end{enumerate}
    \end{itemize}

\item {\bf Open access to data and code}
    \item[] Question: Does the paper provide open access to the data and code, with sufficient instructions to faithfully reproduce the main experimental results, as described in supplemental material?
    \item[] Answer: \answerYes{}
    \item[] Justification: The public code repository is available at \url{https://github.com/PeterXingyuZhao/Arithmetic_Transformer}. The data are synthetic or template-based, with generation procedures described in Section~\ref{sec:experimental-setup} and the appendix.
    \item[] Guidelines:
    \begin{itemize}
        \item The answer \answerNA{} means that paper does not include experiments requiring code.
        \item Please see the NeurIPS code and data submission guidelines (\url{https://neurips.cc/public/guides/CodeSubmissionPolicy}) for more details.
        \item While we encourage the release of code and data, we understand that this might not be possible, so \answerNo{} is an acceptable answer. Papers cannot be rejected simply for not including code, unless this is central to the contribution (e.g., for a new open-source benchmark).
        \item The instructions should contain the exact command and environment needed to run to reproduce the results. See the NeurIPS code and data submission guidelines (\url{https://neurips.cc/public/guides/CodeSubmissionPolicy}) for more details.
        \item The authors should provide instructions on data access and preparation, including how to access the raw data, preprocessed data, intermediate data, and generated data, etc.
        \item The authors should provide scripts to reproduce all experimental results for the new proposed method and baselines. If only a subset of experiments are reproducible, they should state which ones are omitted from the script and why.
        \item At submission time, to preserve anonymity, the authors should release anonymized versions (if applicable).
        \item Providing as much information as possible in supplemental material (appended to the paper) is recommended, but including URLs to data and code is permitted.
    \end{itemize}

\item {\bf Experimental setting/details}
    \item[] Question: Does the paper specify all the training and test details (e.g., data splits, hyperparameters, how they were chosen, type of optimizer) necessary to understand the results?
    \item[] Answer: \answerYes{}
    \item[] Justification: The main text describes tasks, data distributions, architecture, optimizer, batch size, learning rate, dropout, and decoding; Appendix~\ref{sec:append-synthetic-details}, Appendix~\ref{sec:append-data-gen-comparison}, Appendix~\ref{sec:append-data-gen-sorting}, and the experiment-settings appendix provide further details.

    \item[] Guidelines:
    \begin{itemize}
        \item The answer \answerNA{} means that the paper does not include experiments.
        \item The experimental setting should be presented in the core of the paper to a level of detail that is necessary to appreciate the results and make sense of them.
        \item The full details can be provided either with the code, in appendix, or as supplemental material.
    \end{itemize}

\item {\bf Experiment statistical significance}
    \item[] Question: Does the paper report error bars suitably and correctly defined or other appropriate information about the statistical significance of the experiments?
    \item[] Answer: \answerNo{}
    \item[] Justification: The paper reports learning curves, error rates, ablation results, and robustness checks, but it does not currently report multi-seed error bars, confidence intervals, or formal statistical significance tests for the main experimental claims.

    \item[] Guidelines:
    \begin{itemize}
        \item The answer \answerNA{} means that the paper does not include experiments.
        \item The authors should answer \answerYes{} if the results are accompanied by error bars, confidence intervals, or statistical significance tests, at least for the experiments that support the main claims of the paper.
        \item The factors of variability that the error bars are capturing should be clearly stated (for example, train/test split, initialization, random drawing of some parameter, or overall run with given experimental conditions).
        \item The method for calculating the error bars should be explained (closed form formula, call to a library function, bootstrap, etc.)
        \item The assumptions made should be given (e.g., Normally distributed errors).
        \item It should be clear whether the error bar is the standard deviation or the standard error of the mean.
        \item It is OK to report 1-sigma error bars, but one should state it. The authors should preferably report a 2-sigma error bar than state that they have a 96\% CI, if the hypothesis of Normality of errors is not verified.
        \item For asymmetric distributions, the authors should be careful not to show in tables or figures symmetric error bars that would yield results that are out of range (e.g., negative error rates).
        \item If error bars are reported in tables or plots, the authors should explain in the text how they were calculated and reference the corresponding figures or tables in the text.
    \end{itemize}

\item {\bf Experiments compute resources}
    \item[] Question: For each experiment, does the paper provide sufficient information on the computer resources (type of compute workers, memory, time of execution) needed to reproduce the experiments?
    \item[] Answer: \answerNo{}
    \item[] Justification: The paper provides model sizes and training hyperparameters, but it does not currently provide sufficient GPU/CPU type, memory, wall-clock time, or total compute estimates for each experiment.
    \item[] Guidelines:
    \begin{itemize}
        \item The answer \answerNA{} means that the paper does not include experiments.
        \item The paper should indicate the type of compute workers CPU or GPU, internal cluster, or cloud provider, including relevant memory and storage.
        \item The paper should provide the amount of compute required for each of the individual experimental runs as well as estimate the total compute. 
        \item The paper should disclose whether the full research project required more compute than the experiments reported in the paper (e.g., preliminary or failed experiments that didn't make it into the paper). 
    \end{itemize}
    
\item {\bf Code of ethics}
    \item[] Question: Does the research conducted in the paper conform, in every respect, with the NeurIPS Code of Ethics \url{https://neurips.cc/public/EthicsGuidelines}?
    \item[] Answer: \answerYes{}
    \item[] Justification: The research uses synthetic arithmetic data, public benchmark templates, and evaluations of existing models, and does not involve human-subject experiments, private data, or deployment of high-risk systems.
    \item[] Guidelines:
    \begin{itemize}
        \item The answer \answerNA{} means that the authors have not reviewed the NeurIPS Code of Ethics.
        \item If the authors answer \answerNo, they should explain the special circumstances that require a deviation from the Code of Ethics.
        \item The authors should make sure to preserve anonymity (e.g., if there is a special consideration due to laws or regulations in their jurisdiction).
    \end{itemize}

\item {\bf Broader impacts}
    \item[] Question: Does the paper discuss both potential positive societal impacts and negative societal impacts of the work performed?
    \item[] Answer: \answerYes{}
\item[] Justification: Appendix~\ref{app:broader-impacts} discusses
potential benefits for evaluating reasoning reliability, risks
of overgeneralizing the findings or drawing unsupported deployment
conclusions, and measures for limiting these risks.

    \item[] Guidelines:
    \begin{itemize}
        \item The answer \answerNA{} means that there is no societal impact of the work performed.
        \item If the authors answer \answerNA{} or \answerNo, they should explain why their work has no societal impact or why the paper does not address societal impact.
        \item Examples of negative societal impacts include potential malicious or unintended uses (e.g., disinformation, generating fake profiles, surveillance), fairness considerations (e.g., deployment of technologies that could make decisions that unfairly impact specific groups), privacy considerations, and security considerations.
        \item The conference expects that many papers will be foundational research and not tied to particular applications, let alone deployments. However, if there is a direct path to any negative applications, the authors should point it out. For example, it is legitimate to point out that an improvement in the quality of generative models could be used to generate Deepfakes for disinformation. On the other hand, it is not needed to point out that a generic algorithm for optimizing neural networks could enable people to train models that generate Deepfakes faster.
        \item The authors should consider possible harms that could arise when the technology is being used as intended and functioning correctly, harms that could arise when the technology is being used as intended but gives incorrect results, and harms following from (intentional or unintentional) misuse of the technology.
        \item If there are negative societal impacts, the authors could also discuss possible mitigation strategies (e.g., gated release of models, providing defenses in addition to attacks, mechanisms for monitoring misuse, mechanisms to monitor how a system learns from feedback over time, improving the efficiency and accessibility of ML).
    \end{itemize}
    
\item {\bf Safeguards}
    \item[] Question: Does the paper describe safeguards that have been put in place for responsible release of data or models that have a high risk for misuse (e.g., pre-trained language models, image generators, or scraped datasets)?
    \item[] Answer: \answerNA{}
    \item[] Justification: The work does not release high-risk pretrained language models, image generators, scraped datasets, or other assets with substantial misuse risk; the released materials are intended for reproducing synthetic arithmetic experiments.
    \item[] Guidelines:
    \begin{itemize}
        \item The answer \answerNA{} means that the paper poses no such risks.
        \item Released models that have a high risk for misuse or dual-use should be released with necessary safeguards to allow for controlled use of the model, for example by requiring that users adhere to usage guidelines or restrictions to access the model or implementing safety filters. 
        \item Datasets that have been scraped from the Internet could pose safety risks. The authors should describe how they avoided releasing unsafe images.
        \item We recognize that providing effective safeguards is challenging, and many papers do not require this, but we encourage authors to take this into account and make a best faith effort.
    \end{itemize}

\item {\bf Licenses for existing assets}
    \item[] Question: Are the creators or original owners of assets (e.g., code, data, models), used in the paper, properly credited and are the license and terms of use explicitly mentioned and properly respected?
    \item[] Answer: \answerYes{}
\item[] Justification: Appendix~\ref{app:assets} credits
the original creators of the reused software, datasets,
and pretrained models and documents their sources,
model identifiers, and licenses or terms of use.
The accompanying repository preserves the original
copyright and MIT license notices for NanoGPT,
teaching\_arithmetic, and GSM8K, identifies our
modifications, and documents the separate terms
governing pretrained models and generated outputs.
    % \item[] Answer: \answerNo{}
    % \item[] Justification: The paper cites existing assets such as NanoGPT, Pythia-1B, GSM8K-style benchmarks, and evaluated open-source LLMs, but it does not currently list all corresponding licenses or terms of use.

    \item[] Guidelines:
    \begin{itemize}
        \item The answer \answerNA{} means that the paper does not use existing assets.
        \item The authors should cite the original paper that produced the code package or dataset.
        \item The authors should state which version of the asset is used and, if possible, include a URL.
        \item The name of the license (e.g., CC-BY 4.0) should be included for each asset.
        \item For scraped data from a particular source (e.g., website), the copyright and terms of service of that source should be provided.
        \item If assets are released, the license, copyright information, and terms of use in the package should be provided. For popular datasets, \url{paperswithcode.com/datasets} has curated licenses for some datasets. Their licensing guide can help determine the license of a dataset.
        \item For existing datasets that are re-packaged, both the original license and the license of the derived asset (if it has changed) should be provided.
        \item If this information is not available online, the authors are encouraged to reach out to the asset's creators.
    \end{itemize}

\item {\bf New assets}
    \item[] Question: Are new assets introduced in the paper well documented and is the documentation provided alongside the assets?
    \item[] Answer: \answerYes{}
    \item[] Justification: The new experimental code, synthetic data-generation procedures, and clause-extension templates are documented through the supplementary repository and the appendix descriptions of task construction and evaluation.
    \item[] Guidelines:
    \begin{itemize}
        \item The answer \answerNA{} means that the paper does not release new assets.
        \item Researchers should communicate the details of the dataset\slash code\slash model as part of their submissions via structured templates. This includes details about training, license, limitations, etc. 
        \item The paper should discuss whether and how consent was obtained from people whose asset is used.
        \item At submission time, remember to anonymize your assets (if applicable). You can either create an anonymized URL or include an anonymized zip file.
    \end{itemize}

\item {\bf Crowdsourcing and research with human subjects}
    \item[] Question: For crowdsourcing experiments and research with human subjects, does the paper include the full text of instructions given to participants and screenshots, if applicable, as well as details about compensation (if any)? 
    \item[] Answer: \answerNA{}
    \item[] Justification: The paper does not involve crowdsourcing experiments or research with human subjects; all data are synthetic, template-derived, or based on existing public benchmarks and models.
    \item[] Guidelines:
    \begin{itemize}
        \item The answer \answerNA{} means that the paper does not involve crowdsourcing nor research with human subjects.
        \item Including this information in the supplemental material is fine, but if the main contribution of the paper involves human subjects, then as much detail as possible should be included in the main paper. 
        \item According to the NeurIPS Code of Ethics, workers involved in data collection, curation, or other labor should be paid at least the minimum wage in the country of the data collector. 
    \end{itemize}

\item {\bf Institutional review board (IRB) approvals or equivalent for research with human subjects}
    \item[] Question: Does the paper describe potential risks incurred by study participants, whether such risks were disclosed to the subjects, and whether Institutional Review Board (IRB) approvals (or an equivalent approval/review based on the requirements of your country or institution) were obtained?
    \item[] Answer: \answerNA{}
    \item[] Justification: The paper does not involve crowdsourcing or human-subject research, so IRB or equivalent approval is not applicable.
    \item[] Guidelines:
    \begin{itemize}
        \item The answer \answerNA{} means that the paper does not involve crowdsourcing nor research with human subjects.
        \item Depending on the country in which research is conducted, IRB approval (or equivalent) may be required for any human subjects research. If you obtained IRB approval, you should clearly state this in the paper. 
        \item We recognize that the procedures for this may vary significantly between institutions and locations, and we expect authors to adhere to the NeurIPS Code of Ethics and the guidelines for their institution. 
        \item For initial submissions, do not include any information that would break anonymity (if applicable), such as the institution conducting the review.
    \end{itemize}

\item {\bf Declaration of LLM usage}
    \item[] Question: Does the paper describe the usage of LLMs if it is an important, original, or non-standard component of the core methods in this research? Note that if the LLM is used only for writing, editing, or formatting purposes and does \emph{not} impact the core methodology, scientific rigor, or originality of the research, declaration is not required.
    %this research? 
    \item[] Answer: \answerYes{}
    \item[] Justification: The use of LLMs as evaluated systems is described in Section~\ref{sec:external-robustness}, where modern open-source LLMs are tested on clause-extended GSM8K problems.
    \item[] Guidelines:
    \begin{itemize}
        \item The answer \answerNA{} means that the core method development in this research does not involve LLMs as any important, original, or non-standard components.
        \item Please refer to our LLM policy in the NeurIPS handbook for what should or should not be described.
    \end{itemize}

\end{enumerate}
\fi

\end{document}